\documentclass{article}

\usepackage[main,final]{neurips_2025}

\usepackage[utf8]{inputenc} 
\usepackage[T1]{fontenc}    
\usepackage[colorlinks=true,linkcolor=red,citecolor=teal]{hyperref}       
\usepackage{url}            
\usepackage{amsfonts}       
\usepackage{nicefrac}       
\usepackage{microtype}      
\usepackage{xcolor}         

\usepackage{amsmath,amsfonts,amssymb,mathrsfs}
\usepackage{graphicx,psfrag,epsf,url,wrapfig}
\usepackage{enumerate,color,centernot}
\usepackage{natbib}
\usepackage{adjustbox,booktabs,arydshln}
\usepackage{caption}
\usepackage{subcaption}
\usepackage{xcolor,tikz,wrapfig}
\usetikzlibrary{decorations.pathreplacing,angles,quotes}

\usepackage[ruled,vlined]{algorithm2e}

\usepackage{xr}
 
\usepackage{microtype}
\usepackage{graphicx}
\usepackage{booktabs} 

\usepackage{hyperref}

\usepackage{amsmath}
\usepackage{amssymb}
\usepackage{mathtools}
\usepackage{amsthm}

\usepackage[capitalize,noabbrev]{cleveref}

\theoremstyle{plain}
\newtheorem{theorem}{Theorem}[section]

\newtheorem{lemma}[theorem]{Lemma}

\theoremstyle{definition}

\newtheorem{assumption}[theorem]{Assumption}
\theoremstyle{remark}

\newcommand{\bX}{{\boldsymbol X}}
\newcommand{\bY}{{\boldsymbol Y}}

\newcommand{\bx}{{\boldsymbol x}} 
\newcommand{\bz}{{\boldsymbol z}}
\newcommand{\bZ}{{\boldsymbol Z}}

\newcommand{\bb}{{\boldsymbol b}}
\newcommand{\be}{{\boldsymbol e}}
\newcommand{\bbf}{{\boldsymbol f}}
\newcommand{\bg}{{\boldsymbol g}}
\newcommand{\bm}{{\boldsymbol m}}
\newcommand{\bw}{{\boldsymbol w}}

\newcommand{\bv}{{\boldsymbol v}}
\newcommand{\bt}{{\boldsymbol t}}

\newcommand{\mK}{\mathcal{K}}
\newcommand{\bV}{{\boldsymbol V}}

\newcommand{\bu}{{\boldsymbol u}}

\newcommand{\bxi}{{\boldsymbol \xi}}

\newcommand{\bbeta}{{\boldsymbol \beta}}

\newcommand{\btheta}{{\boldsymbol \theta}}
\newcommand{\bTheta}{{\boldsymbol \Theta}}
\newcommand{\bvartheta}{{\boldsymbol \vartheta}}

\newcommand{\bmu}{{\boldsymbol \mu}}

\newcommand{\bYmis}{{\boldsymbol Y}_{{\rm mis}}}

\title{Uncertainty Quantification for Physics-Informed Neural Networks with Extended Fiducial Inference}

\author{%
  Frank Shih \\
  Department of Epidemiology and Biostatistics \\
  Memorial Sloan Kettering Cancer Center \\
  New York, NY 10017 \\
  \texttt{shihf@mskcc.org} \\
  \And
  Zhenghao Jiang \\
  Department of Statistics \\ 
  Purdue University \\
  West Lafayette, IN 47907 \\
  \texttt{jiang976@purdue.edu} \\
  \AND
  Faming Liang \\
  Department of Statistics \\
  Purdue University \\
  West Lafayette, IN 47907 \\
  \texttt{fmliang@purdue.edu} \\
}

\begin{document}

\maketitle

\begin{abstract}
Uncertainty quantification (UQ) in scientific machine learning is increasingly 
critical as neural networks are widely adopted to tackle complex 
problems across diverse scientific disciplines. 
For physics-informed neural networks (PINNs), a prominent model in scientific machine learning, uncertainty is typically quantified using 
Bayesian or dropout methods. However, both approaches suffer from a fundamental limitation: the prior distribution or dropout rate required to construct 
honest confidence sets cannot be determined without additional information.
In this paper, we propose a novel method within the framework of extended fiducial inference (EFI) to provide rigorous uncertainty quantification for PINNs. The proposed method leverages a narrow-neck hyper-network to learn the parameters of the PINN and quantify their uncertainty based on imputed random errors in the observations. This approach overcomes the limitations of Bayesian and dropout methods, enabling the construction of honest confidence sets based solely on observed data.
This advancement represents a significant breakthrough for PINNs, greatly enhancing their reliability, interpretability, and applicability to real-world scientific and engineering challenges. Moreover, it establishes 
a new theoretical framework for EFI, extending its application to large-scale models, eliminating the need for sparse hyper-networks, and significantly improving the automaticity and robustness of statistical inference.

 \vspace{2mm}
\textbf{Keywords}: Adaptive Stochastic Gradient MCMC,  Deep Learning, 
Black-Scholes Model, 
Partial Differential Equation, Porous-FKPP Model
\end{abstract}

\section{Introduction} 
 
Physics-informed neural networks (PINNs) \citep{Raissi2019PhysicsinformedNN} are a class of scientific machine learning  models that integrate physical principles directly into the training process of the DNN models. They achieve this by incorporating terms derived from ordinary differential equations (ODEs) or partial differential equations (PDEs) into the DNN’s loss function. 
PINNs take spatial-temporal coordinates as input and produce functions that approximate the solutions to the differential equations. Because they embed physics-based constraints, PINNs can typically address problems 
that are described by few data while ensuring adherence to the given physical laws \citep{Zou2025UncertaintyQF,Cuomo2022ScientificML}, paving the ways for the use of neural networks in out-of-distribution (OOD) prediction (see e.g., \citet{Yao2024ImprovingOG}). 

From a Bayesian viewpoint, PINNs can be seen as using ``informative priors'' drawn from mathematical models, thus requiring less data for system identification \citep{Zhang2019QuantifyingTU}. However, this naturally raises a challenge in uncertainty quantification: 
How should one balance the prior information and the observed data to ensure faithful inference about the underlying physical system? As emphasized in \citet{Zou2025UncertaintyQF}, uncertainty quantification is becoming increasingly critical as neural networks are widely employed to tackle complex scientific problems, particularly in high-stake application scenarios.
However, addressing this issue within the Bayesian framework is difficult, as the prior information and data information are  essentially 
exchangeable under Bayesian formulations.
The use of ``informative priors'' can conflict with the spirit of ``posterior consistency'' (see e.g., \citet{Ghosal2000ConvergenceRO}), a foundational principle in Bayesian inference. 
 This conflict creates a dilemma: {\it using an informative prior risks overshadowing the data, while using a weak prior may lead to violations of the underlying physical law.} In practice, 
 this makes it difficult, if not impossible, to properly calibrate 
 the resulting credible intervals without additional information.


 In addition to Bayesian methods \citep{BPINN}, dropout \citep{srivastava2014dropout} has also been employed to quantify the uncertainty of PINNs, as demonstrated by \citet{Zhang2019QuantifyingTU}. Dropout is primarily used as a regularization technique to reduce overfitting during DNN training. Gal and Ghahramani \citep{Gal2016DropoutAA} showed that dropout training in DNNs can be interpreted as approximate Bayesian inference in deep Gaussian processes, allowing model uncertainty to be estimated from dropout-trained DNN models.
However, this approach shares a limitation similar to Bayesian methods: 
{\it The dropout rate, which directly influences the magnitude of the estimated model uncertainty, cannot be determined without additional information, making it challenging to ensure consistent and reliable uncertainty quantification.}

This paper introduces an EFI \citep{LiangKS2024EFI} 
approach to quantify the uncertainty in PINNs. 
EFI provides a rigorous theoretical framework that addresses the limitations of Bayesian and dropout methods by formulating the problem as a structural equation-solving task. In this framework, each observation is expressed 
 as a data-generating equation, with the random errors (contained in observations) and DNN parameters 
treated as unknowns (see Section \ref{sect:EFI}).
EFI jointly imputes the random errors and estimates the inverse function that maps the observations and imputed random errors to DNN parameters. 
Consequently, the imputed random errors are propagated to the DNN parameters through the estimated inverse function, allowing the model uncertainty to be accurately quantified without the need for additional information.
Our contribution in this paper is two-fold: 

\begin{itemize} 
\vspace{-0.1in}
\item {\bf A new theoretical framework for EFI:} We develop 
 a new theoretical framework for EFI that significantly enhances the 
 automaticity of statistical inference.  Originally, EFI was 
 developed in \citet{LiangKS2024EFI} under a Bayesian framework, 
where a sparse prior is imposeStochasticd on the hyper-network (referred to as the $\bw$-network in Section \ref{sect:EFI}) to ensure consistent 
 estimation of the inverse function. However, due to the limitations
  of existing sparse deep learning theory \citep{SunSLiang2021},
  this Bayesian approach could only be applied to models 
with  dimensions fixed 
or increasing at a very low rate with the sample size. 
In this paper, we propose learning the inverse function using a narrow-neck 
$\bw$-network, which ensures consistent estimation of the inverse function 
without relying on the use of sparse priors.  
Moreover, it enables EFI to work for large-scale models, such 
as PINNs, where the number of model parameters can far exceed the sample size. 
By avoiding the need for Bayesian sparse priors, our framework allows 
EFI to fulfil the original 
goal of fiducial inference: {\it Inferring the uncertainty of model parameters 
based solely on observations}. 

\item {\bf Open-source software for uncertainty quantification in PINNs:} 
 We provide an open source software package for uncertainty quantification 
 in PINNs, which can be easily extended to conventional DNNs and other 
 high-dimensional statistical models.
 
\end{itemize} 

\paragraph{Related Work}
The proposed method belongs to the class of imprecise probabilistic techniques \citep{Augustin2014IntroductionTI}. 
However, compared to other methods in the class, such as credal Bayesian deep learning \citep{Caprio2023CredalBD}, imprecise Bayesian neural networks \citep{Caprio2023ImpreciseBN}, and other Bayesian neural network-based methods, the key advantage of EFI is that it avoids the need for prior specification while ensuring accurate calibration of predictions.

\textcolor{black}{
Another related line of work concerns uncertainty quantification for machine learning models. Beyond the Bayesian and dropout methods noted above, this line includes conformal prediction \citep{Vovk2005AlgorithmicLI}, deep ensembles \citep{Lakshminarayanan2016SimpleAS}, and stochastic deep learning \citep{SunLiang2022kernel, LiangSLiang2022}, among others.
  These methods primarily target predictive uncertainty and are often ineffective or inapplicable for quantifying uncertainty in model parameters. In contrast, EFI addresses both predictive uncertainty and parameter uncertainty, and further provides theoretical guarantees for the validity of the resulting prediction and confidence intervals. The ability to accurately quantify uncertainty in deep neural network parameters is a distinctive advantage of EFI.}


\section{A Brief Review of EFI} \label{sect:EFI} 

While fiducial inference was widely considered as a big blunder by R.A. Fisher, the goal he initially set ---inferring the uncertainty of model parameters based solely on observations --- has been continually pursued by many statisticians, see e.g.  structural inference \citep{Fraser1966StructuralPA,Fraser1968Book}, generalized fiducial inference
\citep{hannig2009gfi,hannig2016gfi,Murph2022GeneralizedFI}, and inferential models \citep{Martin2013InferentialMA, Martin2015Book, Martin2023FiducialIV}.
To this end, \citet{LiangKS2024EFI} 
developed the EFI method based on the fundamental concept of structural inference. 

Consider a regression model: $Y=f(\bX,Z,\btheta)$, 
where $Y\in \mathbb{R}$ and $\bX\in \mathbb{R}^{d}$ represent the response and explanatory variables, respectively; $\btheta\in \mathbb{R}^p$ represents the vector of  parameters; 
and $Z\in \mathbb{R}$ represents a scaled random error following  
 a known distribution $\pi_0(\cdot)$.  
Suppose that a random sample of size $n$,  denoted by $\{(y_1,\bx_1), (y_2,\bx_2),\ldots,(y_n,\bx_n)\}$, has been collected from the model. In structural inference, the observations can be expressed in data-generating equations as follows:  
\begin{equation} \label{dataGeneqg}
y_i=f(\bx_i,z_i,\btheta), \quad i=1,2,\ldots,n.
\end{equation}
This system of equations consists of $n+p$ unknowns, namely, $\{\btheta, z_1, z_2, \ldots, z_n
\}$, while there are only $n$ equations. Therefore, the values of $\btheta$ cannot be uniquely determined by the data-generating equations, and this lack of uniqueness of unknowns introduces uncertainty in $\btheta$.  

Let $\bZ_n=\{z_1,z_2,\ldots,z_n\}$ denote the unobservable random errors contained in the data,
which are also called latent variables in EFI.  
Let $G(\cdot)$ denote an inverse function/mapping for   $\btheta$, i.e., 
\begin{equation} \label{Inveq}
\btheta=G(\bY_n,\bX_n,\bZ_n).
\end{equation}
It is worth noting that the inverse function is generally non-unique. For example, it can be constructed by solving any $p$ equations in (\ref{dataGeneqg}) for $\btheta$. 
As noted by \citet{LiangKS2024EFI}, this non-uniqueness of inverse function 
mirrors the flexibility of frequentist methods, where different 
estimators of $\btheta$ can be constructed to achieve desired properties 
such as efficiency, unbiasedness, and robustness. 

Since the inverse function $G(\cdot)$ is generally unknown, \citet{LiangKS2024EFI} proposed to approximate it using a sparse DNN, see Figure \ref{EFInetwork} in the Appendix for illustration.
They also introduced an adaptive stochastic gradient Langevin dynamics (SGLD) algorithm, which facilitates the  simultaneous training of the sparse DNN 
 and simulation of the latent variables $\bZ_n$.  
See Algorithm \ref{EFIalgorithm} for the pseudo-code. Refer to 
Section \ref{sect:sup:EFI} of the Appendix for the mathematical formulation of the method. 
Briefly, they let $\bw_n$ denote the weights of $\bw$-network and define an energy function  $U_n(\bY_n,\bX_n,\bZ_n,\bw_n)$. subsequently, they define 
a posterior distribution $\pi_{\epsilon}(\bw_n|\bX_n,\bY_n,\bZ_n)$ for $\bw_n$ 
and a predictive distribution $\pi_{\epsilon}(\bZ_n|\bX_n,\bY_n,\bw_n)$ for 
$\bZ_n$, where $\epsilon$ can be read as a temperature. They treat $\bZ_n$ 
as missing data and learn $\bw_n$ through solving the following equation: 
 \begin{equation} \label{identityeq}
 \begin{split}
 \nabla_{\bw_n} \log \pi_{\epsilon}(\bw_n|\bX_n,\bY_n) &=\int \Big[\nabla_{\bw_n} \log \pi_{\epsilon}(\bw_n|\bX_n,\bY_n,\bZ_n) 
   \pi_{\epsilon}(\bZ_n|\bX_n,\bY_n,\bw_n)\Big] d\bZ_n=0,     
 \end{split}
 \end{equation} 
using Algorithm \ref{EFIalgorithm}.

 \begin{algorithm}[!ht]
 \caption{Adaptive SGLD for EFI computation}
 \label{EFIalgorithm}
 \SetAlgoLined 
 {\bf (i) (Initialization)} Initialize
 $\bw_n^{(0)}$, $\bZ_n^{(0)}$, $M$ (the number of fiducial samples to collect), and $\mK$ (burn-in iterations).

\For{k=1,2,\ldots,\mbox{$\mK+M$}}{

 {\bf (ii) (Latent variable imputation)} Given $\bw_n^{(k)}$, simulate $\bZ_n^{(k+1)}$ using the SGLD algorithm: 
\begin{equation*}
    \bZ_n^{(k+1)} = \bZ_n^{(k)} + \upsilon_{k+1} \nabla_{\bZ_n} \log \pi_{\epsilon}(\bZ_n^{(k)}|\bX_n,\bY_n,\bw_n^{(k)}) +\sqrt{2  \upsilon_{k+1}} \be^{(k+1)}    
\end{equation*}
where $\upsilon_{k+1}$ is the learning rate, and  
$\be^{(k+1)} \sim N(0,I_{d_{\bz}})$.

 {\bf (iii) (Parameter updating)} Draw a minibatch $\{(y_1,\bx_1,z_1^{(k)}),\ldots,(y_m,\bx_m,z_m^{(k)})\}$ 
 and update the network weights by the SGD algorithm: 
\begin{equation} 
\label{sgd_update}
    \bw_n^{(k+1)}=\bw_n^{(k)} 
    +\gamma_{k+1} \Biggl[ \frac{n}{m} \sum_{i=1}^m \nabla_{\bw_n} \log \pi_{\epsilon}(y_i|\bx_i,z_i^{(k)},\bw_n^{(k)}) 
    + \nabla_{\bw_n} \log \pi(\bw_n^{(k)}) \Biggr], 
\end{equation}
where $\gamma_{k+1}$ is the step size, and $\log \pi_{\epsilon}(y_i|\bx_i,z_i^{(k)},\bw_n^{(k)})$ can be appropriately defined according to (\ref{app:energyfunction11}).  
  
{\bf (iv) (Fiducial sample collection)} If $k+1 > \mK$, calculate 
 $\hat{\btheta}_i^{(k+1)}=\hat{g}(y_i,\bx_i,z_i^{(k+1)},\bw_n^{(k+1)})$ 
 for each $i\in \{1,2,\ldots,n\}$ and average them to get a fiducial $\bar{\btheta}$-sample as calculated in (\ref{app:thetabareq}). 
}

{\bf (v) (Statistical Inference)} Conducting statistical inference for the model based on the collected fiducial samples.
\end{algorithm}


Under mild conditions for the adaptive SGLD algorithm, it can be shown that  
\begin{equation} \label{wconvergence}
\|\bw_n^{(k)} -\bw_n^* \|\stackrel{p}{\to} 0, \quad \mbox{as $k\to \infty$},
\end{equation}
 where $\bw_n^*$ denotes a solution to equation (\ref{identityeq}) and $\stackrel{p}{\to}$ denotes convergence in probability, and  that 
\begin{equation} \label{Zconvergence}
\bZ_n^{(k)} \stackrel{d}{\rightsquigarrow} \pi_{\epsilon}(\bZ_n|\bX_n,\bY_n,\bw_n^*), 
\quad \mbox{as $k \to \infty$},
\end{equation}
in 2-Wasserstein distance, where  $\stackrel{d}{\rightsquigarrow}$ denotes weak convergence. 
To study the limit of (\ref{Zconvergence}) as $\epsilon$ decays to 0, i.e., 
$p_n^*(\bz|\bY_n,\bX_n,\bw_n^*)=  
\lim_{\epsilon \downarrow 0} \pi_{\epsilon}(\bZ_n|\bX_n,\bY_n,\bw_n^*)$,
where 
$p_n^*(\bz|\bY_n,\bX_n,\bw_n^*)$ 
is referred to as the extended fiducial 
density (EFD) of $\bZ_n$, 
\citet{LiangKS2024EFI} impose specific conditions on the structure of the $\bw$-network, including that the $\bw$-network is sparse 
and that the output layer width (i.e., the dimension of $\btheta$) is either fixed or grows very slowly with the sample size $n$. 
Under these assumptions, they prove the consistency of $\bw_n^*$ based on the sparse deep learning theory developed in \citet{SunSLiang2021}. This consistency  further implies that 
 \begin{equation} \label{mappingest}
 G^*(\bY_n,\bX_n,\bZ_n)= \frac{1}{n} \sum_{i=1}^n \hat{g}(y_i,\bx_i,z_i,\bw_n^*),
 \end{equation}
 serves as a consistent estimator for the inverse 
function/mapping $\btheta=G(\bY_n,\bX_n,\bZ_n)$, where $\hat{g}(\cdot)$ denotes 
the learned neural network function. 
Refer to Appendix \ref{sect:EFDlimit} for the expression of $p_n^*(\bz|\bY_n,\bX_n,\bw_n^*)$. 



Let $\mathcal{Z}_n=\{\bz \in \mathbb{R}^n: U_n(\bY_n,\bX_n,\bZ_n, \bw_n^*)=0\}$  
denote the zero-energy set.
Under some regularity conditions on the energy function, \citet{LiangKS2024EFI} proved that $\mathcal{Z}_n$ is invariant to the choice of $G(\cdot)$. 
Let $\Theta:=\{\btheta \in \mathbb{R}^p: \btheta=G^*(\bY_n,\bX_n,\bz), \bz\in \mathcal{Z}_n\}$ denote the parameter space of the target model, which represents the set of all possible values of $\btheta$ that $G^*(\cdot)$ takes when $\bz$ runs over $\mathcal{Z}_n$.  
Then, for any function $b(\btheta)$ of 
interest, its EFD $\mu_n^*(\cdot|\bY_n,\bX_n)$ associated with 
 $G^*(\cdot)$ is given by 
\begin{equation} \label{EFDeq}
\mu_n^*(B|\bY_n,\bX_n) =\int_{\mathcal{Z}_n(B)} d P_n^*(\bz|\bY_n,\bX_n,\bw_n^*), 
\end{equation}
for any measurable set $B \subset \Theta$,
where $\mathcal{Z}_n(B)=\{\bz\in \mathcal{Z}_n: b(G^*(\bY_n,\bX_n,\bz)) \in B\}$, and $P_n^*(\bz|\bX_n,\bY_n$, $\bw_n^*)$ 
denote the cumulative distribution 
function (CDF) corresponding to $p_n^*(\bz|\bX_n,\bY_n,\bw_n^*)$.
The EFD  provides an uncertainty measure for $b(\btheta)$. 
Practically, 
it can be constructed based on the samples
$\{b(\bar{\btheta}_1), 
b(\bar{\btheta}_2), \ldots, b(\bar{\btheta}_M)\}$, where 
$\{\bar{\btheta}_1, \bar{\btheta}_2, \ldots, \bar{\btheta}_M\}$ denotes  
the fiducial $\bar{\btheta}$-samples collected at step (iv) of Algorithm \ref{EFIalgorithm}. 
\textcolor{black}{As a practical application, \citet{Kim2025ExtendedFI} applied EFI to quantify the uncertainty of individual treatment effects 
in causal inference.} 

Finally, we note that for a neural network model, its parameters are only unique up to certain loss-invariant transformations, such as reordering hidden neurons within the same hidden layer or simultaneously altering the sign or scale of certain connection weights \citep{SunSLiang2021}. 
Therefore, for the $\bw$-network, the consistency of $\bw_n^*$ refers to its consistency with respect to one of the equivalent 
solutions to (\ref{identityeq}), while mathematically $\bw_n^*$
can still be treated as unique.

\section{EFI for Uncertainty Quantification in PINNs}

\subsection{EFI Formulation for PDEs}

Consider a multidimensional dynamic process, $u(\bx)$, defined on a 
domain $\Omega \subset \mathbb{R}^d$ through a PDE:  
\[
\begin{split}
\mathcal{F}(u(\bx); \bbeta) & =f(\bx),  \quad \bx\in \Omega, \\
\mathcal{B}(u(\bx)) & =b(\bx), \quad  \bx\in \partial \Omega, 
\end{split} 
\]
where $\bx=(x_1,x_2,\ldots,x_{d-1},t)^T \in \mathbb{R}^d$ indicates 
the space-time coordinate vector, $u(\cdot)$ represents the unknown 
solution, $\bbeta$ are the parameters related to the physics, 
and $f$ and $b$ are called the physics term and initial/boundary term, 
respectively. The observations are given in the forms 
$\{\bx_i^u, u_i\}_{i=1}^{n_u}$, $\{\bx_i^f, f_i\}_{i=1}^{n_f}$, 
and $\{\bx_i^b,b_i\}_{i=1}^{n_b}$. 
Let $u_{\bvartheta}(\bx)$ denote the DNN approximation to the solution 
$u(\bx)$, where $\bvartheta$ denotes the DNN parameters. 
In data-generating equations, the observations
can be expressed as 
\begin{equation} \label{eq:PDE-PINNs}
    \begin{split}
    u_i & = u_{\bvartheta}(\bx_i^u)+   z_i^u, \quad i=1,2,\ldots,n_u,\\
    f_i & = \mathcal{F}(u_{\bvartheta}(\bx_i^f); \bbeta) + z_i^f, \quad i=1,2,\ldots, n_f, \\ 
    b_i &= \mathcal{B}(u_{\bvartheta}(\bx_i^b)) + z_i^b, \quad i=1,2,\ldots, n_b,
    \end{split}
\end{equation}
where $z_i^u$, $z_i^f$, and $z_i^b$ are independent Gaussian random errors with zero mean.  
Our objective is to infer $u$ and/or $\bbeta$, as well as to quantify their uncertainty, given the data and  governing  
physical law. In physics, inferring $u$ with $\bbeta$ known is  termed the {\it forward problem}, while 
inferring $u$ when $\bbeta$ is unknown is referred to as the {\it inverse problem}.

When applying EFI to address this problem, the EFI network comprises two DNNs. The first, referred to as the data modeling network, is used to approximate $u(\bx)$. The second, called the $\bw$-network, is used to approximate the parameters of the data modeling network as well as other parameters in equation (\ref{eq:PDE-PINNs}).
For an illustration, see Figure \ref{EFInetwork:Double}. 
Specifically,
    for the inverse problem, the output of the $\bw$-network corresponds to 
    $\btheta=\{\bvartheta,\bbeta\}$; 
    for the forward problem, its output corresponds to $\btheta=\{\bvartheta\}$.

   \begin{figure}[!ht] 
    \centering
    \includegraphics[width=0.6\textwidth]{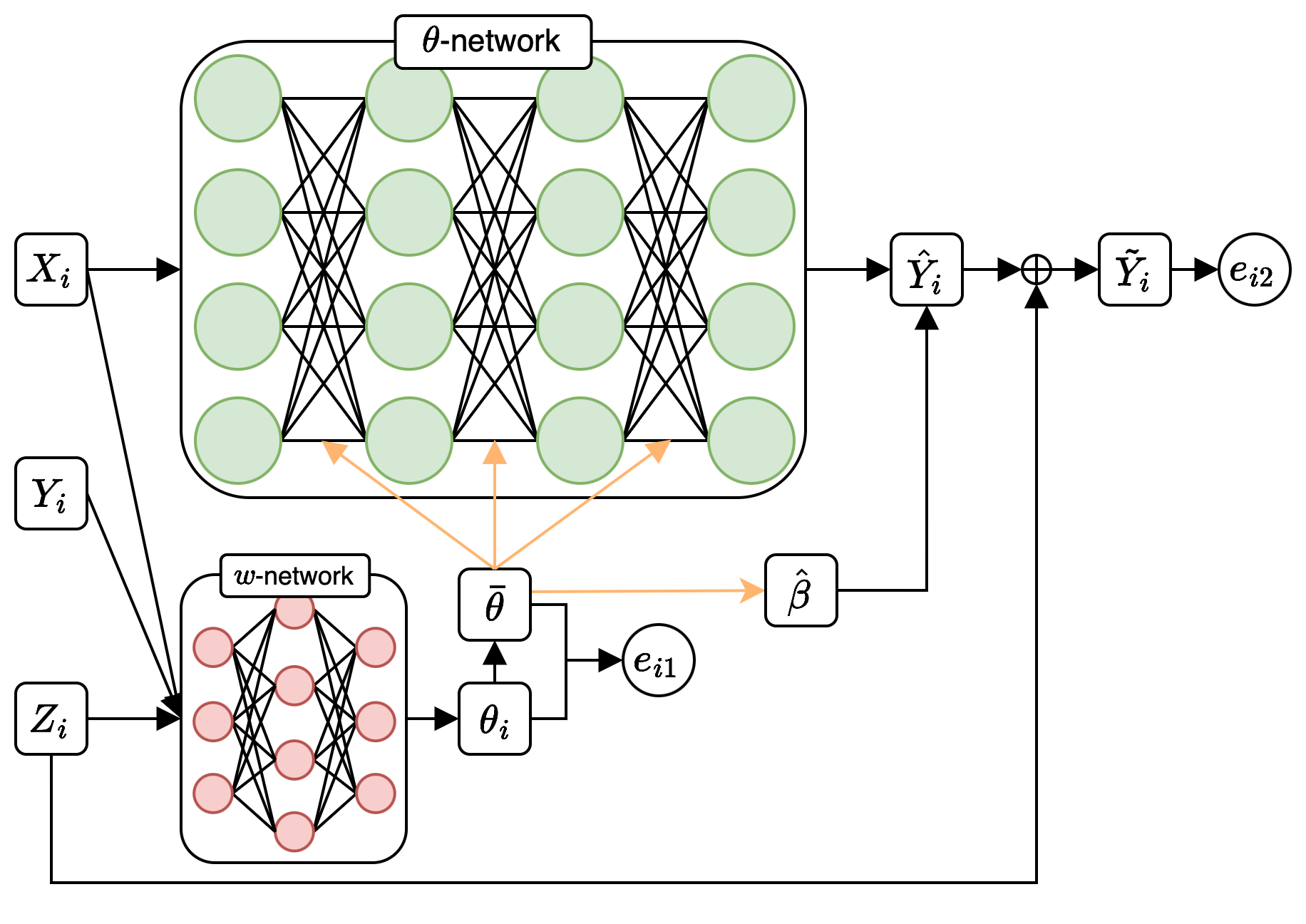}
    \caption{An EFI network with a double neural network (double-NN) structure.}
    \label{EFInetwork:Double}
 \end{figure} 

Given the data-generating equations, the energy function 
for EFI can be defined as follows:
\begin{equation}
\begin{split}
&U_n  (\bu_n,\bbf_n,\bg_n,\bz_n,\bw_n) =\eta_{\theta} \sum_{i=1}^n  \|\hat{\btheta_i}-\bar{\btheta}\|^2 + \eta_u \sum_{i=1}^{n_u} \|u_i-u_{\bvartheta}(\bx_i^u)-z_i^u\|^2 \\ 
 & + \eta_f \sum_{i=1}^{n_f} \|
f_i-\mathcal{F}(u_{\bvartheta}(\bx_i^f); \bbeta) - z_i^f\|^2 + \eta_b \sum_{i=1}^{n_b} \| b_i -\mathcal{B}(u_{\bvartheta}(\bx_i^b))- z_i^b\|^2, \\
\end{split} 
\end{equation}
 where $\bu_n=(u_1,\ldots,u_n)^T$, $\bbf=(f_1,\ldots,f_n)^T$, 
 $\bb_n=(b_1,\ldots,b_n)^T$, $\bz_n=(z_1^u,\ldots,z_n^u; z_1^f,\ldots,z_n^f$; $z_n^b,\ldots,z_n^b)^T$; 
 $n=\#\{z_i^u\ne 0, z_i^f \ne 0, z_i^b\ne 0\}$ denotes the total number 
 of noisy observations in the dataset;  
 and $\eta_{\theta}$, $\eta_u$, 
 $\eta_f$, and $\eta_b$ are belief weights for balancing different terms. 
 Fortunately, as shown in \citet{LiangKS2024EFI}, the choices for these 
 terms will not affect much the performance of the algorithm 
 as long as $\epsilon \to 0$. 


 Using EFI to solve PINNs, if we can correctly impute $\bZ_n$ and consistently estimate the inverse function $G(\cdot)$, the uncertainty of 
$\btheta$ can be accurately quantified according to (\ref{EFDeq}). However, 
  consistent estimation of the inverse function is unattainable under the current EFI theoretical framework due to the high dimensionality of 
  $\btheta$, which often far exceeds the sample size $n$.
This limitation arises from the existing sparse deep learning theory \citep{SunSLiang2021}, which constrains the dimension of $\btheta$ 
to remain fixed or grow very slowly with $n$.
To address this challenge, we propose a new theoretical framework for EFI, as detailed in Section \ref{sect:narrow-neck}, 
which extends EFI to accommodate large-scale models and addresses the constraints of the current framework.


\subsection{A New Theoretical Framework of EFI for Large-Scale Models} 
\label{sect:narrow-neck}


EFI treats the random errors in observations as latent variables. Consequently, training the $\bw$-network is reduced to a problem of parameter estimation with missing data. Under a Bayesian setting, \citet{LiangKS2024EFI} addressed the problem by solving equation (\ref{identityeq}) and employ  Algorithm \ref{EFIalgorithm} for the solution.
By imposing regularity conditions such as smoothness and dissipativity \citep{raginsky2017non}, \citet{LiangKS2024EFI} established the following convergence results for Algorithm \ref{EFIalgorithm}:


\begin{lemma} \label{lem1:Algconvergence} (Theorem 4.1 and Theorem 4.2, \citep{LiangKS2024EFI}) Suppose the regularity conditions in   \citet{LiangKS2024EFI}  hold, and the 
learning rate sequence $\{\upsilon_k: k=1,2,\ldots\}$ and 
the step size sequence $\{\gamma_k: k=1,2,\ldots \}$ are set as: 
$\upsilon_k=\frac{C_\upsilon}{c_\upsilon+k^{\alpha}}$ and  
$\gamma_k=\frac{C_\gamma}{c_\gamma+k^{\beta}}$
for some constants 
$C_{\upsilon}>0$, $c_{\upsilon}>0$, $C_{\gamma}>0$ and $c_{\gamma}>0$, and 
$\alpha,\beta \in (0,1]$ satisfying $\beta \leq \alpha \leq \min\{1,2\beta\}$. 
\begin{itemize} 
\item[(i)] (Root Consistency) There exists a root $\bw_n^* \in \{\bw: \nabla_{\bw} \log \pi(\bw|\bX_n,\bY_n)=0\}$ such that 
\[
 E \| \bw_n^{(k)}-\bw_n^*\| \leq \zeta \gamma_k, \quad k \geq k_0,
\]
 for some constant $\zeta>0$ and iteration number $k_0>0$. 
\item[(ii)] (Weak Convergence of Latent Variables) Let $\pi_z^*=\pi(\bZ_n|\bX_n,\bY_n,\bw_n^*)$ and $T_k = \sum_{i=0}^{k-1}\upsilon_{i+1}$,  and let $\pi_z^{(T_k)}$ denote the probability law of $\bZ_n^{(k)}$. Then,
\[  
\mathbb{W}_2(\pi_z^{(T_k)}, \pi_z^*) \leq (c_0 \delta_g^{1/4}+c_1 \gamma_1^{1/4})T_k + c_2 e^{-T_k/c_{LS}},
\]
for any $k\in \mathbb{N}$, where $\mathbb{W}_2(\cdot,\cdot)$ is the 2-Wasserstein distance, 
$c_0$, $c_1$, and $c_2$ are some positive constants, 
$c_{LS}$ is the logarithmic Sobolev constant of $\pi_z^*$, and $\delta_g$ is a coefficient reflecting the variation of the stochastic gradient 
used in latent variable imputation step. 
\end{itemize}
\end{lemma}

In our implementation, the latent variable imputation step is done for 
each observation separately, ensuring $\delta_g=0$. 
We choose $\alpha \in (0,1]$  and set
$\gamma_1 \prec \frac{1}{T_k^4}$ for any $T_k$, which ensures $\mathbb{W}_2(\pi_z^{(T_k)}, \pi_z^*) \to 0$ as $k\to \infty$.  
It is worth noting that Lemma \ref{lem1:Algconvergence} holds regardless of the size of the $\bw$-network. Thus, it remains 
valid for large-scale models. However, in this work, 
we do not impose any priors on $\bw_n$, which corresponds to the 
non-informative prior setting $\pi(\bw_n) \propto 1$.

 Given the root consistency result,  it is still necessary to establish that the resulting inverse function estimator (\ref{mappingest}) is consistent with respect to the true parameter $\btheta^*$ to ensure valid downstream inference. \citet{LiangKS2024EFI} established this consistency using sparse deep learning theory \citep{SunSLiang2021}, but their approach was limited to settings where the dimension of $\btheta$ is fixed or increases very slowly with  the sample size $n$. 
In this work, we achieve the consistency by employing a narrow-neck $\bw$-network, which overcomes the limitation of the current framework of EFI.



To motivate the development of the narrow-neck $\bw$-network,  
we first note an important mathematical fact: As implied by (\ref{app:energyfunction11}), each 
$\hat{\btheta}_i \in \mathbb{R}^p$ in the EFI network tends to converge to a constant vector as $\epsilon \to 0$. Consequently, 
different components of $\hat{\btheta}$ become highly correlated across 
$n$ observations and $\hat{\btheta}$ 
can be effectively represented in a much lower-dimensional space, 
even though the dimension of $\btheta$ may be very high.
A straightforward solution to address this issue is to 
incorporate a restricted Boltzmann machine (RBM) \citep{Hinton2006ReducingTD} into the $\bw$-network, as illustrated in 
Figure \ref{fig:RBM}, where the neck layer is binary and the last two layers form a Gaussian-binary RBM \citep{Gu2022ApproximationPO,Chu2017RestrictedBM}. 
As discussed in \citet{LiangKS2024EFI}, $\hat{\btheta}$ can be treated as Gaussian random variables in the EFI network. 
The binary layer of the RBM serves as a dimensionality reducer for $\hat{\btheta}$.
Such a RBM-embedded $\bw$-network can be trained using 
the imputation-regularized optimization (IRO) algorithm \citep{Liang2018missing} in a similar way to that used in  \citet{Wu2019AccelerateTO}.

Since the primary role of the $\bw$-network is forward learning of 
$\btheta$, it can also be formulated as a  
 stochastic neural network (StoNet) \citep{LiangSLiang2022,SunLiang2022kernel}, where the binary layer can be extended to be continuous.  
 This StoNet-formulation leads to the following hierarchical model: 
 \begin{equation} \label{hiereq}
 \begin{split}
 \hat{\theta}_{i}^{(j)} &= \bm_i^T \bxi_j + e_{i,j},  \\ 
 \bm_i &= R(\bmu_i, \bv_i), \\ 
 \end{split}
 \end{equation} 
 where $i=1,2,\ldots,n$, $j=1,2,\ldots,p$, 
 $e_{i,j} \sim N(0,\sigma_{\hat{\theta}}^2)$, $\bv_i \in \mathbb{R}^{d_h}$ denotes a vector of random errors following a known distribution, $d_h$ denotes the width of the stochastic 
  neck layer, $h$ denotes the number of hidden layers of the $\bw$-network, 
 $\bmu_i=g(X_i,Y_i,Z_i;\bw_n^{(1)})$ is the mean of the feeding  
 vector to the stochastic neck layer, and $R(\cdot)$ represents a 
 transformation. 
 With a slight abuse of notation, we assume that $\bm_i$ has been augmented with  
 a constant component to account for the intercept term in the regression model for $\hat{\theta}_{i}^{(j)}$.  
 For $R(\cdot)$, we recommend the setting: 
\begin{itemize}
\item[(*)]  {\it The neck layer is stochastic with:
$R(\bmu_i,\bv_i)=\Psi(\bmu_i+\bv_i)$, 
where $\bv_i \sim N(0, \sigma_{v}^2)$ with a pre-specified value of $\sigma_v^2$, and 
$\Psi$ is an (element-wise) activation function. }
\end{itemize} 

 Under this setting, the $\bw^{(1)}$-network forms a nonlinear Gaussian regression with response $\bmu_i+\bv_i$ for $i=1,2,\ldots,n$.
Conceptually, the StoNet can be trained using the
stochastic EM algorithm \citep{Celeux1996StochasticVO} 
by iterating  between the  steps: 
(i) {\it Latent variable imputation}: 
  Impute the latent variables $\{\bv_i: i=1,\ldots,n\}$ and  
  $\bZ_n=\{z_1,\ldots,z_n\}$ conditioned on the current
 estimates of $\bw_n=\{\bw_n^{(1)},\bw_n^{(2)}\}$, where $\bw_n^{(2)}=\{\bxi_j: j=1,2,\ldots,p\}$. 
 (ii) {\it Optimization}: Conditioned 
 on the imputed latent variables, update the 
  estimates of $\bw_n^{(1)}$ and $\bw_n^{(2)}$ separately.
 Specifically, $\bw_n^{(1)}$ can be estimated by training 
 the $\bw_n^{(1)}$-network using SGD, and $\bw_n^{(2)}$ can be estimated by performing  $p$ linear regressions as specified in (\ref{hiereq}).

Following the standard theory of the stochastic EM algorithm \citep{Nielsen2000, Liang2018missing},  $\{\bw_n^{(1)},\bw_n^{(2)}\}$ will converge to a solution 
 to  the equation:
$\nabla_{\bw_n} \log \pi_{\epsilon}(\bw_n|\bX_n,\bY_n) 
=\int  \Big[\nabla_{\bw_n} \log \pi_{\epsilon}(\bw_n|\bX_n,\bY_n,\bZ_n,\bV_n) 
\pi_{\epsilon}(\bZ_n,\bV_n|\bX_n,\bY_n,\bw_n)\Big] d\bZ_n d \bV_n=0$,     
as the sample size $n$ and the number of iterations of the algorithm become 
large, where the prior $\pi(\bw_n) \propto 1$ and $\bV_n=(\bv_1, \bv_2, \ldots, \bv_n)^T$. Denote the converged solution by $\bw_n^*=\{\bw_n^{*(1)},\bw_n^{*(2)}\}$. 
The consistency of the resulting inverse function estimator can be 
established by leveraging the sufficient dimension reduction property of
the StoNet \citep{LiangSLiang2022} and the prediction property 
of linear regressions.  
 Specifically, $\{\bm_i:i=1,\ldots,n\}$ serves as a sufficient dimension reduction for $\{(X_i, Y_i, Z_i):i=1,\ldots,n\}$. 
 The consistency of the inverse mapping 
 can thus be ensured by the linear relationship   $\hat{\btheta}_i \sim \bm_i$, as described in (\ref{hiereq}), along with the prediction property of linear regression. Note that the linear relationship   $\hat{\btheta}_i \sim \bm_i$ can be generally ensured by the universal approximation ability 
 of the DNN. 
 
 Toward a rigorous mathematical development, we impose the following 
 conditions: 
\begin{equation} \label{w2req} 
\begin{split}
(i) \ \lambda_{\min}(\mathbb{E}(\bm_i\bm_i^T)) \geq \rho_{\min}:=c' \sigma_v^2+\delta>0, \quad 
(ii) \  
 \lambda_{\max} (\mathbb{E}(\bm_i\bm_i^T))  \leq \rho_{\max} <\infty,
\end{split}
\end{equation}
where $\mathbb{E}(\cdot)$ denotes expectation;
$\lambda_{\min}(\cdot)$ and $\lambda_{\max}(\cdot)$ denote the minimum and maximum eigenvalue of a matrix, respectively;
and $c'> 0$, $\delta>0$,
and $\rho_{\max}>0$ are some constants. 
Condition (i) is generally satisfied by choosing a sufficiently narrow 
neck layer, in particular, we set $d_h \prec n$. 
Condition (ii) is justified in Appendix \ref{sect:justification}.

By the asymptotic equivalence between the StoNet and conventional
DNN \citep{LiangSLiang2022} (see also Appendix \ref{sect:proofEqu}), 
the StoNet can be trained by directly training the DNN. Based on this asymptotic equivalence, 
Theorem \ref{thm:1} establishes the consistency of the inverse mapping 
learned by Algorithm \ref{EFIalgorithm} for large-scale models, see 
Appendix \ref{sect:proof3.1} for the proof. 
 
\begin{theorem} \label{thm:1} 
Suppose that the narrow neck layer is set as in (*) with 
$\sigma_v^2 \prec \frac{\epsilon}{\eta h d_h p}$,  
the activation function $\Psi(\cdot)$ is $c$-Lipschitz continuous, 
and $d_h \prec n$ is sufficiently small such that the conditions 
in (\ref{w2req}) are satisfied while admitting a non-empty zero energy set 
$\mathcal{Z}_n$. 
 Additionally, assume that the other regularity conditions (Assumptions \ref{ass:1}-\ref{ass:2} in Appendix \ref{sect:proof}) hold.
If  $\epsilon \prec \min\{ \frac{n}{p^2 d_h^2}, \frac{h}{p d_h}\}$,
then the inverse mapping $\btheta=G(\bY_n,\bX_n,\bZ_n)$, learned by Algorithm \ref{EFIalgorithm} with a narrow neck $\bw$-network, is consistent. 
\end{theorem}

This narrow neck setting for the $\bw$-network eliminates the need to specify a prior for $\bw_n^{(1)}$. 
Notably, under this setting, the resulting estimate 
$\btheta$ is not necessarily sparse, introducing a new research paradigm for high-dimensional problems; moreover,  
 the  $\bw_n^{(1)}$-network does not need to be excessively large, 
keeping the overall size of the $\bw$-network manageable. This facilitates the application of EFI to high-dimensional and complex models.

\section{Simulation Studies}

We compared EFI with Bayesian and dropout methods across multiple simulation studies, including the Poisson equation under various settings and the Black-Scholes model. Below, we present the results for the 1-D Poisson equation, with results from other studies provided in Appendix \ref{sect:AdditionalStudy}. The experimental settings for all simulation studies are detailed in Appendix \ref{sect:settings}. 


\subsection{1-D Poisson Equation}

Consider a 1-D Poisson equation as in \citet{BPINN}: 
 \begin{equation} \label{poissoneq}
 \beta \frac{\partial^2 u}{\partial x^2} =f, \quad \mbox{$x \in \Omega$},
 \end{equation}
 where $\Omega=[-0.7,0.7]$, $\beta=0.01$, $u=\sin^3(6x)$, and $f$ can be 
 derived from \eqref{poissoneq}. Here, we assume the analytical expression of $f$ is unavailable; instead, 200 sensor measurements of $f$ are available with the sensors equidistantly distributed across $\Omega$. 
Additionally, there are two sensors at $x=-0.7$ and $x=0.7$ to provide the left/right  Dirichlet boundary conditions for $u$. 
We model the boundary noise as $z_i^u \sim N(0,0.05^2)$ for $i=1,2,\ldots,20$, with 10 observations drawn from each boundary sensor. For the interior domain, we assume noise-free measurements of $f$.
This simulation is repeated over 100 independent datasets to evaluate the robustness and reliability of the proposed approach.

 The Bayesian method was first applied to this example as in \citet{BPINN} using a two-hidden-layer DNN, with 50 hidden units 
in each layer, to approximate $u(\bx)$. 
\textcolor{black}{The same network architecture was also used for this example 
by other methods as shown in Table \ref{table:1d poisson}.} 
The negative log-posterior is given as follows: 
 \begin{equation}
 \label{eq:bpinn loglikelihood}
 \begin{split}
\ell(\bvartheta) &= \sum_{i=1}^{n_u} \frac{\|u_i-u_{\bvartheta}(\bx_i^u)\|^2}{2\sigma_u^2} 
+ \sum_{i=1}^{n_f} \frac{\|
f_i-\mathcal{F}(u_{\bvartheta}(\bx_i^f); \bbeta) \|^2}{2\sigma_f^2}  
+ \sum_{i=1}^{n_b} \frac{\| b_i -\mathcal{B}(u_{\bvartheta}(\bx_i^b))\|^2}{2\sigma_b^2} \\
& \quad - \log(\pi(\bvartheta))+C,     
 \end{split}
\end{equation}
where $C$ denotes the log-normalizing constant, $n_u=20$, $n_f=200$, $n_b=0$, and $\sigma_u=0.05$.
The prior $\pi(\bvartheta)$
is specified as in \citet{BPINN}, where each unknown parameter is 
assigned an independent standard Gaussian distribution. The same prior is used across all Bayesian simulations in this paper. We implement the method using 
the Python package {\it hamiltorch} \citep{cobb2020scaling}.
Since $f$ is observed exactly, the corresponding variance $\sigma_f$ should theoretically be zero. However, the formulation in (\ref{eq:bpinn loglikelihood}) does not allow $\sigma_f=0$, necessitating experimentation with different non-zero values of $\sigma_f$. These variations resulted in different widths of confidence intervals, as illustrated in Figure \ref{fig:bpinn_traj}. Notably, there is no clear guideline for selecting $\sigma_f$ to achieve the appropriate interval width necessary for the desired coverage rate. This ambiguity underscores the dilemma (mentioned in Introduction) inherent in the Bayesian method. Specifically, the $f$-term in (\ref{eq:bpinn loglikelihood}) functions as part of the prior for the parameter $\bvartheta$.


A similar issue arises when selecting the dropout rate for PINNs. As illustrated in Figure \ref{fig:pinn_traj}, the uncertainty estimation in PINNs is highly sensitive to the choice of the dropout rate. When the dropout rate approaches zero, the confidence interval collapses into a point estimate, failing to capture uncertainty. Conversely, excessively large dropout rates introduce significant bias, leading to unreliable interval estimates. This highlights the challenge of determining an appropriate dropout rate to balance bias and variability in uncertainty quantification. 
\textcolor{black}{As a possible way to address this issue, Concrete Dropout \citep{gal2017concrete} was implemented for the example with the code provided at
\url{https://github.com/yaringal/ConcreteDropout}. Unlike the conventional dropout method using fixed dropout rate, Concrete Dropout treats the dropout rate as a hyperparameter and learns it simultaneously when training the model.}


\begin{figure*}[!ht]
    \centering
    \begin{subfigure}{0.3\textwidth}
        \centering
        \includegraphics[width=\textwidth]{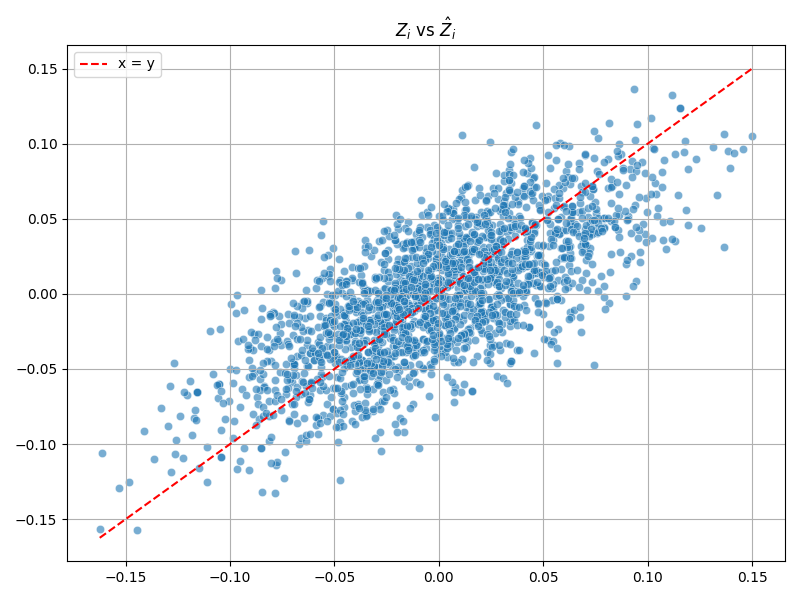}
        \caption{Imputed errors vs true errors}
    \end{subfigure}
    \hfill
    \begin{subfigure}{0.3\textwidth}
        \centering
        \includegraphics[width=\textwidth]{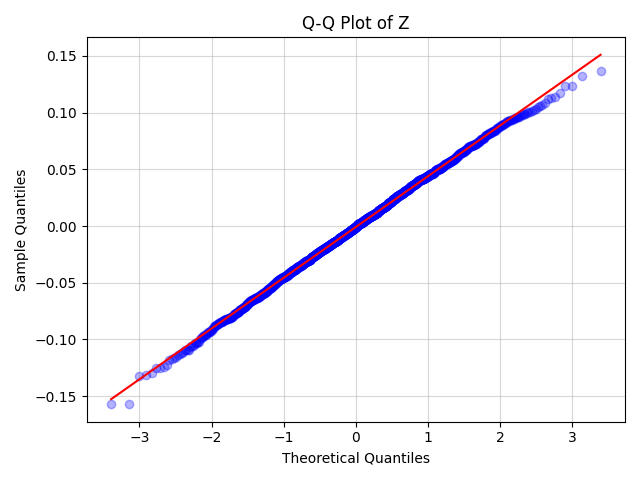}
        \caption{QQ-plot of imputed errors}
    \end{subfigure}
    \hfill
    \begin{subfigure}{0.3\textwidth}
        \centering
        \includegraphics[width=\textwidth]{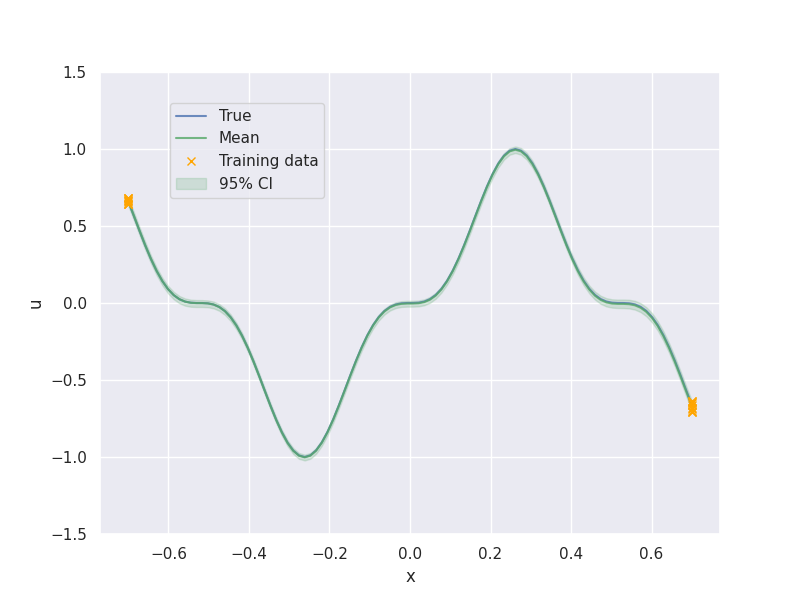}
        \caption{Confidence interval}
    \end{subfigure}
    \caption{EFI-PINN diagnostic for 1D-Poisson}
    \label{fig:efi poisson}
\end{figure*}

 We then applied EFI to this example, using the same DNN structure 
 for the data-modeling network.
 Figure \ref{fig:efi poisson} summarizes the imputed random errors by 
 EFI for 100 datasets, showing that EFI correctly imputes the realized random errors from the observations.  Consequently, EFI achieves the correct recovery of the underlying physical law.

 \begin{table*}[!ht]
  \caption{Metrics for 1D-Poisson, averaged over 100 runs.}
  \label{table:1d poisson}
  \centering
  \begin{tabular}{lccccc}
    \toprule
    Method    & MSE    & Coverage Rate     & CI-Width \\
    \midrule
        PINN (no dropout) &  0.000121 (0.000012) & 0.0757 (0.011795) & 0.002139 (0.000024) \\
        Dropout (0.5\%)  & 0.000228 (0.000024) & 1.0000 (0.000000) & 0.191712 (0.004367) \\
        Dropout (1\%)  & 0.000851 (0.000111) & 0.9999 (0.000104) & 0.274651 (0.006734) \\
        Dropout (5\%) & 0.006276 (0.000982) & 0.9893 (0.004024) & 0.660927 (0.030789) \\
        Concrete Dropout   & 0.000184 (0.000020) & 0.1185 (0.018002) & 0.003706 (0.000073) \\
        Bayesian ($\sigma_f=0.05$)  & 0.000826 (0.000102) & 0.9997 (0.000302) & 0.372381 (0.002635) \\
        Bayesian ($\sigma_f=0.005$)  & 0.000170 (0.000020) & 0.9914 (0.008600) & 0.081557 (0.000692) \\
        Bayesian ($\sigma_f=0.0005$)  & 0.001460 (0.000219) & 0.5884 (0.032390) & 0.055191 (0.001123) \\
        \bf{EFI}  & \bf{0.000148 (0.000016)} & \bf{0.9517 (0.014186)} & \bf{0.050437 (0.000462)} \\
    \bottomrule
  \end{tabular}
\end{table*}

 We evaluated the accuracy and robustness of each method by recording three metrics, as summarized in Table \ref{table:1d poisson}, based on an average across 100 repeated experiments. The `MSE' measures the mean squared distance between the network prediction and the true solution. The coverage rate represents the percentage of the true solution contained within the 95\% confidence interval, while the CI-Width quantifies the corresponding interval width. 
 As discussed above, the dropout method produces inflated coverage rates and   excessively wide confidence intervals, with CI widths ranging from 0.1917 to 0.6609. Similarly, the Bayesian method is sensitive to the choice of $\sigma_f$, with larger $\sigma_f$ values yielding inflated coverage rates,  while smaller $\sigma_f$ values result in underestimated coverage rates.
 In contrast, EFI is free of hyperparameter tuning, making inference
 on the model parameters based  on the observations and the embedded 
 physical law.  It produces the most balanced results, with a mean squared error (MSE) of $1.48\times 10^{-4}$ and a coverage rate of 95.17\%, closely aligning with the nominal 95\% target. Moreover, it provides the shortest CI-width  0.0504. 
 \textcolor{black}{Concrete Dropout achieves a comparable MSE as EFI, indicating similar model estimation accuracy. However, it performs significantly worse in terms of uncertainty quantification, with a coverage rate of only 11.85\%, far below the nominal 95\%. This substantial under-coverage is due to an underestimation of predictive uncertainty, as evidenced by the markedly narrower confidence interval width.}

\section{Real Data Examples}

This section demonstrates the abiliy of the EFI-PINN framework to quantify uncertainty for models learned from real data. We considered two models: the Montroll growth model, 
\begin{equation} \label{Montrollmodel}
    \frac{du}{dt} = k \cdot u \cdot \left(1 - \left(\frac{u}{C}\right)^\theta \right),
\end{equation}
where \(k\), \(C\), and \(\theta\) are unknown parameters; and  
the reaction-diffusion model governed by the generalized 
Porous-Fisher-Kolmogoriv-Petrovsky-Piskunov (P-FKPP) equation: 
\begin{equation} \label{G-P-FKPPmodel}
\frac{du}{dt}=D \frac{\partial}{\partial x} \left[ \left(\frac{u}{K}\right)^m \frac{\partial u}{\partial x}\right]  + ru\left( 1-\frac{u}{K} \right),
\end{equation}
where $D$, $m$, and $r$ are unknown parameters. 
We modeled the Chinese hamster V79 fibroblast tumor cell growth data 
\citep{Montroll} using equation (\ref{Montrollmodel}), and the 
scratch assay data \citep{Jin2006EstimatingTN} using equation (\ref{G-P-FKPPmodel}). 
Additional details about the datasets are provided in the Appendix.
Figure \ref{fig:2PDEmodels} illustrates that the proposed EFI-PINN framework is not only capable of learning the PDE models from data but also effectively quantifying the uncertainty of the models. Further results and discussions can be found in the Appendix. The prediction uncertainty can also be quantified with the proposed method.





\begin{figure}[h]
    \centering
 \begin{tabular}{cc}
        \includegraphics[width=0.45\linewidth,height=1.65in]{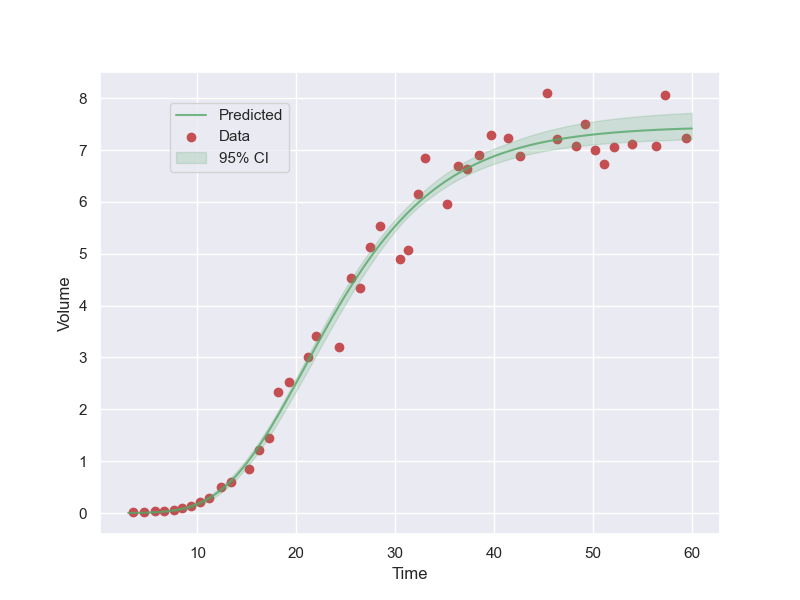} & 
        \includegraphics[width=0.45\linewidth,height=1.5in]{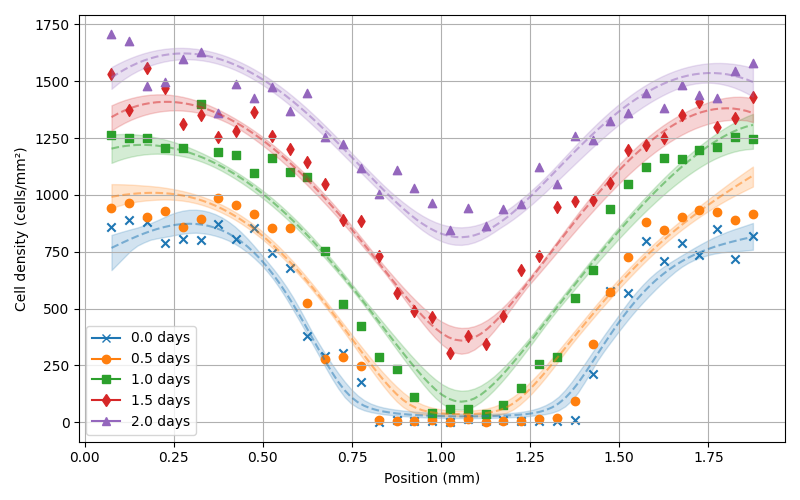}
\end{tabular} 
    \caption{ 
    Confidence bands (shaded areas) for the learned models: 
    (left): Montroll growth model; (right): generalized P-FKPP model. 
    }
    \label{fig:2PDEmodels}
\end{figure}

\section{Conclusion}
This paper presents a novel theoretical framework for EFI, enabling effective uncertainty quantification for PINNs. EFI addresses the challenge of uncertainty quantification through a unique approach of “solving data-generating equations,” transforming it into an objective process that facilitates the construction of honest confidence sets. In contrast, existing methods such as dropout and Bayesian approaches rely on subjective hyperparameters (e.g., dropout rates or priors), undermining the honesty of the resulting confidence sets.
This establishes a new research paradigm for statistical inference of complex models, with the potential to significantly impact the advancement of modern data science.


 Although Algorithm \ref{EFIalgorithm} performs well in our examples, its efficiency can be further improved through several enhancements. For instance, in the latent variable imputation step, the  SGLD algorithm \citep{Welling2011Bayesian} can be replaced with the Stochastic Gradient Hamiltonian Monte Carlo \citep{SGHMC2014}, which offers better sampling efficiency. Similarly, in the parameter updating step, the SGD algorithm can be accelerated with momentum.
Momentum-based algorithms facilitate faster convergence to a good local minimum during the early stages of training. As training progresses, the momentum can be gradually reduced to zero, ensuring alignment with the  convergence theory of EFI. This approach balances computational efficiency with theoretical rigor, enhancing the overall performance of EFI.


While the examples presented focus on relatively small-scale models, this is not a limitation of the approach. In principle, EFI can be extended to large-scale models, such as ResNets and CNNs, through transfer learning—a direction we plan to explore in future work.
 
 \section*{Acknowledgments}
  Liang's research is supported in part by the NSF grant DMS-2210819 and the NIH grant R01-GM152717.\\
  Shih's research is partially supported by MSK Cancer Center Support Grant/Core Grant (P30 CA008748).

\newpage
\bibliography{reference}
\bibliographystyle{asa}


\appendix


\newpage
\section*{NeurIPS Paper Checklist}

\begin{enumerate}

\item {\bf Claims}
    \item[] Question: Do the main claims made in the abstract and introduction accurately reflect the paper's contributions and scope?
    \item[] Answer: \answerYes{} 
    \item[] Justification: The abstract states that we provide theoretical proof for the proposed framework and numerical experiment for overcoming limitations for Bayesian PINN and Dropout PINN.

\item {\bf Limitations}
    \item[] Question: Does the paper discuss the limitations of the work performed by the authors?
    \item[] Answer: \answerYes{} 
    \item[] Justification: In the conclusion, we mentioned that the examples we presented are relatively small scale models, however, the proposed algorithm can be extended easily to large scale models through transfer learning.

\item {\bf Theory assumptions and proofs}
    \item[] Question: For each theoretical result, does the paper provide the full set of assumptions and a complete (and correct) proof?
    \item[] Answer: \answerYes{} 
    \item[] Justification: In this paper, we provide convergence theory for proposed algorithm.

    \item {\bf Experimental result reproducibility}
    \item[] Question: Does the paper fully disclose all the information needed to reproduce the main experimental results of the paper to the extent that it affects the main claims and/or conclusions of the paper (regardless of whether the code and data are provided or not)?
    \item[] Answer: \answerYes{} 
    \item[] Justification: We provide our source code and hyperparameters in the supplementary material.

\item {\bf Open access to data and code}
    \item[] Question: Does the paper provide open access to the data and code, with sufficient instructions to faithfully reproduce the main experimental results, as described in supplemental material?
    \item[] Answer: \answerYes{} 
    \item[] Justification: We provide our source code and hyperparameters in the supplementary material.

\item {\bf Experimental setting/details}
    \item[] Question: Does the paper specify all the training and test details (e.g., data splits, hyperparameters, how they were chosen, type of optimizer, etc.) necessary to understand the results?
    \item[] Answer: \answerYes{} 
    \item[] Justification: We provide our source code and hyperparameters in the supplementary material.

\item {\bf Experiment statistical significance}
    \item[] Question: Does the paper report error bars suitably and correctly defined or other appropriate information about the statistical significance of the experiments?
    \item[] Answer: \answerYes{} 
    \item[] Justification: All our experiments are replicated 100 times with randomly chosen random seeds, each metic is averaged over all 100 runs to ensure statistical significance.

\item {\bf Experiments compute resources}
    \item[] Question: For each experiment, does the paper provide sufficient information on the computer resources (type of compute workers, memory, time of execution) needed to reproduce the experiments?
    \item[] Answer: \answerYes{}
    \item[] Justification: We provide the wall-clock time for the Poisson-1D experiment with different algorithms in the supplementary material.
    
\item {\bf Code of ethics}
    \item[] Question: Does the research conducted in the paper conform, in every respect, with the NeurIPS Code of Ethics \url{https://neurips.cc/public/EthicsGuidelines}?
    \item[] Answer: \answerYes{} 
    \item[] Justification: The research relies solely on simulated and publicly available open-source data. No human, animal, or private data were used. The methods developed and evaluated pose no foreseeable ethical concerns and comply fully with the NeurIPS Code of Ethics.

\item {\bf Broader impacts}
    \item[] Question: Does the paper discuss both potential positive societal impacts and negative societal impacts of the work performed?
    \item[] Answer: \answerYes{} 
    \item[] Justification: This work develops a foundational algorithm for uncertainty quantification in PINNs. While it does not involve sensitive data or direct deployment, it may support future applications in scientific modeling and engineering.
    
\item {\bf Safeguards}
    \item[] Question: Does the paper describe safeguards that have been put in place for responsible release of data or models that have a high risk for misuse (e.g., pretrained language models, image generators, or scraped datasets)?
    \item[] Answer: \answerNA{} 
    \item[] Justification: The paper does not release any pretrained models or datasets with potential for misuse. All experiments are conducted using simulation or publicly available open-source data, and the developed methods pose no foreseeable dual-use concerns.

\item {\bf Licenses for existing assets}
    \item[] Question: Are the creators or original owners of assets (e.g., code, data, models), used in the paper, properly credited and are the license and terms of use explicitly mentioned and properly respected?
    \item[] Answer: \answerYes{} 
    \item[] Justification: All external assets used in this work—such as open-source code and datasets—are properly credited in the paper. We ensured that their licenses (e.g., MIT, Apache 2.0) were respected, and we include relevant citations and links where applicable.

\item {\bf New assets}
    \item[] Question: Are new assets introduced in the paper well documented and is the documentation provided alongside the assets?
    \item[] Answer: \answerNA{} 
    \item[] Justification: The paper does not release new assets yet.

\item {\bf Crowdsourcing and research with human subjects}
    \item[] Question: For crowdsourcing experiments and research with human subjects, does the paper include the full text of instructions given to participants and screenshots, if applicable, as well as details about compensation (if any)? 
    \item[] Answer: \answerNA{} 
    \item[] Justification: This paper does not involve any crowdsourcing or research with human subjects.

\item {\bf Institutional review board (IRB) approvals or equivalent for research with human subjects}
    \item[] Question: Does the paper describe potential risks incurred by study participants, whether such risks were disclosed to the subjects, and whether Institutional Review Board (IRB) approvals (or an equivalent approval/review based on the requirements of your country or institution) were obtained?
    \item[] Answer: \answerNA{} 
    \item[] Justification:  This paper does not involve research with human subjects or any activity requiring IRB or equivalent approval.

\item {\bf Declaration of LLM usage}
    \item[] Question: Does the paper describe the usage of LLMs if it is an important, original, or non-standard component of the core methods in this research? Note that if the LLM is used only for writing, editing, or formatting purposes and does not impact the core methodology, scientific rigorousness, or originality of the research, declaration is not required.
    \item[] Answer: \answerNA{} 
    \item[] Justification: This paper does not involve the use of large language models (LLMs) as part of the core methodology. Any use of LLMs was limited to minor writing or editing support and did not influence the scientific content or originality of the research.
\end{enumerate}

\newpage

\appendix

{\Large\bf Appendix}

\setcounter{section}{0}
\renewcommand{\thesection}{A\arabic{section}}
\setcounter{table}{0}
\renewcommand{\thetable}{A\arabic{table}}
\setcounter{equation}{0}
\renewcommand{\theequation}{A\arabic{equation}}
\setcounter{figure}{0}
\renewcommand{\thefigure}{A\arabic{figure}}

This appendix is organized as follows. Section \ref{sect:sup:EFI} provides a brief review of the EFI method and its self-diagnostic property. Section \ref{sect:EFDlimit} presents the analytical expression of the extended fiducial density (EFD) function of $\bZ_n$. Section \ref{sect:narrow-net} includes an illustrative plot of the proposed narrow-neck $\bw$-network. Section \ref{sect:proof} gives the proof of Theorem \ref{thm:1}. Section \ref{sect:AdditionalStudy} presents additional numerical results. Finally, Section \ref{sect:settings} details the parameter settings used in our numerical experiments.

\section{A Brief Review of the EFI Method and Its Self-Diagnosis} \label{sect:sup:EFI} 


\subsection{The EFI Method}

To ensure a smooth presentation of the EFI method, this section partially overlaps with Section \ref{sect:EFI} of the main text.
Consider a regression model: 
\[
Y=f(\bX,Z,\btheta),
\]
where $Y\in \mathbb{R}$ and $\bX\in \mathbb{R}^{d}$ represent the response and explanatory variables, respectively; $\btheta\in \mathbb{R}^p$ represents the vector of  parameters; 
and $Z\in \mathbb{R}$ represents a scaled random error following  
 a known distribution $\pi_0(\cdot)$.  
Suppose that a random sample of size $n$,  denoted by $\{(y_1,\bx_1), (y_2,\bx_2),\ldots,(y_n,\bx_n)\}$, has been collected from the model. In structural inference, the observations can be expressed in data-generating equations as follows:  
\begin{equation} \label{app:dataGeneqg}
y_i=f(\bx_i,z_i,\btheta), \quad i=1,2,\ldots,n.
\end{equation}
This system of equations consists of $n+p$ unknowns, namely, $\{\btheta, z_1, z_2, \ldots, z_n
\}$, while there are only $n$ equations. Therefore, the values of $\btheta$ cannot be uniquely determined by the data-generating equations, and this lack of uniqueness of unknowns introduces uncertainty in $\btheta$.  

Let $\bZ_n=\{z_1,z_2,\ldots,z_n\}$ denote the unobservable random errors contained in the data,
which are also called latent variables in EFI.  
Let $G(\cdot)$ denote an inverse function/mapping for   $\btheta$, i.e., 
\begin{equation} \label{app:Inveq}
\btheta=G(\bY_n,\bX_n,\bZ_n).
\end{equation}
It is worth noting that the inverse function is generally non-unique. For example, it can be constructed by solving any $p$ equations in (\ref{dataGeneqg}) for $\btheta$. 
As noted by \citet{LiangKS2024EFI}, this non-uniqueness of inverse function 
mirrors the flexibility of frequentist methods, where different 
estimators of $\btheta$ can be constructed to achieve desired properties 
such as efficiency, unbiasedness, and robustness. 



 \begin{figure}[htbp]
     \centering
   \includegraphics[width=0.8\textwidth]{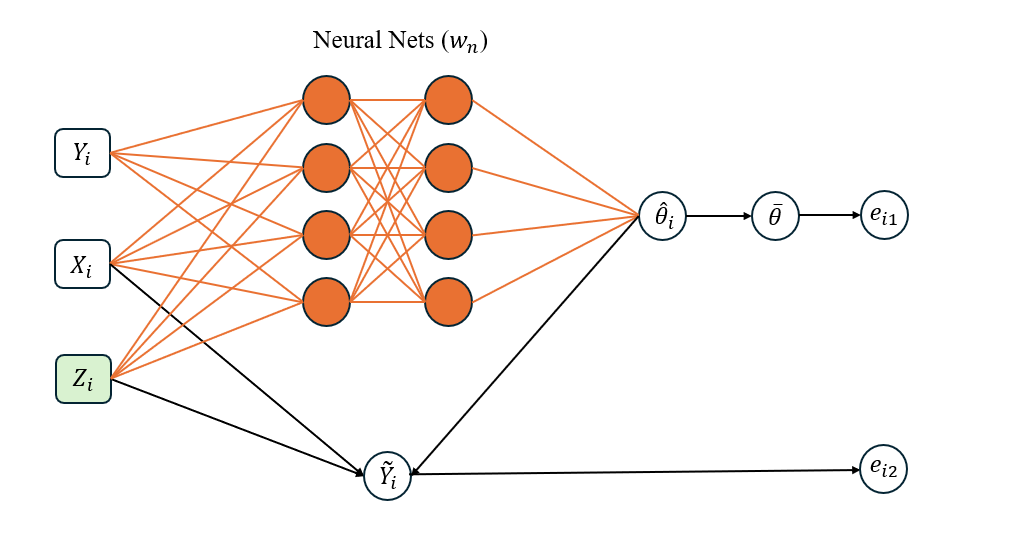}
  \caption{Diagram of EFI given in \citet{LiangKS2024EFI}: 
  The orange nodes and orange links form a deep neural network (DNN), referred to as the $\bw$-network, which is parameterized by $\bw_n$ (with the subscript $n$ indicating its dependence on the training sample size $n$); the green node represents the latent variable to impute; and  the black lines represent deterministic functions.}
     \label{EFInetwork}
\end{figure}

Since the inverse function $G(\cdot)$ is generally unknown, \citet{LiangKS2024EFI} proposed to approximate it using a sparse DNN, see Figure \ref{EFInetwork}  for illustration.
The EFI network  has two output nodes defined, respectively,  by 
\begin{equation} \label{app:outputeq}
\begin{split} 
&e_{i,1} :=\|\hat{\btheta}_i-\bar{\btheta}\|^2, \\
\quad \mbox{and} \quad 
&e_{i,2} :=d(y_i,\tilde{y}_i):=d(y_i,\bx_i, z_i, \bar{\btheta}), 
\end{split}
\end{equation}
where $\tilde{y}_i=f(\bx_i,z_i,\bar{\btheta})$,  $f(\cdot)$ is as specified in (\ref{dataGeneqg}), and $d(\cdot)$ is a function measuring the difference between $y_i$ and $\tilde{y}_i$. That is, the node $e_{i,1}$ quantifies 
the variation of $\hat{\btheta}_i$, while the node $e_{i,2}$ represents the 
fitting error. For a normal linear/nonlinear regression, 
$d(\cdot)$ can be defined as 
\begin{equation} \label{app:deq0}
d(y_i,\bx_i,z_i,\bar{\btheta})=\|y_i-f(\bx_i,z_i,\bar{\btheta})\|^2.
\end{equation}
For logistic regression, it is defined as a ReLU function, see  \citet{LiangKS2024EFI} for details.


 
Let $\hat{\btheta}_i:=\hat{g}(y_i,\bx_i,z_i,\bw_n)$ denote the DNN prediction function parameterized by the weights $\bw_n$ in the EFI network, and let 
\begin{equation} \label{app:thetabareq}
\bar{\btheta}:=\frac{1}{n} \sum_{i=1}^n \hat{\btheta}_i=\frac{1}{n} \sum_{i=1}^n \hat{g}(y_i,\bx_i,z_i,\bw_n),
\end{equation}
which serves as an estimator of the inverse function $G(\cdot)$.  

EFI defines an energy function 
\begin{equation} \label{app:energyfunction11}
 U_n(\bY_n,\bX_n,\bZ_n,\bw_n) = \eta_{\theta} \sum_{i=1}^n\| \hat{\btheta}_i- \bar{\btheta} \|^2 + \sum_{i=1}^n d(y_i,\bx_i,z_i,\bar{\btheta}),    
\end{equation}
where $\eta_{\theta}>0$ is a  regularization parameter, $\hat{\btheta}_i$'s and $\bar{\btheta}$ can be expressed as functions of $(\bY_n,\bX_n,\bZ_n,\bw_n)$,
and $d(\cdot)$ is a function measuring the difference 
between $y_i$ and $\tilde{y}_i$.
The likelihood function is given by 
\begin{equation} \label{app:likelihoodeq} 
\pi_{\epsilon}(\bY_n|\bX_n,\bZ_n,\bw_n) \propto  e^{- U_n(\bY_n,\bX_n,\bZ_n,\bw_n)/\epsilon},
\end{equation}
for some constant $\epsilon$ close to 0. 
As discussed in \citet{LiangKS2024EFI}, the choice of $\eta_{\theta}$ does not affect much on the performance of EFI as long as 
$\epsilon$ is sufficiently small. 
Subsequently, the posterior of $\bw_n$ is given by 
\[
\begin{split}
\pi_{\epsilon}(\bw_n|\bX_n,\bY_n,\bZ_n) &  
\propto \pi(\bw_n) e^{-U_n(\bY_n,\bX_n,\bZ_n,\bw_n)/\epsilon }, 
 \end{split} 
 \]
where $\pi(\bw_n)$ denotes the prior  of $\bw_n$;
and the predictive distribution of $\bZ_n$ is given by  
\[
\begin{split}
 \pi_{\epsilon}(\bZ_n|\bX_n,\bY_n,\bw_n) &  
 \propto \pi_0^{\otimes n}(\bZ_n) e^{-U_n(\bY_n,\bX_n,\bZ_n,\bw_n)/\epsilon},
 \end{split} 
 \]
where $\pi_0^{\otimes n}(\bZ_n)=\prod_{i=1}^n \pi_0(z_i)$ under the assumption that $z_i$'s are independently identically distributed (i.i.d.).
 In EFI, $\bw_n$ is estimated through maximizing the posterior 
 $\pi_{\epsilon}(\bw_n|\bX_n,\bY_n)$ given the observations $\{\bX_n,\bY_n \}$.
  By the Bayesian version of Fisher's identity \citep{SongLiang2020eSGLD}, 
 the gradient equation $\nabla_{\bw_n} \log \pi_{\epsilon}$ $ (\bw_n|\bX_n,\bY_n)=0$ can be re-expressed as
\begin{equation} \label{app:identityeq}
\begin{split}
\nabla_{\bw_n} \log \pi_{\epsilon}(\bw_n|\bX_n,\bY_n) &=\int \Big[\nabla_{\bw_n} \log \pi_{\epsilon}(\bw_n|\bX_n,\bY_n,\bZ_n) \\ 
 & \quad \times \pi_{\epsilon}(\bZ_n|\bX_n,\bY_n,\bw_n)\Big] d\bZ_n=0,     
\end{split}
\end{equation} 
which can be solved using  an adaptive stochastic gradient MCMC algorithm 
\citep{LiangSLiang2022,deng2019adaptive}.
The algorithm works by iterating between the {\it latent variable imputation} and {\it parameter updating} steps,
see Algorithm \ref{EFIalgorithm} for the pseudo-code.
 This algorithm is termed ``adaptive'' because the transition kernel in the latent variable imputation step changes with the working parameter estimate of  
 $\bw_n$. The parameter updating step can be implemented using mini-batch SGD, and the latent variable imputation step can be executed in parallel for each observation $(y_i,\bx_i)$. Hence, the algorithm is scalable with respect to  large datasets.

Under mild conditions for the adaptive SGLD algorithm, it can be shown that  
\begin{equation} \label{app:wconvergence}
\|\bw_n^{(k)} -\bw_n^* \|\stackrel{p}{\to} 0, \quad \mbox{as $k\to \infty$},
\end{equation}
 where $\bw_n^*$ denotes a solution to equation (\ref{identityeq}) and $\stackrel{p}{\to}$ denotes convergence in probability, and  that 
\begin{equation} \label{app:Zconvergence}
\bZ_n^{(k)} \stackrel{d}{\rightsquigarrow} \pi_{\epsilon}(\bZ_n|\bX_n,\bY_n,\bw_n^*), 
\quad \mbox{as $k \to \infty$},
\end{equation}
in 2-Wasserstein distance, where  $\stackrel{d}{\rightsquigarrow}$ denotes weak convergence. 
To study the limit of (\ref{Zconvergence}) as $\epsilon$ decays to 0, i.e., 
\[
p_n^*(\bz|\bY_n,\bX_n,\bw_n^*)=  
\lim_{\epsilon \downarrow 0} \pi_{\epsilon}(\bZ_n|\bX_n,\bY_n,\bw_n^*),
\]
where 
$p_n^*(\bz|\bY_n,\bX_n,\bw_n^*)$ 
is referred to as the extended fiducial 
density (EFD) of $\bZ_n$, 
\citet{LiangKS2024EFI} impose specific conditions on the structure of the $\bw$-network, including that the $\bw$-network is sparse 
and that the output layer width (i.e., the dimension of $\btheta$) is either fixed or grows very slowly with the sample size $n$. 
Under these assumptions, they prove the consistency of $\bw_n^*$ based on the sparse deep learning theory developed in \citet{SunSLiang2021}. This consistency  further implies that 
 \begin{equation} \label{app:mappingest}
 G^*(\bY_n,\bX_n,\bZ_n)= \frac{1}{n} \sum_{i=1}^n \hat{g}(y_i,\bx_i,z_i,\bw_n^*),
 \end{equation}
 serves as a consistent estimator for the inverse 
function/mapping $\btheta=G(\bY_n,\bX_n,\bZ_n)$. 
Refer to Appendix \ref{sect:EFDlimit} for the analytic expression of $p_n^*(\bz|\bY_n,\bX_n,\bw_n^*)$. 



Let $\mathcal{Z}_n=\{\bz \in \mathbb{R}^n: U_n(\bY_n,\bX_n,\bZ_n, \bw_n^*)=0\}$  
denote the zero-energy set.
Under some regularity conditions on the energy function, \citet{LiangKS2024EFI} proved that $\mathcal{Z}_n$ is invariant to the choice of $G(\cdot)$. 
Let $\Theta:=\{\btheta \in \mathbb{R}^p: \btheta=G^*(\bY_n,\bX_n,\bz), \bz\in \mathcal{Z}_n\}$ denote the parameter space of the target model, which represents the set of all possible values of $\btheta$ that $G^*(\cdot)$ takes when $\bz$ runs over $\mathcal{Z}_n$.  
Then, for any function $b(\btheta)$ of 
interest, its EFD $\mu_n^*(\cdot|\bY_n,\bX_n)$ associated with 
 $G^*(\cdot)$ is given by 
\begin{equation} \label{app:EFDeq}
\mu_n^*(B|\bY_n,\bX_n) =\int_{\mathcal{Z}_n(B)} d P_n^*(\bz|\bY_n,\bX_n,\bw_n^*), 
\end{equation}
for any measurable set $B \subset \Theta$,
where $\mathcal{Z}_n(B)=\{\bz\in \mathcal{Z}_n: b(G^*(\bY_n,\bX_n,\bz)) \in B\}$, and $P_n^*(\bz|\bX_n,\bY_n$, $\bw_n^*)$ 
denote the cumulative distribution 
function (CDF) corresponding to $p_n^*(\bz|\bX_n,\bY_n,\bw_n^*)$.
The EFD  provides an uncertainty measure for $b(\btheta)$. 
Practically, 
it can be constructed based on the samples
$\{b(\bar{\btheta}_1), 
b(\bar{\btheta}_2), \ldots, b(\bar{\btheta}_M)\}$, where 
$\{\bar{\btheta}_1, \bar{\btheta}_2, \ldots, \bar{\btheta}_M\}$ denotes  
the fiducial $\bar{\btheta}$-samples collected at step (iv) of Algorithm \ref{EFIalgorithm}. 

Finally, we note that for a neural network model, its parameters are only unique up to certain loss-invariant transformations, such as reordering hidden neurons within the same hidden layer or simultaneously altering the sign or scale of certain connection weights \citep{SunSLiang2021}. 
Therefore, for the $\bw$-network, the consistency of $\bw_n^*$ refers to its consistency with respect to one of the equivalent 
solutions to (\ref{identityeq}), while mathematically $\bw_n^*$
can still be treated as unique.

\subsection{Self-Diagnosis in EFI}


Given the flexibility of DNN models, reliable diagnostics are crucial in deep learning to ensure model robustness and accuracy while identifying potential issues during training. Unlike dropout and Bayesian methods, which lack self-diagnostic capabilities, 
EFI includes a built-in mechanism for self-diagnosis. Specifically, this can be achieved through (i) analyzing the QQ-plot of the imputed random errors, and (ii) verifying that the energy function $U_n$ converges to zero.
 
According to Lemma \ref{lem1:Algconvergence}, the imputed random errors 
$\hat{\bZ}_n$ should follow the same distribution as the true random errors $\bZ_n$. Since the theoretical distribution of $\bZ_n$ is known, 
  the convergence of $\hat{\bZ}_n$ can be assessed using QQ-plot as 
 shown in Figure \ref{fig:efi poisson}(b).  
When $\bZ_n$ has been correctly imputed,  the energy function must converge to zero to ensure the consistency of the inverse function estimator. 
In practice, we can check whether $U_n(\bY_n,\bX_n,\bZ_n,\bw_n) = o(\epsilon)$ as $\epsilon\to 0$.  The validity of inference for the model uncertainty 
can thus be ensured if both diagnostic tests are satisfied. 
This diagnostic method is entirely data-driven, offering a 
simple way for validating the EFI results. 

If the diagnostic tests are not satisfied, the hyperparameters of EFI can be adjusted to ensure both tests are met for valid inference. 
These adjustments may include modifying the width of the neck layer, 
the size of the $\bw_n^{(1)}$-network,  the size of the data-modeling network, as well as tuning the learning rates and iteration numbers used in Algorithm \ref{EFIalgorithm}.

\section{Extended Fiducial Density Function of $\bZ_n$} \label{sect:EFDlimit}

Let $\mathcal{Z}_n=\{\bz \in \mathbb{R}^n: U_n(\bY_n,\bX_n,\bZ_n, \bw_n^*)=0\}$  
denote the zero-energy set, and let
$P_n^*(\bz|\bX_n,\bY_n$, $\bw_n^*)$ 
denote the cumulative distribution 
function (CDF) corresponding to $p_n^*(\bz|\bX_n,\bY_n,\bw_n^*)$.
Under some regularity conditions on the energy function, \citet{LiangKS2024EFI} proved that $\mathcal{Z}_n$ is invariant to the choice of $G(\cdot)$. 
Furthermore, they studied the convergence of 
$\lim_{\epsilon\downarrow 0} \pi_{\epsilon}(\bz|\bX_n,\bY_n,\bw_n^*)$ in  two cases:   $\Pi_n(\mathcal{Z}_{n})>0$ and $\Pi_n(\mathcal{Z}_{n})=0$, 
where $\Pi_n(\cdot)$ denotes the probability measure corresponding to the density function $\pi_0^{\otimes}(\bz)$ on $\mathbb{R}^n$. 
Specifically, 
\begin{itemize} 
 \item[(a)]{\it ($\Pi_n(\mathcal{Z}_{n})>0$)}:  In this case, 
 $p_n^*(\bz|\bX_n,\bY_n,\bw_n^*)$ is given by 
\begin{equation} \label{theoreticalzdista}
\frac{d P_n^*(\bz|\bX_n,\bY_n,\bw_n^*)}{d\bz}= \frac{1}{\Pi_n(\mathcal{Z}_n)} \pi_0^{\otimes n}(\bz), \quad 
\bz \in \mathcal{Z}_n,
\end{equation}
 which is invariant to the choices of the inverse function $G(\cdot)$
 and energy function $U_n(\cdot)$. For example, the logistic regression 
 belongs to this case as shown in \citet{LiangKS2024EFI}. 
 
 \item[(b)] {\it ($\Pi_n(\mathcal{Z}_{n})=0$)}: In this case, $\mathcal{Z}_n$ forms a manifold in $\mathbb{R}^n$ with the highest dimension $p$, and 
 $p_n^*(\bz|\bY_n,\bX_n,\bw_n^*)$ concentrates on the highest dimensional manifold and 
 is given by 
 \begin{equation} \label{theoreticalzdist}
  \frac{d P_n^*(\bz|\bX_n,\bY_n,\bw_n^*)}{d \mathscr{\nu}}(\bz)= \frac{ \pi_0^{\otimes n}(\bz) \left( {\rm det}( \nabla^2_{\bt} U_n(\bz))\right)^{-1/2}} { \int_{\mathcal{Z}_{n}} \pi_0^{\otimes n}(\bz) \left( {\rm det}(\nabla_{\bt}^2 U_n(\bz) \right)^{-1/2} d \mathscr{\nu}}, \quad \bz \in \mathcal{Z}_{n},
 \end{equation} 
 where $\mathscr{\nu}$ is the sum of intrinsic measures on the $p$-dimensional manifold in $\mathcal{Z}_n$, and  $\bt \in \mathbb{R}^{n-p}$ denotes the coefficients of the normalized smooth normal 
vectors in the tubular neighborhood decomposition \citep{Milnor1974CharacteristicC} of $\bz$.
 \end{itemize}
 
By Theorem 3.2 and Lemma 4.2 in \citet{LiangKS2024EFI}, if the target model is noise-additive and $d(\cdot)$ is specified as in (\ref{app:deq0}), 
then  $P_n^*(\bz|\bX_n,\bY_n,\bw_n^*)$ in 
(\ref{theoreticalzdist}) can be reduced to 
\begin{equation} \label{EFDeq:Z}
\frac{dP_n^*(\bz|\bX_n,\bY_n,\bw_n^*)}{d\nu}= \frac{\pi_0^{\otimes n}(\bz)}{\int_{\mathcal{Z}_n} \pi_0^{\otimes n}(\bz) d \nu}. 
\end{equation}
That is, under the consistency of $\bw_n^*$,   
$p_n^*(\bz|\bX_n,\bY_n,\bw_n^*)$ is reduced to a truncated density function  of $\pi_0^{\otimes n}(\bz)$ on the manifold 
$\mathcal{Z}_n$, while $\mathcal{Z}_n$ itself is also invariant to 
the choice of the inverse function.
In other words, for noise-additive models, 
the EFD of $\bZ_n$ is asymptotically invariant to the inverse function we learned given its consistency.

\section{Narrow-Neck $\bw$-Networks} \label{sect:narrow-net}

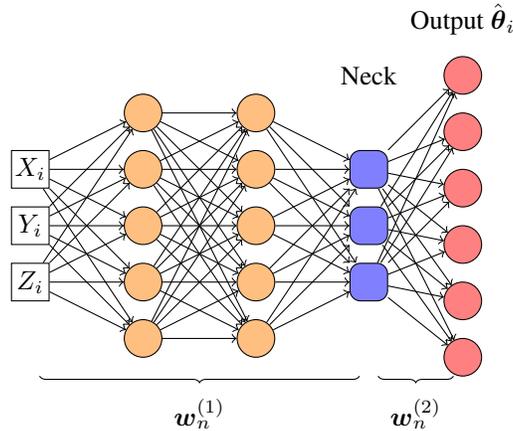
\begin{figure}[htbp]
\vspace{-0.1in}
\begin{center}
\begin{tikzpicture}[node distance=0.75cm]
\tikzstyle{randomnode} = [rectangle, rounded corners, minimum width=0.5cm, minimum height=0.5cm,text centered, draw=black, fill=blue!50]
\tikzstyle{inputnode} = [rectangle, minimum width=0.5cm, minimum height=0.5cm, text centered, draw=black]
\tikzstyle{hiddennode} = [circle, minimum width=0.5cm, minimum height=0.5cm, text centered, draw=black, fill=orange!50]
\tikzstyle{outputnode} = [circle, minimum width=0.5cm, minimum height=0.5cm,text centered, draw=black, fill=red!50]
\tikzstyle{arrow} = [thick,->,>=stealth]

\node(X1)[inputnode]{};
\node(X2)[inputnode,below of=X1]{};
\node(X3)[inputnode,below of=X2]{};
\draw (X1) node[] {$X_i$};
\draw (X2) node[] {$Y_i$};
\draw (X3) node[] {$Z_i$};

\node(Y1)[hiddennode,right of=X1,xshift=0.75cm,yshift=0.75cm]{};
\node(Y2)[hiddennode,right of=X1,xshift=0.75cm,yshift=0cm]{};
\node(Y3)[hiddennode,right of=X2,xshift=0.75cm,yshift=0cm]{};
\node(Y4)[hiddennode,right of=X3,xshift=0.75cm,yshift=0cm]{};
\node(Y5)[hiddennode,right of=X3,xshift=0.75cm,yshift=-0.75cm]{};

\node(V1)[hiddennode,right of=Y1,xshift=0.75cm]{};
\node(V2)[hiddennode,right of=Y2,xshift=0.75cm]{};
\node(V3)[hiddennode,right of=Y3,xshift=0.75cm]{};
\node(V4)[hiddennode,right of=Y4,xshift=0.75cm]{};
\node(V5)[hiddennode,right of=Y5,xshift=0.75cm]{};

\node(W2)[randomnode,right of=V2,xshift=0.75cm]{};
\node(W3)[randomnode,right of=V3,xshift=0.75cm]{};
\node(W4)[randomnode,right of=V4,xshift=0.75cm]{};
\node[above of=W2,yshift=0.5cm]{Neck};

 \node(U1)[outputnode,right of=V1,xshift=2.0cm,yshift=0.5cm]{};
 \node(U2)[outputnode,right of=V2,xshift=2.0cm,yshift=0.5cm]{};
 \node(U3)[outputnode,right of=V3,xshift=2.0cm,yshift=0.5cm]{};
 \node(U4)[outputnode,right of=V4,xshift=2.0cm,yshift=0.5cm]{};
 \node(U5)[outputnode,right of=V5,xshift=2.0cm,yshift=0.5cm]{};
 \node(U6)[outputnode,right of=V5,xshift=2.0cm,yshift=-0.25cm]{};
\node[above of=U1,xshift=0.0cm,yshift=0.0cm]{Output $\hat{\btheta}_i$};

 \draw[->](X1)--(Y1);
 \draw[->](X1)--(Y2);
 \draw[->](X1)--(Y3);
 \draw[->](X1)--(Y4);
 \draw[->](X1)--(Y5);
 \draw[->](X2)--(Y1);
 \draw[->](X2)--(Y2);
 \draw[->](X2)--(Y3);
 \draw[->](X2)--(Y4);
 \draw[->](X2)--(Y5);
 \draw[->](X3)--(Y1);
 \draw[->](X3)--(Y2);
 \draw[->](X3)--(Y3);
 \draw[->](X3)--(Y4);
 \draw[->](X3)--(Y5);

 \draw[->](Y1)--(V1);
 \draw[->](Y1)--(V2);
 \draw[->](Y1)--(V3);
 \draw[->](Y1)--(V4);
 \draw[->](Y1)--(V5);
 \draw[->](Y2)--(V1);
 \draw[->](Y2)--(V2);
 \draw[->](Y2)--(V3);
 \draw[->](Y2)--(V4);
 \draw[->](Y2)--(V5);
 \draw[->](Y3)--(V1);
 \draw[->](Y3)--(V2);
 \draw[->](Y3)--(V3);
 \draw[->](Y3)--(V4);
 \draw[->](Y3)--(V5);
 \draw[->](Y4)--(V1);
 \draw[->](Y4)--(V2);
 \draw[->](Y4)--(V3);
 \draw[->](Y4)--(V4);
 \draw[->](Y4)--(V5);
 \draw[->](Y5)--(V1);
 \draw[->](Y5)--(V2);
 \draw[->](Y5)--(V3);
 \draw[->](Y5)--(V4);
 \draw[->](Y5)--(V5);
 
 \draw[->](V1)--(W2);
 \draw[->](V2)--(W2);
 \draw[->](V3)--(W2);
 \draw[->](V4)--(W2);
 \draw[->](V5)--(W2);

 \draw[->](V1)--(W3);
 \draw[->](V2)--(W3);
 \draw[->](V3)--(W3);
 \draw[->](V4)--(W3);
 \draw[->](V5)--(W3);

 \draw[->](V1)--(W4);
 \draw[->](V2)--(W4);
 \draw[->](V3)--(W4);
 \draw[->](V4)--(W4);
 \draw[->](V5)--(W4);




 \draw[->](W2)--(U1);
 \draw[->](W2)--(U2);
 \draw[->](W2)--(U3);
 \draw[->](W2)--(U4);
 \draw[->](W2)--(U5);
 \draw[->](W2)--(U6); 

 \draw[->](W3)--(U1);
 \draw[->](W3)--(U2);
 \draw[->](W3)--(U3);
 \draw[->](W3)--(U4);
 \draw[->](W3)--(U5);
 \draw[->](W3)--(U6); 

 \draw[->](W4)--(U1);
 \draw[->](W4)--(U2);
 \draw[->](W4)--(U3);
 \draw[->](W4)--(U4);
 \draw[->](W4)--(U5);
 \draw[->](W4)--(U6); 

 \node(s1)[below of=X3,yshift=-0.5cm]{};
 \node(s2)[below of=W4,yshift=-0.5cm]{};
 \node(s3)[below of=U5,yshift=-0.25cm]{};
 \draw[decoration={brace,mirror,raise=0ex},decorate]
  (s1) -- node[midway,yshift=-0.5cm] {$\bw_n^{(1)}$} (s2);
  \draw[decoration={brace,mirror,raise=0ex},decorate]
  (s2) -- node[midway,yshift=-0.5cm] {$\bw_n^{(2)}$} (s3);
\end{tikzpicture}
\end{center}
\vspace{-0.2in}
\caption{A conceptual structure of narrow neck $\bw$-networks.}
\label{fig:RBM}
\vspace{-0.1in}
\end{figure}
 
\section{Proof of Theorem \ref{thm:1}} \label{sect:proof}

\subsection{Asymptotic Equivalence between StoNets and DNNs}
\label{sect:proofEqu}

We first provide a brief review of the theory regarding asymptotic equivalence between the StoNet and DNN models, which was originally  
established in \citet{LiangSLiang2022}. 

Consider a StoNet model: 
\begin{equation}\label{eq:stonet}
    \begin{split}
    \bY_1 &= \bb_1 + \bw_1\bX + \be_1, \\
    \bY_i &= \bb_i + \bw_i\Psi(\bY_{i-1})+ \be_i, \quad i=2,3,\dots,h,\\
    \bY &= \bb_{h+1} + \bw_{h+1}\Psi(\bY_{h}) + \be_{h+1}, 
    \end{split}
\end{equation}
where $\bX\in \mathbb{R}^p$ and $\bY\in \mathbb{R}^{d_{h+1}}$ represent 
the input and response variables, respectively;  
$\bY_i\in \mathbb{R}^{d_i}$ are latent variables;  $\be_i\in \mathbb{R}^{d_i}$ are introduced noise variables;
$\bb_i \in \mathbb{R}^{d_i}$ and $\bw_i \in \mathbb{R}^{d_i \times d_{i-1}}$ are model parameters, $d_0 = p$,  and $\Psi(\bY_{i-1})=(\psi({\bY}_{i-1,1}), \psi({\bY}_{i-1,2}), \ldots,\psi(\bY_{i-1,d_{i-1}}))^T$ is an element-wise activation function for $i = 1, \dots, h + 1$. The StoNet 
defines a latent variable model that 
reformulates the DNN as a composition of many simple regressions.
In this context, we assume that $\be_i \sim N(0, \sigma_i^2 I_{d_i})$ for $i=1,2,\dots, h, h+1$, though other distributions can also be considered 
for $\be_i$'s \citep{SunLiang2022kernel}.

The DNN model corresponding to (\ref{eq:stonet}) is given as follows:  
\begin{equation}\label{eq:DNN}
    \begin{split}
    \tilde{\bY}_1 &= \bb_1 + \bw_1\bX, \\ \tilde{\bY}_i & = \bb_i + \bw_i\Psi(\tilde{\bY}_{i-1}), \quad i=2,3,\dots,h,\\
    \bY &= \bb_{h+1} + \bw_{h+1}\Psi(\tilde{\bY}_{h}) + \be_{h+1}. 
    \end{split}
\end{equation}
Let $\bvartheta=\{\bb_1,\bw_1,\ldots,\bb_{h+1},\bw_{h+1}\}$ denote the 
parameter set of the DNN model. 
Let $\pi_{\rm DNN}(\bY|\bX, \bvartheta)$ denote the likelihood function of the DNN model (\ref{eq:DNN}), and let $\pi(\bY, \bYmis | \bX, \bvartheta)$ denote the likelihood function of the StoNet (\ref{eq:stonet}). Let $Q^*(\bvartheta)=\mathbb{E}(\log\pi_{\rm DNN}(\bY|\bX,\bvartheta))$, where the expectation is taken with respect to the joint distribution $\pi(\bX,\bY)$, \citet{LiangSLiang2022} made the following assumption regarding the network structure, activation function, and the variance of the latent variables of StoNet:

\begin{assumption} \label{ass:1}
(i) Parameter space $\tilde{\bTheta}$ (of $\bvartheta$) is compact;
(ii) For any $\bvartheta \in \tilde{\bTheta}$, $\mathbb{E}(\log\pi(\bY, \bYmis|\bX, \bvartheta))^2<\infty$ 
; 
(iii) The activation function $\psi(\cdot)$ is $c$-Lipschitz continuous;
(iv) The network's widths $d_l$'s and depth $h$ are allowed to increase with $n$; 
(v) The noise introduced in StoNet satisfies the following condition: $\sigma_{1} \leq \sigma_{2} \leq \cdots \leq \sigma_{h+1}$, 
and $d_{h+1} (\prod_{i=k+1}^{h} d_i^2) d_k \sigma^2_{k} \prec \frac{\sigma_{h+1}^2}{h}$ for any $k\in\{1,2,\dots,h\}$.
\end{assumption}



\begin{assumption}\label{ass:2}
(i) $Q^*(\bvartheta)$ is continuous in $\bvartheta$ and uniquely maximized at $\bvartheta^*$;
(ii) for any $\epsilon>0$, $sup_{\bvartheta\in\Theta\backslash B(\epsilon)}Q^*(\bvartheta)$ exists, where $B(\epsilon)=\{\bvartheta:\|\bvartheta-\bvartheta^*\|<\epsilon\}$, and $\delta=Q^*(\bvartheta^*)-sup_{\bvartheta\in\Theta\backslash B(\epsilon)}Q^*(\bvartheta)>0$.
\end{assumption}

Under Assumptions \ref{ass:1} and \ref{ass:2}, \citet{LiangSLiang2022} proved the following lemma.

\begin{lemma} \label{lemma:2} \citep{LiangSLiang2022} Suppose that Assumptions \ref{ass:1}-\ref{ass:2}  hold, and $\pi(\bY,\bYmis|\bX,\bvartheta)$ is continuous in $\bvartheta$. Then
\begin{equation} \label{Lemma: StoNet}
 \begin{split}
 (i)  &  \quad  \sup_{\bvartheta\in \Theta}\Big|\frac{1}{n}\sum_{i=1}^n\log\pi(\bY^{(i)}, \bY^{(i)}_{mis}|\bX^{(i)},\bvartheta)
 -\frac{1}{n}\sum_{i=1}^n\log\pi_{\rm DNN}(\bY^{(i)}|\bX^{(i)},\bvartheta)\Big|\overset{p}{\rightarrow} 0, \\
 (ii)  & \quad \|\hat{\bvartheta}_n-\bvartheta^*\|\overset{p}{\rightarrow}0, \quad \mbox{as $n\to \infty$},
\end{split}
\end{equation}
where  $\bvartheta^*=\arg\max_{\bvartheta\in \tilde{\bTheta}}\mathbb{E}(\log\pi_{\rm DNN}(\bY|\bX,\bvartheta))$ denotes the true parameters of the DNN model as specified in (\ref{eq:DNN}), and 
$\hat{\bvartheta}_n = \arg\max_{\bvartheta\in \tilde{\bTheta}}\{\frac{1}{n}\sum_{i=1}^n\log\pi(\bY^{(i)},\bY_{mis}^{(i)}|\bX^{(i)},\bvartheta)\}$ denotes the maximum likelihood estimator of the StoNet model (\ref{eq:stonet}) with the pseudo-complete data. 
\end{lemma}

Lemma \ref{lemma:2} implies that the StoNet and DNN are asymptotically equivalent as the training sample size $n$ becomes large, and it forms the basis for the bridging the StoNet and the DNN. 
The asymptotic equivalence can be elaborated from two perspectives. First, suppose the DNN model (\ref{eq:DNN}) is true. Lemma \ref{lemma:2} implies that when $n$ becomes large, the weights of the DNN can be learned by training a StoNet of the same structure with $\sigma_i^2$'s satisfying Assumption \ref{ass:1}-(v). 
On the other hand,  suppose that the  StoNet (\ref{eq:stonet}) is true, and then 
 Lemma \ref{lemma:2} implies that {\it for any StoNet satisfying Assumptions \ref{ass:1} \& \ref{ass:2}, the weights $\bvartheta$ can be learned by training a DNN of the same structure when the training sample size is large.}


\subsection{Justification for Condition (\ref{w2req})-(ii)}  \label{sect:justification}

To justify this condition, we first introduce the following lemma: 

 \begin{lemma} \label{lemma:eigen} 
Consider a random matrix $\mathbb{M}\in \mathbb{R}^{n\times d}$ with $n \geq d$.  Suppose that the 
eigenvalues of $\mathbb{M}^T \mathbb{M}$ are upper bounded, i.e., $\lambda_{\max}(\mathbb{M}^T\mathbb{M}) \leq \kappa_{\max}$ for 
some constant $\kappa_{\max}>0$. 
Let $\Psi(\mathbb{M})$ denote an elementwise transformation of $\mathbb{M}$. Then 
$\lambda_{\max}\left( (\Psi(\mathbb{M}))^T (\Psi(\mathbb{M})) \right) \leq \kappa_{\max}$
for the {\it tanh}, {\it sigmoid} and ReLU transformations. 
\end{lemma} 
\begin{proof}  
For ReLU, the result follows from Lemma 5 of \citet{Dittmer2018SingularVF}. For 
{\it tanh} and {\it sigmoid}, since they are Lipschitz continuous with a Lipschitz constant of 1, Lemma 5 of  \citet{Dittmer2018SingularVF} also applies. 
\end{proof}

Since the connection weights take values in a compact space $\tilde{\bTheta}$, 
 there exists a constant $0<\tau_{\max}<\infty$ such that 
\[
\lambda_{\max}(\bw_l \bw_l^T) \leq  \tau_{\max},
\]
for any  $l=1,2,\ldots,h+1$, where $\bw_l\in \mathbb{R}^{d_{l} \times d_{l-1}}$ is the weight matrix of the DNN at layer $l$.

Let $\mathbb{M}_l\in \mathbb{R}^{n\times d_l}$ denote the output of  hidden layer $l \in \{1,2,\ldots,h\}$ of the StoNet; that is, 
\begin{equation} \label{layereq}
\begin{split}
\mathbb{M}_l & =\Psi( \Psi(\mathbb{M}_{l-1}) \bw_l^T), \quad l=1,2,\ldots, h-1, \\
\mathbb{M}_h &= \Psi(\Psi(\mathbb{M}_{h-1}) \bw_h^T+\mathbb{V}). 
\end{split} 
\end{equation}
Consider the case that the activation functions are bounded, such as
{\it sigmoid} and {\it tanh}. Then 
the  matrix $\mathbb{E}(\mathbb{M}_h^T \mathbb{M}_h)$ has the eigenvalues upper 
bounded by $n \kappa_{\max}$ for some constant $0<\kappa_{\max}< \infty$.

For the case that the activation functions are unbounded, such as 
{\it ReLU} or {\it leaky ReLu}, we can employ the layer-normalization method in 
training. In this case, by Lemma \ref{lemma:eigen}, 
the  matrix $\Psi(\mathbb{M}_{h-1})^T \Psi(\mathbb{M}_{h-1})$ 
has the eigenvalues upper bounded by $n \kappa_{\max}$, provided that the 
eigenvalues of the  matrix $\mathbb{M}_{h-1}^T \mathbb{M}_{h-1}$ is bounded  by $n \kappa_{\max}$ after layer-normalization.  

Let $\tilde{\mathbb{M}}_h= \Psi(\mathbb{M}_{h-1}) \bw_h^T$. Then, by an extension of Ostrowski's theorem, see Theorem 3.2 of \citet{Higham1998ModifyingTI}, 
 the eigenvalues of the matrix $\tilde{\mathbb{M}}_h^T \tilde{\mathbb{M}}_h= 
\bw_h (\Psi(\mathbb{M}_{h-1}))^T \Psi (\mathbb{M}_{h-1}) \bw_h^T$ is bounded by  
\[
\lambda_{\max}(\tilde{\mathbb{M}}_h^T \tilde{\mathbb{M}}_h) \leq n\kappa_{\max} \tau_{\max}.
\]
Also, we have $\frac{1}{n}\lambda_{\max}(\mathbb{E}(\mathbb{V}^T \mathbb{V}))=\sigma_v^2$  under the assumption (*). Then, with the use of 
Lemma \ref{lemma:eigen} and the Cauchy-Schwarz inequality,  we have 
\[ 
\frac{1}{n}\lambda_{\max}(\mathbb{E}(\mathbb{M}_h^T \mathbb{M}_h)) \leq \kappa_{\max} \tau_{\max} +  \sigma_v^2 + 2  
\sigma_v \sqrt{\kappa_{\max} \tau_{\max}},
\]
as the sample size $n$ becomes large. This concludes the proof 
for condition (\ref{w2req})-(ii).

\subsection{Proof of Theorem \ref{thm:1}} \label{sect:proof3.1}

\begin{proof} 
First, we note that $\sigma_{\hat{\theta}}^2$,
the variance of each component of
$\hat{\btheta}_i (\in \mathbb{R}^p)$,  is of the order 
$O(\epsilon/\eta)$ under the setting of the energy function (\ref{app:energyfunction11}). 
By setting $\sigma_u^2 \prec \frac{\epsilon}{\eta h d_h p}$, it is easy to verify that the $\bw$-network 
satisfies Assumption \ref{ass:1}. Therefore, under the additional regularity conditions given in Assumption \ref{ass:2}, 
the proposed stochastic $\bw$-network is 
asymptotically equivalent to the original $\bw$-network.
Furthermore, the stochastic $\bw$-network can be trained by training the original $\bw$-network using Algorithm \ref{EFIalgorithm}.
 
Next, we prove that the inverse mapping learned through training the stochastic 
$\bw$-network is consistent.
For Lemma \ref{lem1:Algconvergence}, \citet{LiangKS2024EFI} first established the convergence 
of the weights to $\bw_n^*$, and subsequently established the weak convergence of 
the latent variables. Therefore, 
due to the Markov chain nature of Algorithm \ref{EFIalgorithm}, 
it suffices to prove that the inverse mapping produced by the StoNet is consistent, assuming the latent variables have been correctly imputed 
  (i.e., the true values of $\bZ_n$ are known). 
  Once the inverse mapping's consistency is established, the latent variables will be correctly imputed in the next iteration, owing to the algorithm's equation-solving nature, which is ensured by setting $\epsilon \to 0$ upon   convergence to the zero-energy region. 
This ensures that algorithm remains in its equilibrium, enabling accurate
inference of model uncertainty.

Under the StoNet setting,  estimation of $\bw_n^{(2)}$ involves solving  $p$ low-dimensional regressions.  
In particular,  given $\mathbb{M}=(\bm_1,\bm_2,\ldots,\bm_n)^T \in \mathbb{R}^{n\times d_l}$, solving each of the regressions contributes a 
parameter estimation error:
\begin{equation} \label{boundeq1}
\mathbb{E} \| \hat{\bxi}_j^{m}-\bxi_j^*\|^2 = \frac{\sigma_{\hat{\theta}}^2}{n} \mbox{Tr}([\mathbb{E}(\bm_i\bm_i^T)]^{-1})  \leq   
\frac{d_h \sigma_{\hat{\theta}}^2}{n \rho_{\min}}  = O(\frac{d_h \epsilon}{n \eta}),
\end{equation}
where $\bxi_i^*$ denotes the $i$th column of $\bw_n^{*(2)}$, i.e., 
$\bw_n^{*(2)}:=(\bxi_1^*,\ldots,\bxi_p^*)^T \in \mathbb{R}^{p \times d_h}$;  $\hat{\bxi}_j^{m}= 
(\mathbb{M}^T\mathbb{M})^{-1} \mathbb{M}^T \hat{\btheta}^{(j)}$ is the OLS estimator for the regression coefficients; and $\hat{\btheta}^{(j)}=(\hat{\theta}_{1}^{(j)},\ldots,\hat{\theta}_{n}^{(j)})^T$.  

Let $\hat{\bw}^{m(2)}=(\hat{\bxi}_1^{m}, \ldots, \hat{\bxi}_p^{m})^T \in \mathbb{R}^{p \times d_h}$.
By a fundamental property of linear regression, the mean prediction 
$\hat{\btheta}_i^{*}=\hat{\bw}^{m(2)} \bm_i$ is consistent with 
respect to  $\btheta^*=\mathbb{E}(\hat{\btheta}_i)$, 
provided that $d_h \prec n$, the $\bm_i$'s  
extract all $\hat{\btheta}$-relevant information from the input variables,  
and the $\bw$-network has sufficient capacity to establish the linear relationship $\hat{\btheta}_i \sim \bm_i$ for $i=1,2,\ldots,n$.  
The equality $\btheta^*=\mathbb{E}(\hat{\btheta}_i)$ is guaranteed by 
the setting of $\epsilon \to 0$ and 
by the construction of the energy function, 
which approaches to zero if and only if the empirical mean $\frac{1}{n} \sum_{i=1}^n\hat{\btheta}_i$ converges to $\btheta^*$ and the variance of $\hat{\btheta}_i$ approaches to zero.  
Furthermore, as shown in \citet{LiangSLiang2022}, 
the stochastic layer effectively provides a sufficient dimension reduction for the input variables.

Finally, we prove that the inverse mapping obtained by Algorithm \ref{EFIalgorithm} during the training of the non-stochastic $\bw$-network is also consistent. 
Let $\tilde{\btheta}_i^*= \bw^{*(2)}\Psi(\bmu_i)$. 
Let $\tilde{\theta}_i^{*(j)}$
and $\hat{\theta}_i^{*(j)}$ denote the $j$th element of 
$\tilde{\btheta}_i^*$ and 
$\hat{\btheta}_i^*$, respectively. From  equation (*), we have  
\begin{equation} \label{L1bound}
\begin{split} 
\mathbb{E} |\hat{\theta}_{i}^{*(j)}-\tilde{\theta}_{i}^{*(j)}| & = 
\mathbb{E} |(\hat{\bxi}_j^{m})^T \bm_i - (\bxi_j^{*})^T
\Psi(\bmu_i) | \\
 & \leq  \mathbb{E} | (\hat{\bxi}_j^{m} - \bxi_j^{*})^T \bm_i| + 
   \mathbb{E} | (\bxi_j^{*})^T (\bm_i-\Psi(\bmu_i))|  \\ 
 & \leq (\mathbb{E} \|\hat{\bxi}_j^{m} - \bxi_j^{*}\|^2)^{1/2}  
   (\mathbb{E} \|\bm_i\|^2)^{1/2} + (\mathbb{E} \|\bxi_j^*\|^2)^{1/2} ( \mathbb{E}\|\bm_i-\Psi(\bmu_i)\|^2)^{1/2} \\
 & \leq (\mathbb{E} \|\hat{\bxi}_j^{m} - \bxi_j^{*}\|^2)^{1/2} (\mbox{Tr}( \mathbb{E}(\bm_i\bm_i^T))^{1/2} + c (\mathbb{E}\|\bxi_j^*\|^2)^{1/2}  (\mbox{Tr}(\sigma_v^2 I_{d_h}))^{1/2} \\
 & \lesssim \sqrt{\frac{d_h \epsilon}{n \eta}} \sqrt{d_h \rho_{\max}} + c 
 d_{\tilde{\Theta}}  d_h \sigma_v,   \\
\end{split} 
\end{equation}
where $d_{\tilde{\Theta}}$ denotes the radius of the parameter space $\tilde{\bTheta}$ (centered 
at 0), the second inequality follows from 
Cauchy-Schwarz inequality, the third inequality follows from the Taylor expansion for $\Psi(\bmu_i+\bv_i)$ (at the point $\bmu_i$),  and the last inequality 
follows from (\ref{boundeq1}), condition (\ref{w2req})-(ii), and the boundedness of $\bxi_j^*$ as stated in Assumption \ref{ass:1}-(i). 
 
 Substituting $\sigma_v \prec \sqrt{\frac{\epsilon}{\eta h d_h p}}$ and ignoring some constant factors in (\ref{L1bound}), we obtain 
\[
\mathbb{E}\|\hat{\btheta}_i^* - \tilde{\btheta}_i^*\|_1  
\leq p \mathbb{E} |\hat{\theta}_{i}^{*(j)}-\tilde{\theta}_{i}^{*(j)}|
\lesssim p d_h \sqrt{\frac{\epsilon}{n}} + 
 \sqrt{ \frac{\epsilon p d_h}{h}},
\]
where $\|\cdot\|_1$ denotes the $l_1$-norm of a vector. 

By setting $\epsilon \prec \min\{ \frac{n}{p^2 d_h^2}, \frac{h}{p d_h}\}$, we  
have $\mathbb{E}\|\hat{\btheta}_i^* - \tilde{\btheta}_i^*\|_1 =o(1)$, 
which implies $\tilde{\btheta}_i^*$ is also  
consistent with respect to  $\btheta^*=\mathbb{E}(\hat{\btheta}_i)$. 
Consequently, the inverse mapping
$\frac{1}{n} \sum_{i=1}^n \tilde{\btheta}_i^*$ produced by 
Algorithm \ref{EFIalgorithm} in training the non-stochastic $\bw$-network
is also consistent with respect to $\btheta^*$. 
This concludes the proof of the theorem. 
\end{proof}


\section{Additional Numerical Results} \label{sect:AdditionalStudy}

Regarding uncertainty quantification, we note that there are two types of
uncertainties:

\begin{enumerate}
    \item \textbf{Aleatoric uncertainty:} This refers to the \textit{irreducible noise} inherent in the data-generating process. It can be modeled as
    $$
    y_i = f(x_i \mid \theta) + \epsilon_i,
    $$
    where $\epsilon_i \sim \mathcal{N}(0, \sigma^2)$. Estimating the unknown variance $\sigma^2$ corresponds to quantifying the aleatoric uncertainty (system random error). This is precisely what we addressed in the last experiment included in our previous rebuttal, where $\sigma^2 = 0.1$ was treated as unknown.
    \item \textbf{Epistemic uncertainty:} This refers to the \textit{reducible estimation error} due to limited data or incomplete knowledge of the true model (see, for example, model comparison in Section A5.8). 
    In classical statistics, confidence intervals quantify epistemic uncertainty: as the dataset size increases, epistemic uncertainty—and thus the width of the confidence interval—decreases.
\end{enumerate}

In the following section, we demonstrate through different examples that 
EFI is able to accurately quantity both types of uncertainties. We use the coverage rate as the key metric to quantify epistemic uncertainty. The coverage rate can reach the nominal level only when the parameter estimates are unbiased and the uncertainty estimation is accurate.  

\subsection{1-D Poisson Equation}

Figure \ref{fig:pinn_traj} and Figure \ref{fig:bpinn_traj}
provide typical trajectories learned for the 1-D Poisson model (\ref{poissoneq}) 
using the methods: PINN, Dropout, and Bayesian PINN. For the ablation study, we vary the noise standard deviation from $0.01$ to $0.1$ to further assess the validity and accuracy of the proposed method in uncertainty quantification. The results, presented in Tables~\ref{table:sd0.01},~\ref{table:sd0.025}, and~\ref{table:sd0.1}, show that under different noise levels, the EFI algorithm consistently achieves a 95\% coverage rate for the 95\% confidence intervals.

To further investigate the relationship between confidence intervals and epistemic uncertainty, we conducted an additional experiment using two different sample sizes: 20 and 80. The results, presented in Table~\ref{table:sample size different}, clearly show that as the sample size increases, the width of the confidence interval decreases while the coverage rate remains consistent. This demonstrates that the confidence interval effectively captures the reducible nature of epistemic uncertainty.

\begin{figure*}[!ht]
    \centering
    \begin{tabular}{ccc}
        \begin{subfigure}[b]{0.35\textwidth}
            \includegraphics[width=\textwidth]{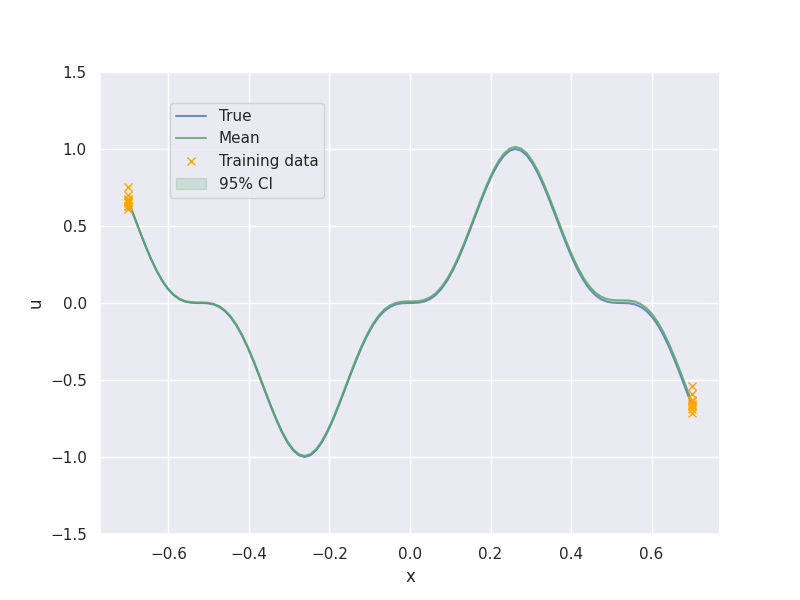}
            \caption{PINN}
        \end{subfigure} &
        \begin{subfigure}[b]{0.35\textwidth}
            \includegraphics[width=\textwidth]{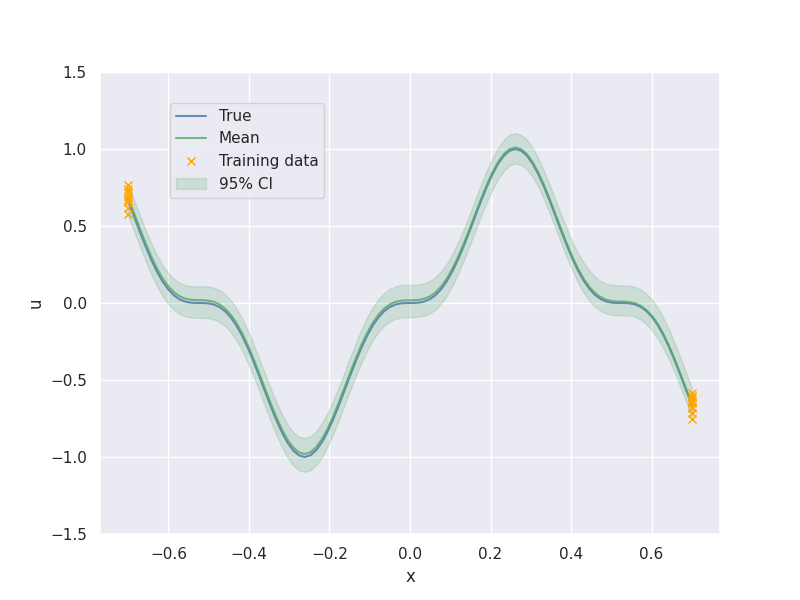}
            \caption{PINN dropout $0.5\%$}
        \end{subfigure} \\
        \begin{subfigure}[b]{0.35\textwidth}
            \includegraphics[width=\textwidth]{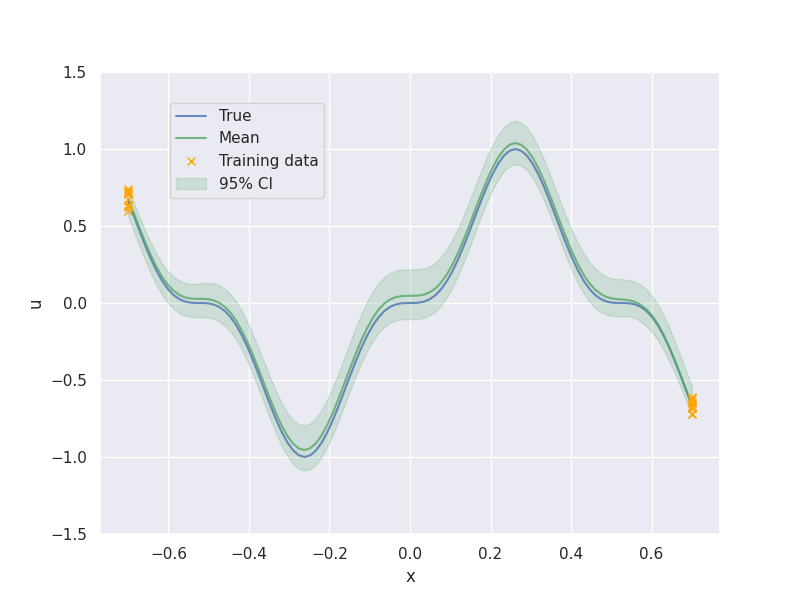}
            \caption{PINN dropout $1\%$}
        \end{subfigure} &
        \begin{subfigure}[b]{0.35\textwidth}
            \includegraphics[width=\textwidth]{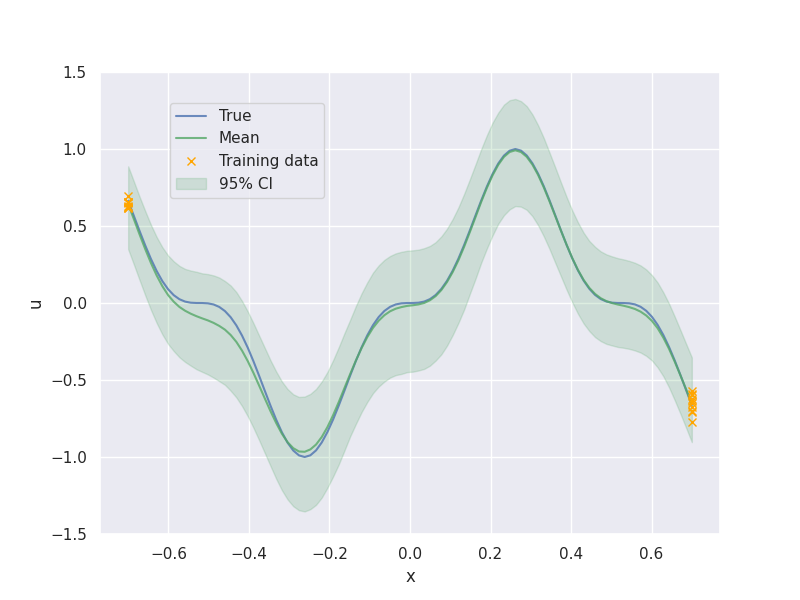}
            \caption{PINN dropout $5\%$}
        \end{subfigure}
    \end{tabular}
    \caption{ 1-D Poisson model (\ref{poissoneq}): 
    (a) Trajectory learned by PINN (without uncertainties); 
     (b)-(d) Trajectories learned 
     by Dropout with different dropout rates.}
    \label{fig:pinn_traj}
\end{figure*}

\begin{figure*}[!ht]
    \centering
    \begin{subfigure}[b]{0.32\textwidth}
        \includegraphics[width=\textwidth]{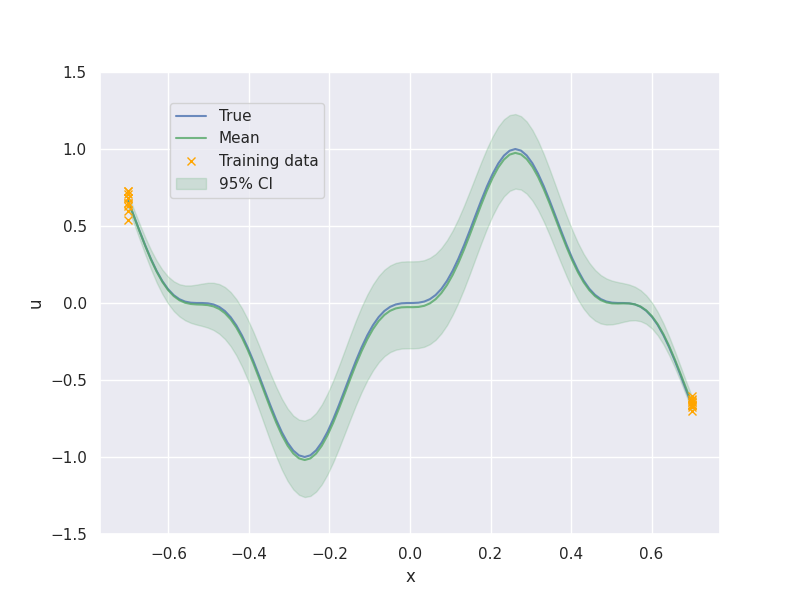}
        \caption{$\sigma_{f}=0.05$}
    \end{subfigure}
    \hfill
    \begin{subfigure}[b]{0.32\textwidth}
        \includegraphics[width=\textwidth]{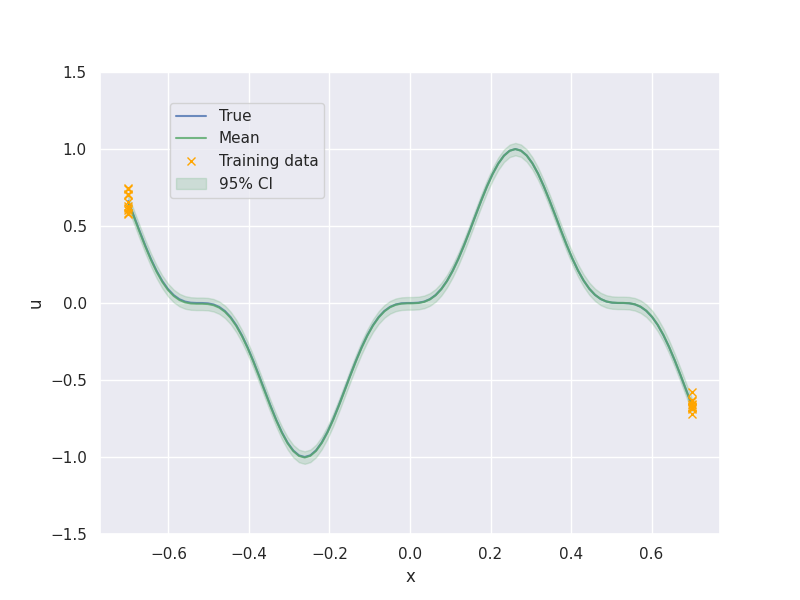}
        \caption{$\sigma_{f}=0.005$}
    \end{subfigure}
    \hfill
    \begin{subfigure}[b]{0.32\textwidth}
        \includegraphics[width=\textwidth]{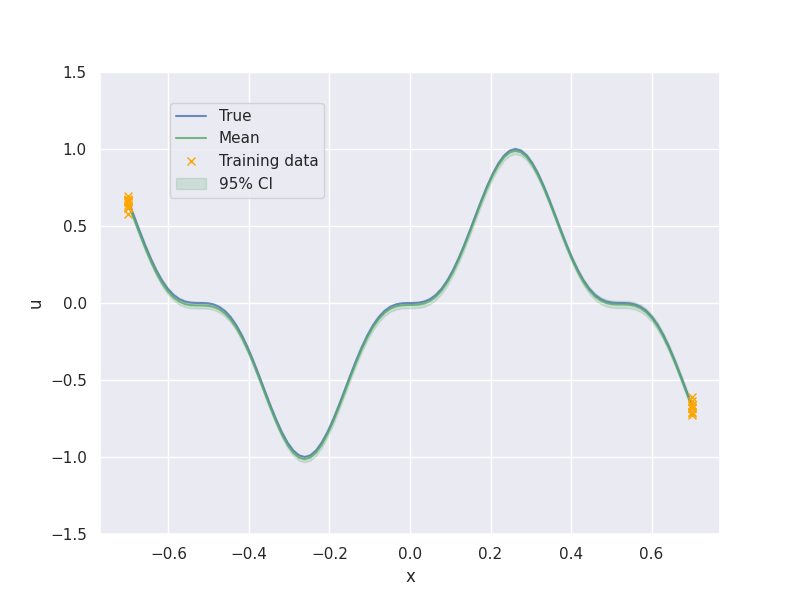}
        \caption{$\sigma_{f}=0.0005$}
    \end{subfigure}
    \caption{1-D Poisson model (\ref{poissoneq}): Typical trajectories learned 
    by Bayesian PINN with different $\sigma_{f}$ values.}
    \label{fig:bpinn_traj}
\end{figure*}

\begin{table*}[htbp]
\caption{Comparison of different methods for 1D-Poisson with $\sigma=0.01$}
\label{table:sd0.01}
\vspace{0.1in}
\centering
\footnotesize
\begin{tabular}{lcccccc}
    \toprule
    Method     & MSE    & Coverage Rate     & CI-Width \\
    \midrule
    PINN (no dropout)  & 0.000008 (0.000001) & 0.2808 (0.028721) & 0.002079 (0.000026) \\
    Dropout (0.5\%)  & 0.000160 (0.000037) & 1.0000 (0.000000) & 0.199798 (0.004135) \\
    Dropout (1\%)  & 0.000676 (0.000082) & 1.0000 (0.000000) & 0.276534 (0.006642) \\
    Dropout (5\%)  & 0.009080 (0.001516) & 0.9639 (0.011218) & 0.593541 (0.025812) \\
    Bayesian ($\sigma_f=0.05$) & 0.000181 (0.000025) & 0.9984 (0.000400) & 0.254861 (0.001554) \\
    Bayesian ($\sigma_f=0.005$) & 0.000007 (0.000001) & 0.9915 (0.003220) & 0.028386 (0.000201) \\
    Bayesian ($\sigma_f=0.0005$) & 0.000007 (0.000001) & 0.9544 (0.013751) & 0.010437 (0.000046) \\
    \bf{EFI}  & \bf{0.000007 (0.000001)} & \bf{0.9423 (0.015229)} & \bf{0.010759 (0.000117)} \\
    \bottomrule
\end{tabular}
\end{table*}

\begin{table*}[htbp]
\caption{Comparison of different methods for 1D-Poisson with $\sigma=0.025$}
\label{table:sd0.025}
\vspace{0.1in}
\centering
\footnotesize
\begin{tabular}{lcccccc}
    \toprule
    Method    & MSE    & Coverage Rate     & CI-Width \\
    \midrule
    PINN (no dropout) & 0.000046 (0.000006) & 0.1529 (0.019175) & 0.002110 (0.000025) \\
    Dropout (0.5\%) & 0.000157 (0.000028) & 1.0000 (0.000000) & 0.195658 (0.003892) \\
    Dropout (1\%) & 0.000666 (0.000084) & 1.0000 (0.000000) & 0.267936 (0.005751) \\
    Dropout (5\%) & 0.004573 (0.000654) & 0.9979 (0.001157) & 0.643227 (0.031514) \\
    Bayesian ($\sigma_f=0.05$) & 0.000161 (0.000019) & 0.9992 (0.000307) & 0.257184 (0.001612) \\
    Bayesian ($\sigma_f=0.005$) & 0.000060 (0.000013) & 0.9630 (0.012690) & 0.037341 (0.000418) \\
    Bayesian ($\sigma_f=0.0005$) & 0.000090 (0.000014) & 0.8504 (0.023270) & 0.024676 (0.000366) \\
    \bf{EFI} & \bf{0.000048 (0.000006)} & \bf{0.9577 (0.014487)} & \bf{0.027845 (0.000115)} \\
    \bottomrule
\end{tabular}
\end{table*}

\begin{table*}[htbp]
\caption{Comparison of different methods for 1D-Poisson with $\sigma=0.1$}
\label{table:sd0.1}
\vspace{0.1in}
\centering
\footnotesize
\begin{tabular}{lcccccc}
    \toprule
    Method   & MSE    & Coverage Rate     & CI-Width \\
    \midrule
    PINN (no dropout) & 0.000581 (0.000066) & 0.0320 (0.006518) & 0.002088 (0.000024) \\
    Dropout (0.5\%) & 0.000935 (0.000106) & 0.9905 (0.002706) & 0.196226 (0.003748) \\
    Dropout (1\%) & 0.001413 (0.000175) & 0.9976 (0.001929) & 0.274188 (0.007653) \\
    Dropout (5\%) & 0.006840 (0.001027) & 0.9955 (0.002439) & 0.658607 (0.031770) \\
    Bayesian ($\sigma_f=0.05$) & 0.000966 (0.000106) & 0.9936 (0.002068) & 0.285660 (0.001405) \\
    Bayesian ($\sigma_f=0.005$) & 0.000674 (0.000080) & 0.9582 (0.015014) & 0.104743 (0.000325) \\
    Bayesian ($\sigma_f=0.0005$) & 0.004944 (0.000788) & 0.3827 (0.030245) & 0.066296 (0.002515) \\
    \bf{EFI} & \bf{0.000624 (0.000058)} & \bf{0.9493 (0.013526)} & \bf{0.099543 (0.000775)} \\
    \bottomrule
\end{tabular}
\end{table*}

\begin{table*}[!ht]
\caption{Numerical results for 1D-Poisson with $\sigma=0.1$, where the confidence interval width shrinks as sample size increases. The results are computed based on 100 independent simulations.}
\label{table:sample size different}
\vspace{0.1in}
\centering
\footnotesize
\begin{tabular}{lcccccc}
    \toprule
    Method  & $n_b$ (sample size) & MSE    & Coverage Rate   & CI-Width \\
    \midrule
    EFI & 20 & 0.000624 (0.000058) & 0.9493 (0.013526) & 0.099543 (0.000775) \\
    EFI & 80 & 0.000173 (0.000017) & 0.9501 (0.013459) & 0.054256 (0.000723) \\
    \bottomrule
\end{tabular}
\end{table*}

\subsection{1-D Poisson Equation with $f$-measurement error}

We revisit the same 1-D Poisson model as defined in \eqref{poissoneq}, but now incorporate measurement errors in both $u$ and $f$. Specifically, we consider  $4$ sensors for $u$ and $40$ sensors for $f$, with each sensor recording 10 replicate measurements. Measurement errors are modeled as $z_i^u\sim N(0, 0.05^2)$ and $ z_i^f \sim N(0, 0.05^2)$, and we simulate 100 independent datasets under this setting. The experimental results are summarized in Table \ref{table:poisson v2}.

For the Bayesian PINN (B-PINN) method \citep{BPINN}, we set $\sigma_u = \sigma_f =0.05$, 
ensuring the likelihood function is correctly specified. However, as shown in Table \ref{table:poisson v2}, B-PINN produces excessively wide confidence intervals, resulting in an inflated and inaccurate coverage rate. This highlights another issue inherent to Bayesian DNNs, as noted in \citet{Liang2018BNN,SunSLiang2021}: their performance can be significantly affected by the choice of prior in small-$n$-large-$p$ settings. Here, 
$p$ refers to the number of parameters in the DNN used to approximate the solution $u(\bx)$. 
Similarly, the dropout method continues to produce overly wide 
confidence intervals and inflated coverage rates. 

In contrast, EFI achieves a coverage rate of 94.88\% with the smallest confidence interval width. Notably, this experiment involves noise in both $u$ and $f$ observations. To evaluate the imputed errors, the QQ-plot of $\hat{z}_i^u$ and $\hat{z}_i^f$ across 100 experiments is shown in Figure \ref{fig:efi_poisson-v2}. The Q-Q plot confirms that the distribution of imputed errors closely follows its theoretical distribution, supporting the validity of the EFI approach in handling measurement noise from different sources.

\begin{table*}[!ht]
  \caption{Comparison of different methods for the 1-D Poisson model (\ref{poissoneq}) (with $f$-measurement error), averaged over 100 runs.}
  \label{table:poisson v2}
  \centering
   \begin{adjustbox}{width=1.0\textwidth}
  \begin{tabular}{lcccccc}
    \toprule
    Method & hidden layers   & MSE    & Coverage Rate     & CI-Width \\
    \midrule
        PINN & [50, 50] & 0.000271 (0.000019) & 0.0921 (0.008076) & 0.003625 (0.000085) \\
        Dropout (0.5\%) & [50, 50] & 0.000310 (0.000022) & 1.0000 (0.000000) & 0.240585 (0.002282) \\
        Dropout (1.0\%) & [50, 50] & 0.000530 (0.000046) & 1.0000 (0.000000) & 0.357782 (0.005141) \\
        Dropout (5.0\%) & [50, 50] & 0.003527 (0.000193) & 1.0000 (0.000000) & 0.669799 (0.006158) \\
        Bayesian & [50, 50] & 0.000233 (0.000017) & 0.9960 (0.002035) & 0.086291 (0.000068)  \\
        \bf{EFI} & \bf{[50, 50]} & \bf{0.000238 (0.000017)} &  \bf{0.9488 (0.009419)} & \bf{0.060269  (0.000365)} \\
        
    \bottomrule
  \end{tabular}
\end{adjustbox}
\end{table*}

 \begin{figure*}[htbp]
    \centering
    \begin{subfigure}{0.32\textwidth}
        \centering
        \includegraphics[width=\textwidth]{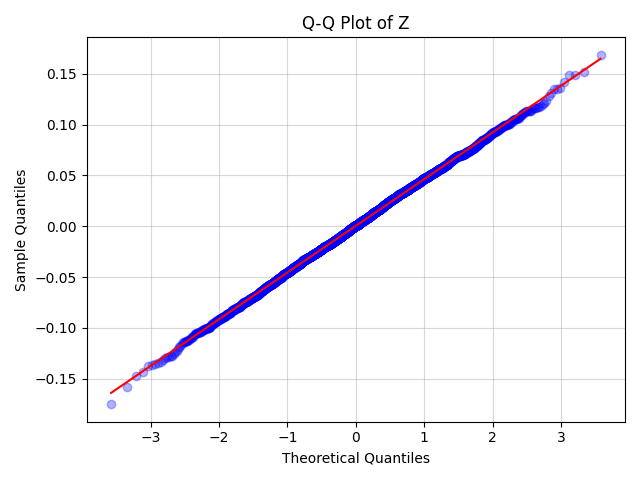}
        \caption{QQ-plot of imputed errors in $u$}
    \end{subfigure}
    \hfill
    \begin{subfigure}{0.32\textwidth}
        \centering
        \includegraphics[width=\textwidth]{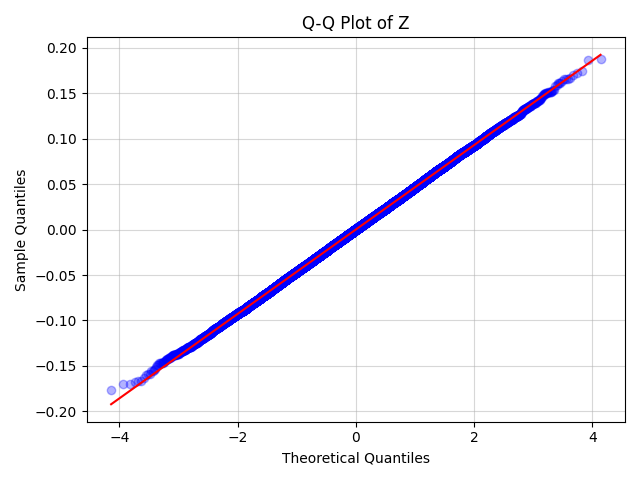}
        \caption{QQ-plot of imputed errors in $f$}
    \end{subfigure}
    \hfill
    \begin{subfigure}{0.34\textwidth}
        \centering
        \includegraphics[width=\textwidth]{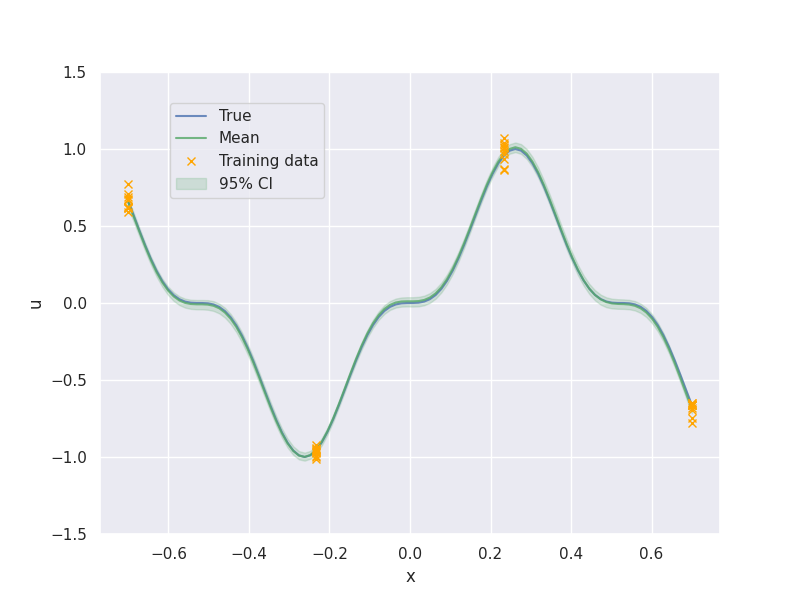}
        \caption{Confidence interval}
    \end{subfigure}
    \caption{EFI-PINN diagnostic for the 1-D Poisson model (\ref{poissoneq}) (with $f$-measurement error)}
    \label{fig:efi_poisson-v2}
\end{figure*}

\subsection{Computating Time}


Table~\ref{tab:wall time} reports the wall-clock time for the Poisson-1D experiment with different algorithms. As shown by the table, EFI is slower than PINN (with dropout) but much faster than B-PINN (with the HMC sampler). 
Importantly, we have demonstrated that EFI is the only existing method capable of correctly constructing confidence intervals in a statistically rigorous manner. 
For larger networks, we can adopt transfer learning techniques by constructing EFI hyper-networks only for the last few layers of the $\btheta$-network, which would significantly reduce the computational cost.

\begin{table}[htbp]
\centering
\caption{Wall-clock time for PINN, B-PINN and EFI-PINN}
\label{tab:wall time}
\begin{adjustbox}{width=1.0\textwidth}
\begin{tabular}{lcccc}
\toprule
\textbf{Algorithm} & \textbf{Hypernetwork} & \textbf{Epoch ($\times 10^3$)} & \textbf{Wall-clock time (s)} & \textbf{Time per epoch (ms)} \\
\midrule
PINN (dropout) & - & 200 & 133 & 0.665 \\
B-PINN & - & 100 & 1330 & 13.300\\
EFI & [16,16,16] & 200 & 391 & 1.955\\
EFI & [16,16,4] & 200 & 385 & 1.925\\
\bottomrule
\end{tabular}
\end{adjustbox}
\end{table}

\subsection{Non-linear Poisson Equation}

We extend our study to a non-linear Poisson equation given by:
\begin{equation} \label{eq:nonlinear poisson}
    \beta \frac{\partial^2 u}{\partial x^2} + k \tanh (u) =f, \quad \mbox{$x \in \Omega$},
\end{equation}
 where $\Omega=[-0.7,0.7]$, $\beta=0.01$, $k=0.7$, $u=\sin^3(6x)$, and $f$ can be 
 derived from \eqref{eq:nonlinear poisson}. 
 For this scenario, we use $4$ sensors located at $x\in \{-0.7, -0.47, 0.47, 0.7\}$ to provide noisy observations of the solution $u$. Additionally, we employ $40$ sensors, equally spaced within $[-0.7, 0.7]$, to measure $f$. Both $u$ and $f$ measurements are assumed to contain noise.
 In the simulation, measurement errors are modeled as $z_i^u \sim N(0,0.05^2)$ for $i=1,2,\ldots,40$, with each solution sensor providing $10$ replicate measurements, and $z_i^f\sim N(0,0.05^2)$ for $i=1,2,\ldots,400$.

   The experimental results are summarized in Table \ref{table:poisson nonlinear}. The findings exhibit a similar pattern to those in Table \ref{table:poisson v2}: The Bayesian and dropout methods yield inflated coverage rates and overly wide confidence intervals, whereas the EFI method achieves an accurate coverage rate and the narrowest confidence interval.
 For a fair comparison, we exclude cases where B-PINN converged to incorrect solutions, as these represent instances of failure in the optimization process, see Figure \ref{fig:bpinn_nonlinear_unstable} for an instance.


\begin{table*}[!ht]
  \caption{Comparison of different methods  for the nonlinear 1-D Poisson model (\ref{eq:nonlinear poisson}) (with $f$-measurement error), averaged over 100 runs.}
  \label{table:poisson nonlinear}
  \centering
 \begin{adjustbox}{width=1.0\textwidth}
  \begin{tabular}{lcccccc}
    \toprule
Method  & hidden layers   & MSE    & Coverage Rate     & CI-Width \\
    \midrule
        PINN & [50, 50] & 0.000507 (0.000044) & 0.0947 (0.007564) & 0.005154 (0.000565) \\
        Dropout (0.5\%) & [50, 50] & 0.001050 (0.000123) & 0.9962 (0.002068) & 0.246367 (0.002182) \\
        Dropout (1\%) & [50, 50] & 0.002807 (0.000309) & 0.9861 (0.003829) & 0.358088 (0.005178) \\
        Dropout (5\%) & [50, 50] & 0.007010 (0.000394) & 0.9956 (0.001566) & 0.543565 (0.003096) \\
        Bayesian (unstable removed) & [50, 50] & 0.000376 (0.000033) & 0.9938 (0.002618) & 0.104673 (0.000394)  \\
       \bf{EFI} & \bf{[50, 50]} & \bf{0.000385 (0.000039)} & \bf{0.9483 (0.009191) }& \bf{0.099880 (0.002853)} \\
        
    \bottomrule
  \end{tabular}
\end{adjustbox}
\end{table*}

\begin{figure}[htbp]
    \centering
    \begin{subfigure}{0.45\textwidth}
        \centering
        \includegraphics[width=\textwidth]{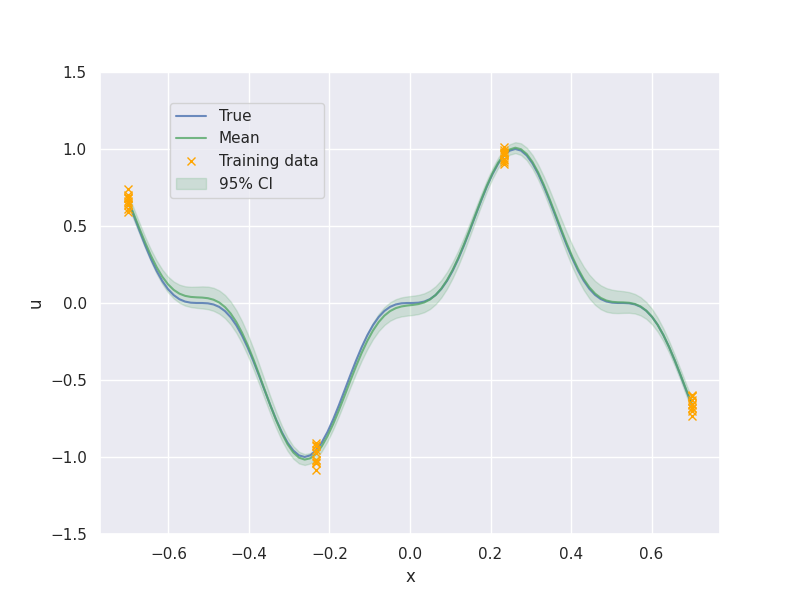}
        \caption{Success}
    \end{subfigure}
    \hfill
    \begin{subfigure}{0.45\textwidth}
        \centering
        \includegraphics[width=\textwidth]{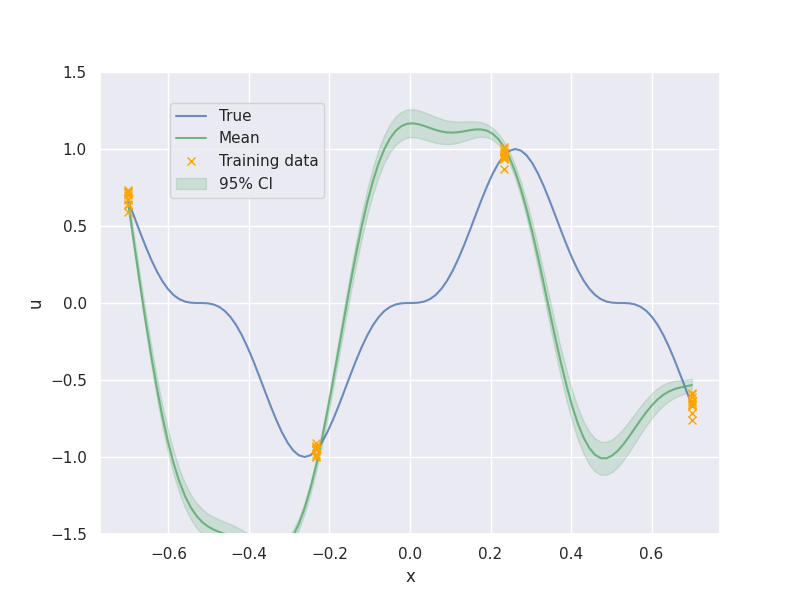}
        \caption{Failure}
    \end{subfigure}
    \caption{Successful and failed optimization results of Bayesian PINN for the nonlinear 1-D Poisson model (\ref{eq:nonlinear poisson}) (with $f$-measurement  error)}
    \label{fig:bpinn_nonlinear_unstable}
\end{figure}

\subsection{Non-linear Poisson Inverse Problem}
In this section, we consider the same non-linear Poisson equation as in (\ref{eq:nonlinear poisson}), but with $k=0.7$ treated as an unknown parameter to be estimated. For this setup, we utilize 8 sensors, evenly distributed across $\Omega=[-0.7,0.7]$ to measure $u$, with measurement noise modeled as $z_i^u \sim N(0, 0.05^2)$. Additionally, 200 sensors are employed to measure $f$, and these measurements are assumed to be noise-free.

To apply EFI framework to inverse problem, we extend the output of the $w$-network by adding an additional dimension dedicated to estimating $k$, as depicted in Figure \ref{fig:efi pe}. This modification enables the EFI framework to simultaneously estimate the solution $u$ and the 
parameter $k$, along with their respective uncertainties.
The results are presented in Table \ref{table:poisson inverse}, demonstrating the capability of EFI to provide accurate uncertainty quantification for both $u$ and $k$. In contrast, B-PINN consistently produces excessively large confidence intervals for both the solution u and the parameter k. Notably, the confidence interval for $k$ estimated by B-PINN is approximately twice as wide as that produced by EFI, indicating a significant overestimation of uncertainty. This highlights the superior precision and robustness of the EFI framework in inverse problems.

\begin{figure}[htbp]
    \centering
    \includegraphics[width=0.5\textwidth]{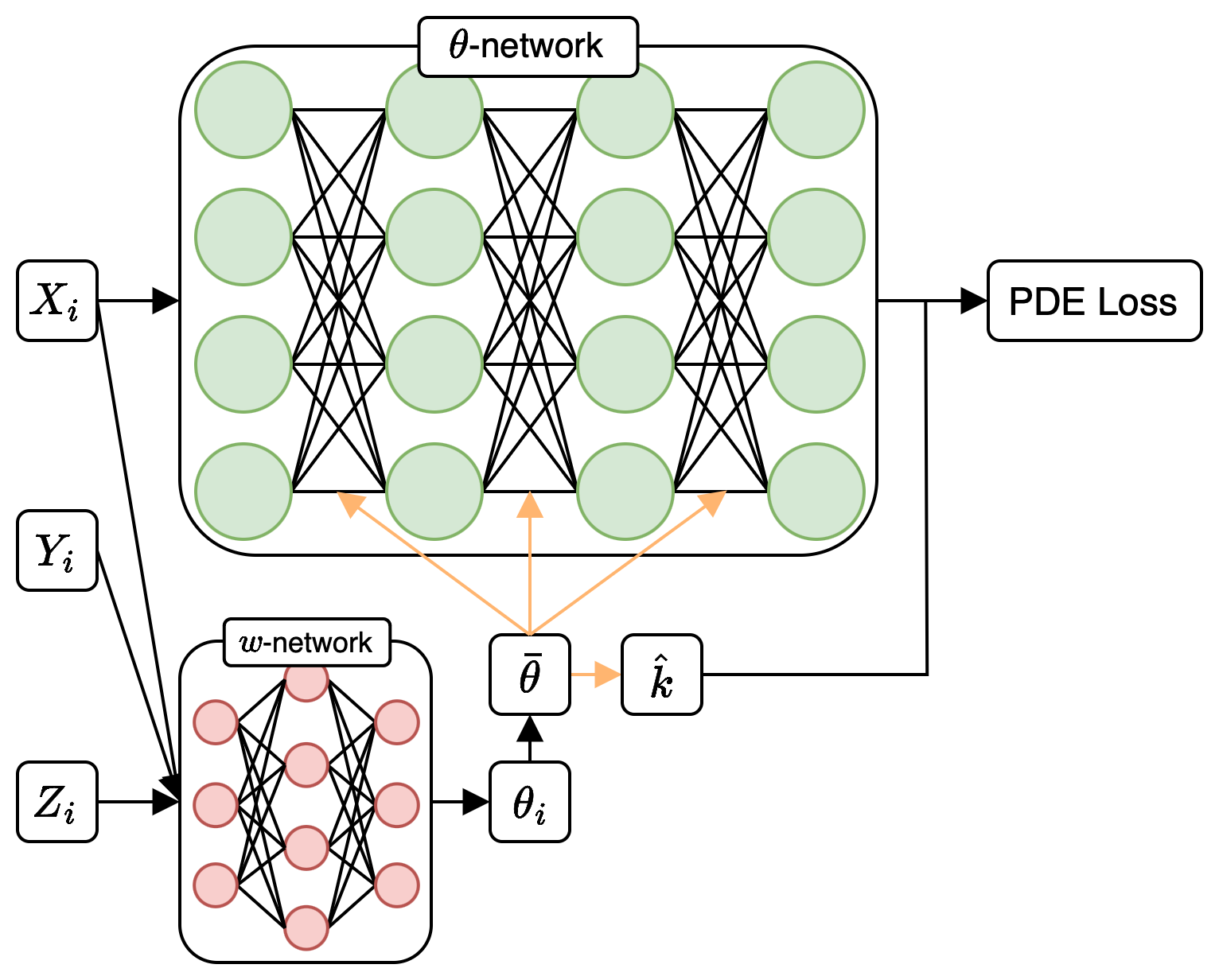}
    \caption{Diagram of EFI for inverse problems, where the orange links indicates the contribution of $\bar{\theta}$ to both $\theta$-network and $k$ estimation.}
    \label{fig:efi pe}
\end{figure}

For this problem, we also tested EFI with a larger data modeling network, consisting of two hidden layers and each hidden layer consisting of 100 hidden units. As expected, EFI produced results similar to those obtained with a much smaller data modeling network. As noted earlier, EFI can accurately quantify model uncertainty as long as the random errors are correctly imputed and the inverse function is consistently estimated. This capability is independent of the specific configurations of the $\bw$-network and the data modeling network, highlighting EFI’s flexibility and robustness.


\begin{table*}[!ht]
  \caption{Comparison of different methods  for the nonlinear 1-D Poisson model (\ref{eq:nonlinear poisson}) with parameter estimation, averaged over 100 runs.}
  \label{table:poisson inverse}
  \centering
  \begin{adjustbox}{width=1.0\textwidth}
  \begin{tabular}{lccccccccc}
    \toprule
    Method  & hidden layers   & MSE($\times 10^{-4})$    & Coverage Rate     & CI-Width & k-Mean & k-Coverage Rate & k-CI-Width \\
    \midrule
        PINN & [50, 50] & 1.79 (0.13) & 0.1464 (0.0097) & 0.0045 (0.0001) & 0.6998 (0.0006) & 0.00 (0.0000) & 7.9e-5 (5e-6) \\
        Dropout (0.5\%) & [50, 50] & 1.84 (0.10) & 1.0000 (0.0000) & 0.2393 (0.0035) & 0.6912 (0.0007) & 0.09 (0.0288) & 0.0045 (0.0002) \\
        Dropout (1.0\%) & [50, 50] & 2.72 (0.24) & 1.0000 (0.0000) & 0.3032 (0.0071) & 0.6874 (0.0008) & 0.11 (0.0314) & 0.0091 (0.0007) \\
        Dropout (5.0\%) & [50, 50] & 27.08 (2.12) & 1.0000 (0.0000) & 0.5363 (0.0131) & 0.6493 (0.0023) & 0.06 (0.0239) & 0.0252 (0.0014) \\
        Bayesian & [50, 50] & 1.31 (0.08)  &    0.9752 (0.0055) & 0.0517 (3e-5) & 0.6994 (0.0005) & 1.00 (0.0000)  &  0.0411 (0.0002)   \\
        \bf{EFI} & \bf{[50, 50]} & \bf{1.03 (0.08)} & \bf{0.9473 (0.0099)} & \bf{0.0396 (0.0002)} & \bf{0.6985 (0.0004)} & \bf{0.94 (0.0239)} & \bf{0.0179 (0.0002)} \\
        \bf{EFI} & \bf{[100, 100]} & \bf{0.98 (0.07)} & \bf{0.9560 (0.0099)} & \bf{0.0395 (0.0002)} & \bf{0.6995 (0.0004)} & \bf{0.96 (0.0197)} & \bf{0.0168 (0.0002)} \\
        
    \bottomrule
  \end{tabular}
\end{adjustbox}
\end{table*}

\subsection{Poisson equation with unknown noise standard deviation}



We now consider the case where the noise standard deviation is treated as an unknown parameter. In this setting, the variability in the observations due to the inherent randomness of the data-generating process—often referred to as systematic error—can be interpreted as aleatoric uncertainty.

As shown in Table~\ref{table:sd unknown}, the EFI algorithm successfully recovers accurate estimates for both the solution $u$ and the noise standard deviation $\sigma$, accompanied by well-calibrated confidence intervals.

\begin{table*}[!ht]
\caption{Numerical results for 1D-Poisson with unknown $\sigma=0.1$ and sample size $n_b=60$, where the number in the parentheses represents the standard error of the estimator.}
\label{table:sd unknown}
\vspace{0.1in}
\centering
\footnotesize
\resizebox{\textwidth}{!}{%
\begin{tabular}{lcccccc}
    \toprule
    Method   & MSE    & Coverage Rate     & CI-Width & $\sigma$-mean & $\sigma$-CR & $\sigma$-CI-Width \\
    \midrule
   EFI  & 0.000220 (0.000025)  & 0.9559 (0.014349) & 0.061347 (0.000848) & 0.097269 (0.001158) & 0.9400 (0.023868) & 0.056768  (0.001284) \\
    \bottomrule
\end{tabular}
}
\end{table*}

\subsection{Black-Scholes Model}

As a practical example, we consider the classical option pricing model in finance --- the Black-Scholes model \citep{Black1973ThePO}:
\begin{equation}
\begin{split}
    \frac{\partial V}{\partial t} + \frac{1}{2} \sigma^2 S^2 \frac{\partial^2 V}{\partial S^2} + r S \frac{\partial }{\partial S} - r V = 0, \\
    C(0,t) = 0 \text{ for all }t, \\
    C(S, t) \rightarrow S-K \text{ as } S\rightarrow \infty, \\
    C(S,T) = \max\{S-K, 0\},
\end{split}
\end{equation}
which describes the price 
$V(S,t)$ of an option. 
Here,  
$S$ is the price of the underlying asset (e.g., a stock), $t$ is time, 
 $\sigma$ represents the volatility of the asset, $r$ is the risk-free interest rate, $K$ is the strike price, and $T$ is the expiration time of the option. The boundary conditions reflect specific financial constraints. 
 
 This model has been widely used to calculate the price of European call and put options. Specifically, 
 the analytic solution for the call option price $C(S_t,t)$ is given by 
\begin{equation}
\begin{split}
    C(S_t, t) &= \Phi(d_{+}) S_t - \Phi(d_{-})K e^{-r(T-t)},\\
    d_{+} &= \frac{1}{\sigma\sqrt{T-t}} \left[\ln \left(\frac{S_t}{K}\right) + \left( r + \frac{\sigma^2}{2} \right) (T-t) \right], \\
    d_{-} &= d_{+} - \sigma \sqrt{T-t},
\end{split}
\end{equation}
where $\Phi(\cdot)$ denotes the standard normal cumulative distribution function. 
However, the uncertainty of the model has not yet been well studied in the literature.  
Accurately quantifying model uncertainty can significantly benefit decision-making, providing investors with a scientific foundation for making safer and more informed choices.

In this simulation experiment, we set $T=1$, $\sigma=0.5$, $r=0.05$ and $K=1$. The domain is defined on $\Omega = [0, T]\times [0,S_{\max}]$, where $S_{\max}=2$. We assume the availability of 5 sensors at $t=0$ for the price levels $S\in\{0.2, 0.4, 0.6, 0.8 ,1.0\}$, with each sensor providing 10 replicate measurements. Measurement errors are modeled as $z_i^u \sim N(0, 0.05^2)$ for $i=1,\dots, 50$, representing noisy observations. For the boundaries at $\{S=0\}$ and $\{t=T\}$, we use $50$ sensors with noise-free measurements. For physical domain, we randomly pick $800$ points from $\Omega$ to satisfy the Black-Scholes equation. \textcolor{black}{The results of the simulation are presented in Table \ref{table:european call}, where the metrics are evaluated at $t=0$ and $t=0.5$. At $t=0$, the evaluation reflects the model’s performance using noisy observed data, while at $t=0.5$, the solutions are extended from the boundaries using the Black-Scholes equation. This setup highlights the model’s ability to handle noisy observations and accurately propagate solutions over time through the governing equation. } EFI demonstrates superior performance by providing not only the most accurate solutions, as evidenced by the lowest MSE, but also the most reliable confidence intervals. 

\begin{table*}[!ht]
  \caption{Comparison of different methods for the Black-Scholes Model, averaged over 100 runs: 
   `CR' refers to the coverage rate with a nominal value of 95\%.}
  \label{table:european call}
  \centering
  \footnotesize
    \begin{adjustbox}{width=1.0\textwidth}
  \begin{tabular}{lccccccc}
    \toprule
    Method  & hidden layers   & MSE($t$=0)($\times 10^{-4})$    &  CR($t$=0)  & CI-Width($t$=0) & MSE($t$=0.5)$(\times 10^{-4})$   
    & CR($t$=0.5) & CI-Width($t$=0.5)\\
    \midrule
        PINN & [50, 50] & 3.08 (0.44) & 0.1410 (0.0102) & 0.0046 (0.0002) & 15.73 (2.26) & 0.2427 (0.0192) & 0.0080 (0.0006) \\
        Dropout (0.5\%) & [50, 50] & 1.37 (0.20) & 0.5897 (0.0216) & 0.0190 (0.0002) & 2.19 (0.37) & 0.6303 (0.0252) & 0.0197 (0.0004) \\
        Dropout (1.0\%) & [50, 50] & 4.30 (2.59) & 0.6743 (0.0219) & 0.0244 (0.0005) & 1.83 (0.26) & 0.6983 (0.0255) & 0.0234 (0.0004) \\
        Dropout (5.0\%) & [50, 50] & 1.71 (0.70) & 0.9137 (0.0122) & 0.0538 (0.0004) & 1.70 (0.21) & 0.9387 (0.0110) & 0.0510 (0.0002) \\
        Bayesian ($\sigma_f=0.05$) & [50, 50] & 1.59 (0.41) & 0.9637 (0.0175) & 0.0516 (0.0010) & 12.61 (7.14) & 0.9413 (0.0187) & 0.0658 (0.0015) \\
        Bayesian ($\sigma_f=0.005$) & [50, 50] & 10.75 (1.46) & 0.5437 (0.0255) & 0.0388 (0.0007) & 39.39 (8.82) & 0.4807 (0.0225) & 0.0426 (0.0010) \\
        \bf{EFI} & \bf{[50, 50]} & \bf{0.38 (0.05)} & \bf{0.9440 (0.0133)} & \bf{0.0158 (0.0001)} & \bf{0.17 (0.02)} & \bf{0.9600 (0.0082)} & \bf{0.0123 (0.0001)} \\
    \bottomrule
  \end{tabular}
  \end{adjustbox}
\end{table*}

To further illustrate these findings, we visualize the prediction surface in Figure \ref{BS:surface}. The figure reveals that B-PINN extends the solution poorly toward the edge at $S=2$, where no data points are available, relying solely on physical laws for extrapolation. In contrast, EFI provides a smoother and more accurate extension.
The dropout method performs reasonably well for this example with a dropout rate of 5\%;  however, its confidence interval remains significantly wider than that of EFI. As previously noted, determining an appropriate dropout rate is not feasible without additional information.
Figure \ref{BS:curve} highlights EFI’s ability to correctly quantify uncertainties. 
Near the boundary at $S=0$, where boundary information is available, the confidence interval is appropriately narrow. As the stock price $S$ increases, and boundary information becomes scarce, the confidence interval widens, reflecting the growing uncertainty.
In comparison, dropout and Bayesian methods fail to capture this behavior accurately. They produce overly broad or inconsistent intervals, particularly near the boundaries and regions with limited data, underscoring their limitations in handling uncertainty quantification for this problem.

\begin{figure*}[!ht]
    \centering
    \begin{tabular}{ccc}
    \begin{subfigure}[b]{0.4\textwidth}
        \centering
        \includegraphics[width=\linewidth]{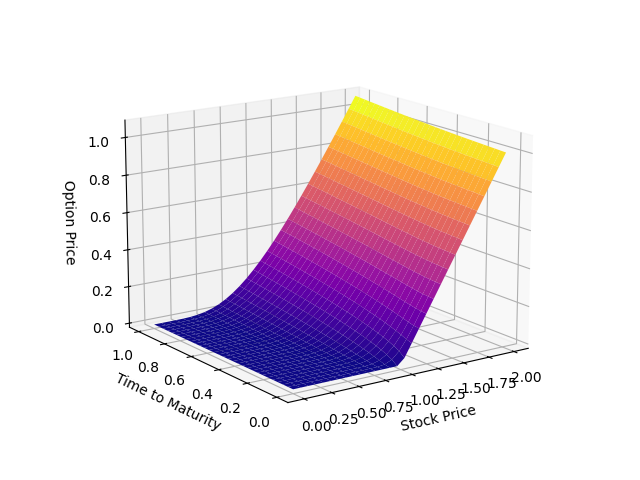}
        \caption{True price}
    \end{subfigure} & 
    \begin{subfigure}[b]{0.4\textwidth}
        \centering
        \includegraphics[width=\linewidth]{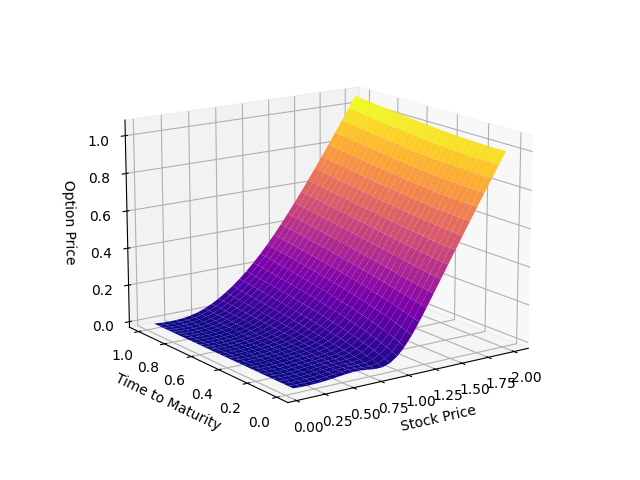}
        \caption{PINN dropout (5.0\%)}
    \end{subfigure} \\
    \begin{subfigure}[b]{0.4\textwidth}
        \centering
        \includegraphics[width=\linewidth]{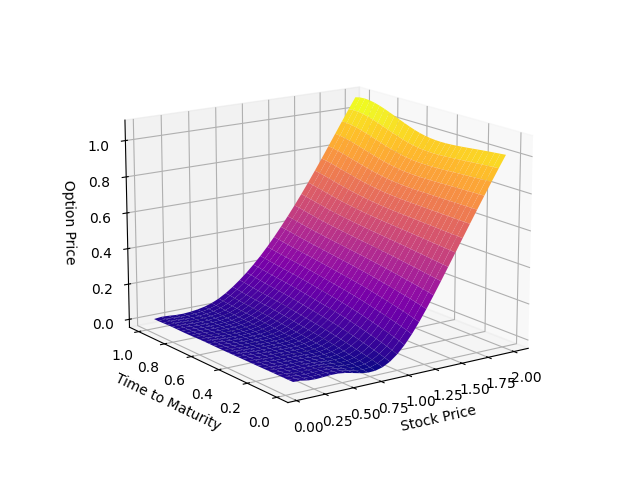}
        \caption{Bayesian PINN ($\sigma_f=0.05$)}
    \end{subfigure} & 
    \begin{subfigure}[b]{0.4\textwidth}
        \centering
        \includegraphics[width=\linewidth]{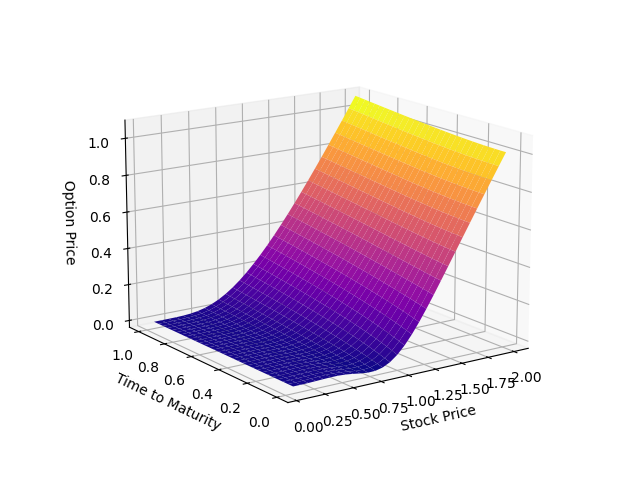}
        \caption{EFI}
    \end{subfigure}
    \end{tabular}
    \caption{European Call Option Price.} 
    \label{BS:surface}
\end{figure*}

\begin{figure*}[!ht]
    \centering
    \begin{subfigure}[b]{0.33\textwidth}
        \centering
        \includegraphics[width=\linewidth]{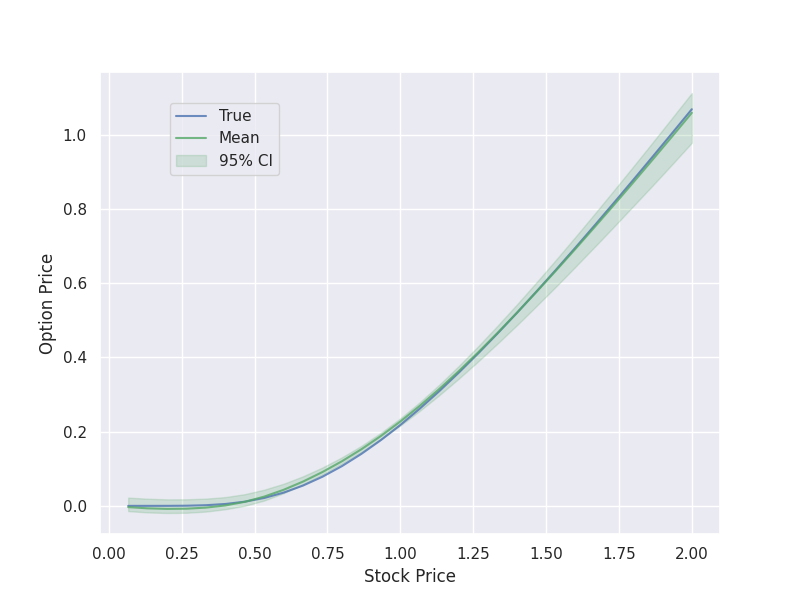}
        \caption{PINN dropout (5.0\%)}
    \end{subfigure}%
    \begin{subfigure}[b]{0.33\textwidth}
        \centering
        \includegraphics[width=\linewidth]{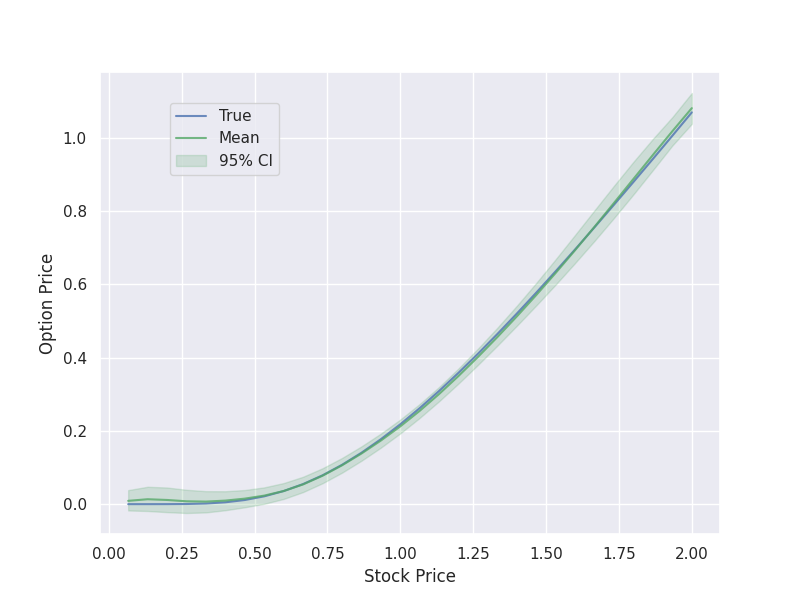}
        \caption{Bayesian PINN ($\sigma_f=0.05$)}
    \end{subfigure}%
    \begin{subfigure}[b]{0.33\textwidth}
        \centering
        \includegraphics[width=\linewidth]{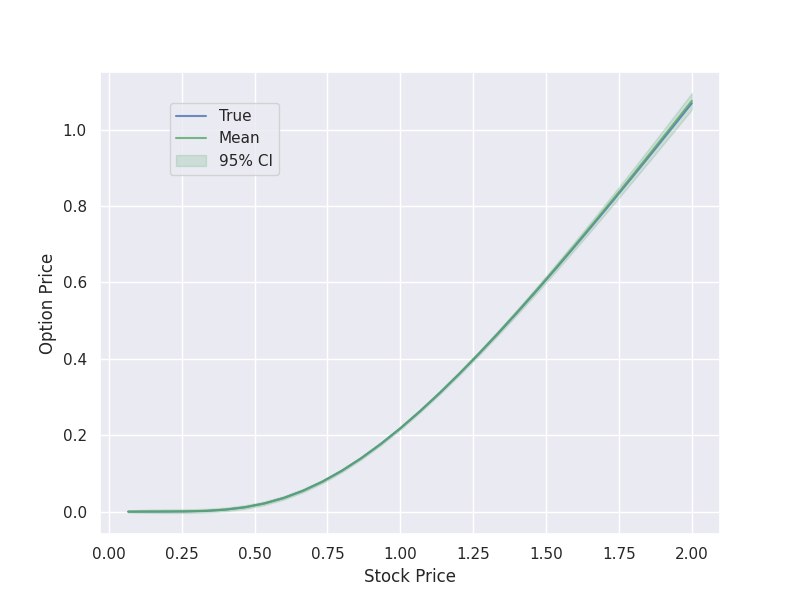}
        \caption{EFI}
    \end{subfigure}%
    \caption{European Call Option Price at $t=0$. }
    \label{BS:curve}
\end{figure*}

\subsection{Real Data: Montroll Growth Model}

Consider the Montroll growth model: 
\begin{equation} \label{eq:app:montroll}
    \frac{dp}{dt}(t) = k \cdot p(t) \cdot \left(1 - \left(\frac{p(t)}{C}\right)^\theta \right),
\end{equation}
where \(k\), \(C\), and \(\theta\) are unknown parameters. 
We applied this model to published data on the growth of Chinese hamster V79 fibroblast tumor cells \citep{Marusic1994AnalysisOG}, which also appears in \cite{Montroll}. The dataset comprises 45 measurements of tumor volumes 
($10^9$ vm$^3$) collected over a 60-day period.
Table \ref{tab:montroll:PINN} shows the parameter estimates obtained using 
PINN, and Figure \ref{fig:montroll 3 param}(a) shows the learned growth curve.

\begin{table}[htbp]
    \centering
      \caption{Parameter estimates obtained with PINN for the model (\ref{eq:app:montroll}).} 
    \vspace{0.1in} 
    \begin{tabular}{cc}
    \hline
        \textbf{Parameter} & \textbf{Estimated Value} \\
        \hline
        \(k\) & 0.8311 \\
        \(C\) & 7.3327 \\
        \(\theta\) & 0.1694 \\
        \hline
    \end{tabular}
    \label{tab:montroll:PINN}
\end{table}


\begin{figure}[!ht]
    \centering
    \begin{minipage}[t]{0.47\linewidth}
        \centering
        \includegraphics[width=\linewidth]{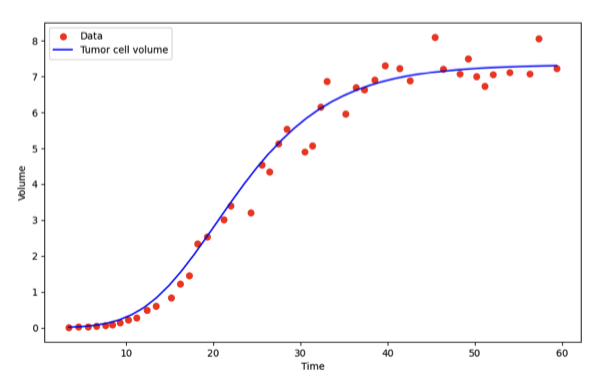}
        \caption*{(a) Montroll growth curve learned with PINN}
    \end{minipage}
    \hfill
    \begin{minipage}[t]{0.47\linewidth}
        \centering
        \includegraphics[width=\linewidth]{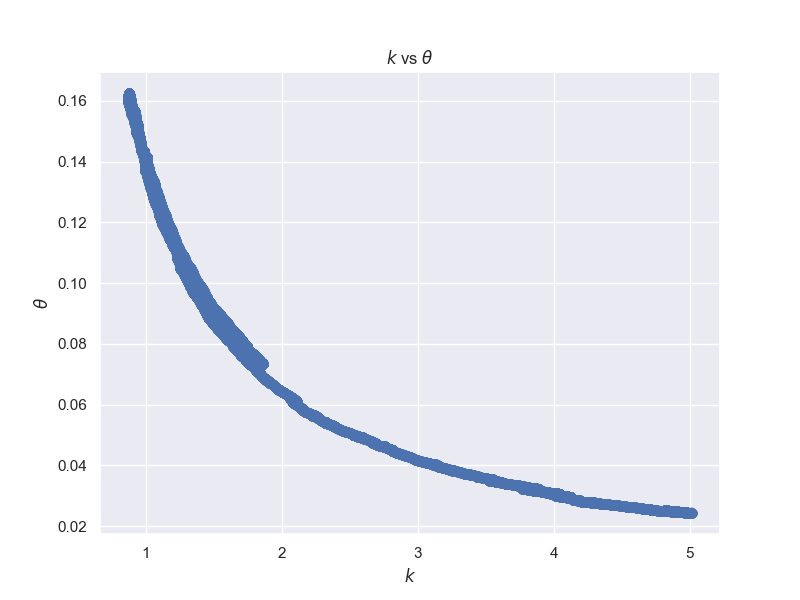}
        \caption*{(b) Relationship between $k$ and $\theta$}
    \end{minipage}
    \caption{Montroll growth Model for Chinese hamster V79 fibroblast tumor cells.}
    \label{fig:montroll 3 param}
\end{figure}

However, the standard PINN method does not provide confidence intervals for \(p(t)\) or for the parameters \(k\), \(C\), and \(\theta\). To apply EFI-PINN to this dataset, we assume a heteroscedastic noise structure given by \(\sigma_t = \sqrt{t} \cdot \sigma\), where \(\sigma\) is treated as an additional unknown parameter. Thus, we estimate four parameters in total under the EFI framework.
Using EFI-PINN, we detected that 
$k$ and $\theta$ are nearly non-identifiable, 
see Figure \ref{fig:montroll 3 param}(b),
which shows their joint distribution by plotting their samples throughout training. 


To address this non-identifiability issue, we fix \(\theta = 0.1\) and reduce the model to:
\begin{equation*}
    \frac{dp}{dt}(t) = k \cdot p(t) \cdot \left(1 - \left(\frac{p(t)}{C}\right)^{0.1} \right).
\end{equation*}

The corresponding confidence intervals for \(p(t)\) are shown in the left plot of 
Figure \ref{fig:2PDEmodels}.  
The confidence intervals of 
$k$, $C$, and $\sigma$ are given in Table \ref{tab:params}.

\begin{table}[!ht]
\centering
\caption{Parameter estimates (with 95\% confidence intervals) obtained by EFI-PINN for the Montroll growth model.}
\label{tab:params}
\vspace{0.1in}
\begin{tabular}{ccc}
\hline
\textbf{Parameter} & \textbf{95\% Confidence Interval} & \textbf{Mean} \\
\hline
\(\theta\) & -- & 0.1\\
\(k\) & (1.25, 1.40) & 1.25 \\
\(C\) & (7.30, 7.69) & 7.44 \\
\(\sigma\) & (0.27, 0.47) & 0.36 \\
\hline
\end{tabular}
\end{table}


The Montroll experiment highlights the strength of EFI in quantifying uncertainty for all parameters of interest. Moreover, it demonstrates EFI’s ability to detect
model identifiability issues, underscoring its utility in the statistical inference of scientific models. 
Additionally, EFI produces an estimate of $\sigma$, which enables the quantification of predictive uncertainty.

\subsection{Real Data: FKPP Model and Porous-FKPP Model}

Consider the Fisher–Kolmogorov–Petrovsky–Piskunov (FKPP) model and the porous FKPP (P-FKPP) model, which are governed by the following reaction-diffusion equations:

\begin{equation} \label{eq:app:FKPPmodel}
\frac{\partial u}{\partial t} = D \frac{\partial^2 u}{\partial x^2} + ru\left(1 - \frac{u}{K}\right),
\end{equation}

\begin{equation} \label{eq:app:G-P-FKPPmodel}
\frac{\partial u}{\partial t} = D \frac{\partial}{\partial x} \left[ \left(\frac{u}{K}\right)^m \frac{\partial u}{\partial x} \right] + ru\left(1 - \frac{u}{K}\right),
\end{equation}

where $D$, $r$, and $m$ are unknown parameters, and $K$ denotes the carrying capacity.
This equation has been used to model a wide range of growth and transport 
of biological processes. We applied it to scratch assay data  \citep{Jin2006EstimatingTN}. The biological experiments were conducted under varying initial cell densities --- specifically, 10,000, 12,000, 14,000, 16,000, 18,000, and 20,000 cells per well. Cell densities were recorded at 37 
equally-spaced spatial positions across five equally-spaced time points --- specifically, 0.0 days, 0.5 days, 1.0 days, 1.5 days, and 2.0 days. See  also \cite{Lagergren2020BiologicallyinformedNN} for additional descriptions of the dataset. In addition to the dataset, we partitioned the space-time domain $[0, 2] \times [0, 2]$, where the first interval corresponds to the spatial domain and the second to the temporal domain, into a $50 \times 10$ grid for computing the PDE loss (i.e., the $f$-term in equation (10)). The fitting curves of the EFI algorithm for different models and initial density values are shown in Figures~\ref{fig:fkpp} and~\ref{fig:p-fkpp}. Notably, beyond point estimation, EFI also constructs confidence intervals. 
Additionally, we note that the right plot of Figure~\ref{fig:2PDEmodels} was generated using a parameter setting different from that listed in Table~\ref{table:params_fkpp}. In this setting, fewer sample points were allocated for computing the PDE loss, which resulted in smaller fitting errors but larger deviations from the assumed PDE model.
Essentially, the two settings correspond to different datasets, as the number of sample points used to evaluate the energy function differs between them.

In Table~\ref{tab:rmse_parameters}, the root mean squared errors (RMSEs) are computed between the predicted solution $u$ and the observed data, reflecting a combination of epistemic and aleatoric uncertainty. Based on the model formulations in \eqref{eq:app:FKPPmodel} and \eqref{eq:app:G-P-FKPPmodel}, the P-FKPP model is more flexible and is therefore expected to exhibit reduced epistemic uncertainty, leading to smaller RMSE values. Consistent with this expectation, Table~\ref{tab:rmse_parameters} shows that the P-FKPP model achieves lower RMSEs compared to the standard FKPP model.

Regarding parameter uncertainty, 
we note that EFI is able to quantify the 
uncertainty associated each parameter. However, due to the transformation applied to the first term of \eqref{eq:app:G-P-FKPPmodel}, the values of $D$ are no longer on the same scale across the two models, 
whereas the values of $r$ remain comparable in scale. The results are reported in 
Table \ref{tab:rmse_parameters}.

\paragraph{Summary}
 Through both simulation and real data experiments, we have demonstrated that the proposed EFI algorithm effectively quantifies uncertainties associated with the model and the data-generating process, resulting in accurate estimation of both epistemic and aleatoric uncertainties.


\begin{table}[!ht]
    \centering
    \caption{RMSE and estimated parameters with 95\% confidence intervals for FKPP and Porous-FKPP models.}
     \begin{adjustbox}{width=1.0\textwidth}
    \begin{tabular}{l|c|c|cc|cc|cc}
        \toprule
        \textbf{Model} & \textbf{Initial Cell Density} & \textbf{RMSE} 
        & \textbf{$D$} & 
         Interval 
        & \textbf{$R$} & 
         Interval
        & \textbf{$M$} &  
        Interval \\
        \midrule
        FKPP & 10000 & 58.04 & 0.00936 & (0.00754, 0.01195)  & 0.829 & (0.797, 0.877) & -- & -- \\
        FKPP & 12000 & 82.09 & 0.00378 & (0.00281, 0.00461) & 0.632 & (0.603, 0.658) & -- & -- \\
        FKPP & 14000 & 82.93 & 0.02929 & (0.02739, 0.03268) & 0.534 & (0.505, 0.585) & -- & -- \\
        FKPP & 16000 & 99.14 & 0.02636 & (0.02503, 0.02789) & 0.608 & (0.585, 0.633) & -- & -- \\
        FKPP & 18000 & 115.27 & 0.03784 & (0.03541, 0.04032) & 0.549 & (0.524, 0.575) & -- & -- \\
        FKPP & 20000 & 136.67 & 0.05471 & (0.05007, 0.05817) & 0.492 & (0.458, 0.520) & -- & --  \\
        \midrule
        P-FKPP & 10000 & 46.90 & 1167.67 & (72.48, 2719.97) & 0.846 & (0.832, 0.856) & 1.335 & (1.037, 1.490) \\
        P-FKPP & 12000 & 67.88 & 1825.54 & (32.84, 4416.53) & 0.674 & (0.649, 0.696) & 1.433 & (1.033, 1.603) \\
        P-FKPP & 14000 & 73.59 & 289.37 & (79.08, 552.13) & 0.625 & (0.600, 0.649) & 1.096 & (0.951, 1.199) \\
        P-FKPP & 16000 & 70.83 & 57.36 & (19.21, 97.22) & 0.628 & (0.608, 0.650) & 0.920 & (0.804, 0.999) \\
        P-FKPP & 18000 & 96.50 & 21.58 & (9.07, 35.78) & 0.563 & (0.536, 0.587) & 0.780 & (0.683, 0.863) \\
        P-FKPP & 20000 & 123.34 & 1.472 & (1.058,2.020) & 0.496 & (0.464, 0.530) & 0.408 & (0.370, 0.452) \\
        \bottomrule
    \end{tabular}
    \label{tab:rmse_parameters}
\end{adjustbox}
\end{table}

\begin{figure}[htbp]
    \centering

    \begin{minipage}[b]{0.45\textwidth}
        \centering
        \includegraphics[width=\linewidth]{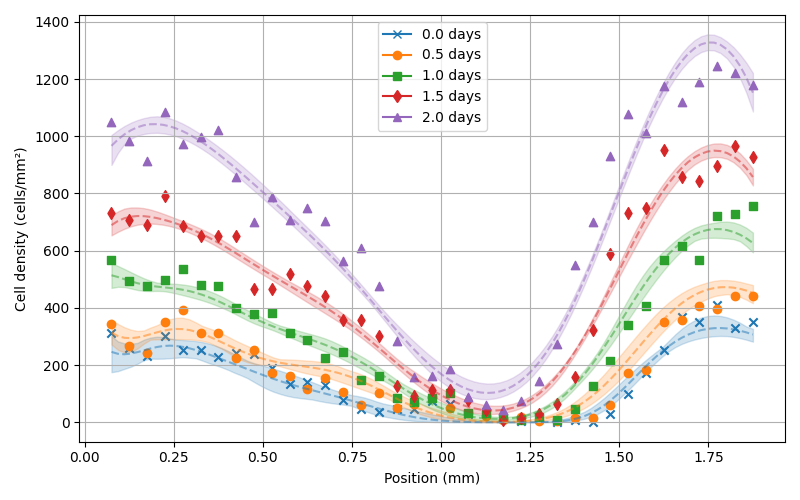}
        \caption*{(a) Initial density 10,000}
    \end{minipage}
    \hfill
    \begin{minipage}[b]{0.45\textwidth}
        \centering
        \includegraphics[width=\linewidth]{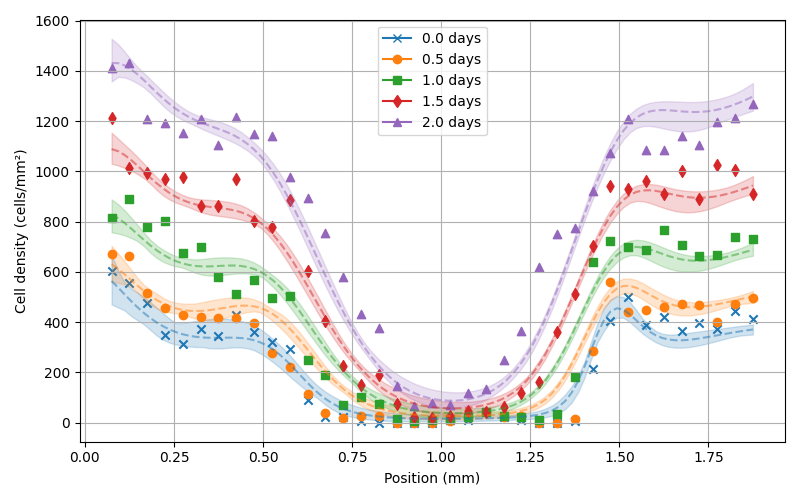}
        \caption*{(b) Initial density 12,000}
    \end{minipage}
    
    \vspace{1em}
    
    \begin{minipage}[b]{0.45\textwidth}
        \centering
        \includegraphics[width=\linewidth]{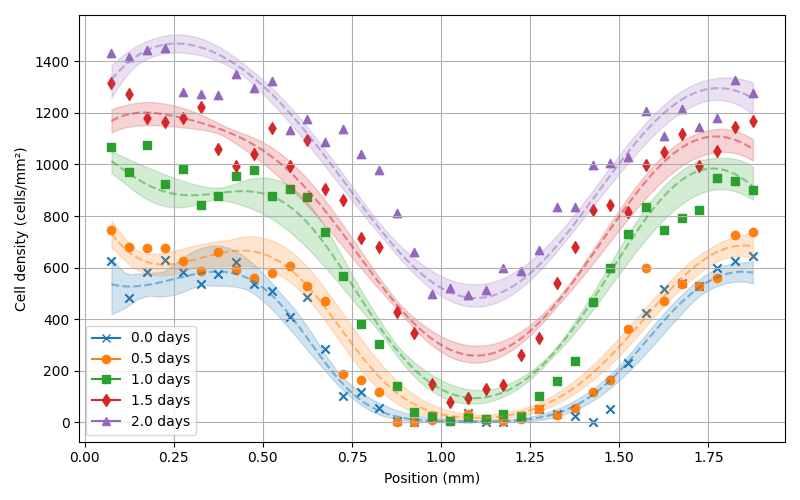}
        \caption*{(c) Initial density 14,000}
    \end{minipage}
    \hfill
    \begin{minipage}[b]{0.45\textwidth}
        \centering
        \includegraphics[width=\linewidth]{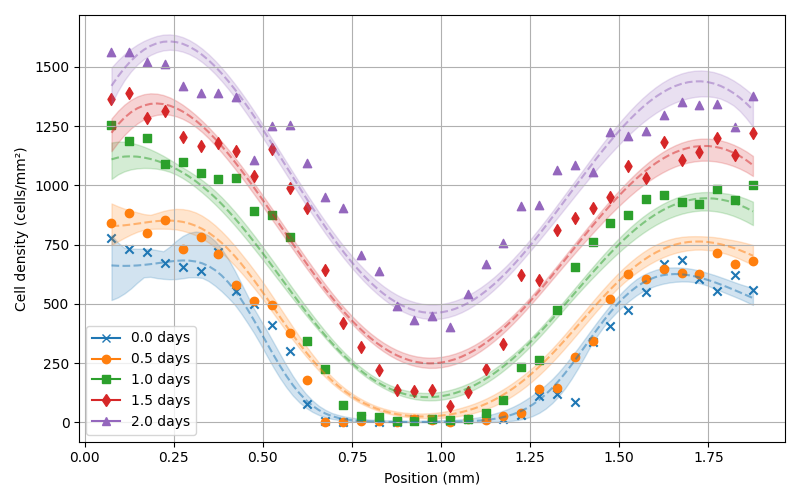}
        \caption*{(d) Initial density 16,000}
    \end{minipage}
    
    \vspace{1em}
    
    \begin{minipage}[b]{0.45\textwidth}
        \centering
        \includegraphics[width=\linewidth]{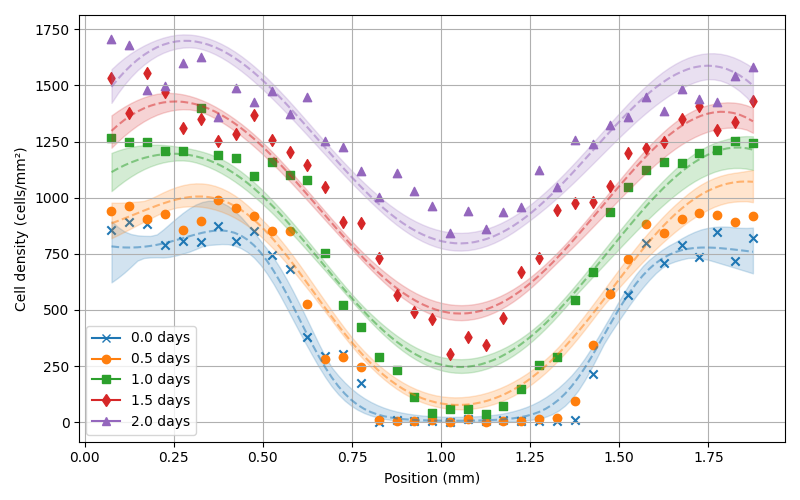}
        \caption*{(e) Initial density 18,000}
    \end{minipage}
    \hfill
    \begin{minipage}[b]{0.45\textwidth}
        \centering
        \includegraphics[width=\linewidth]{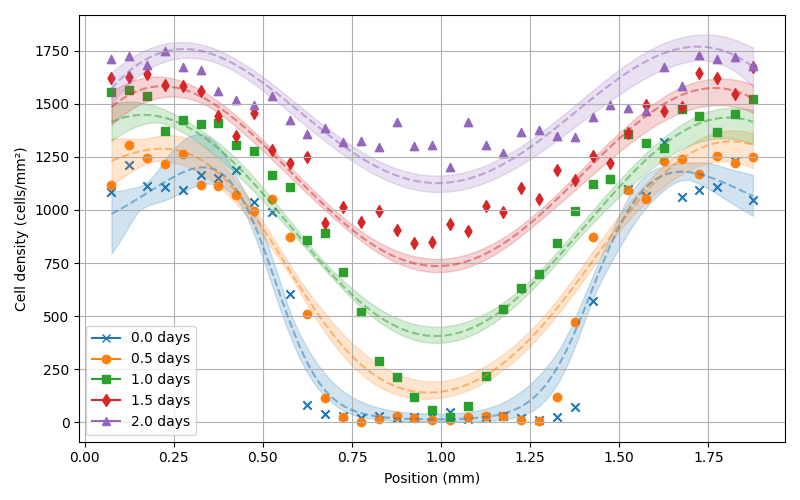}
        \caption*{(f) Initial density 20,000}
    \end{minipage}

    \caption{FKPP model}
    \label{fig:fkpp}
\end{figure}

\begin{figure}[htbp]
    \centering

    \begin{minipage}[b]{0.45\textwidth}
        \centering
        \includegraphics[width=\linewidth]{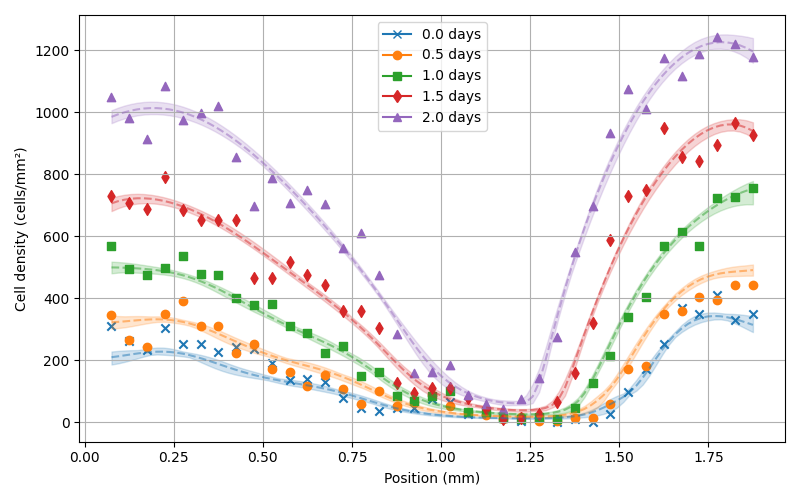}
        \caption*{(a) Initial density 10,000}
    \end{minipage}
    \hfill
    \begin{minipage}[b]{0.45\textwidth}
        \centering
        \includegraphics[width=\linewidth]{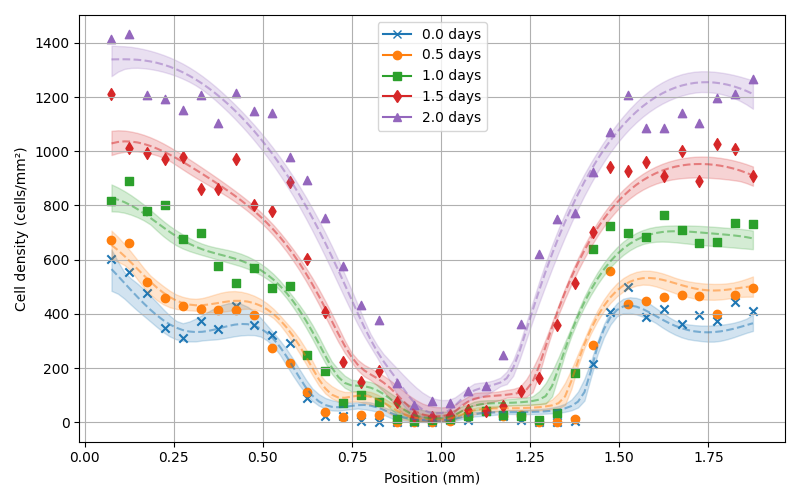}
        \caption*{(b) Initial density 12,000}
    \end{minipage}
    
    \vspace{1em}
    
    \begin{minipage}[b]{0.45\textwidth}
        \centering
        \includegraphics[width=\linewidth]{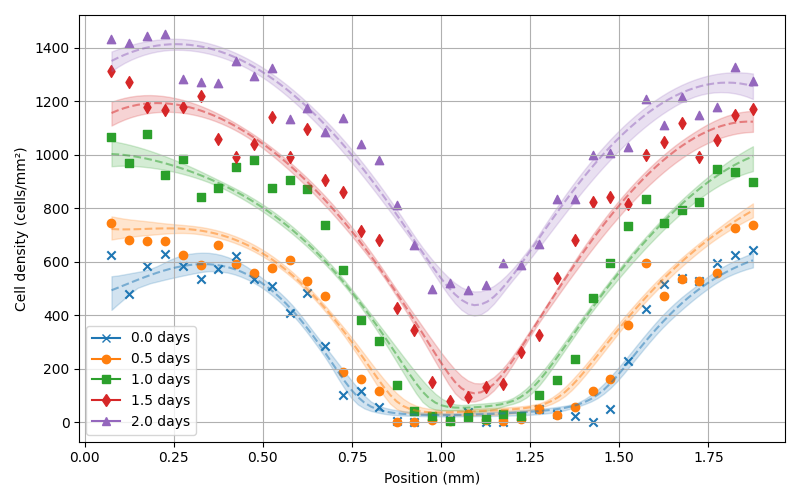}
        \caption*{(c) Initial density 14,000}
    \end{minipage}
    \hfill
    \begin{minipage}[b]{0.45\textwidth}
        \centering
        \includegraphics[width=\linewidth]{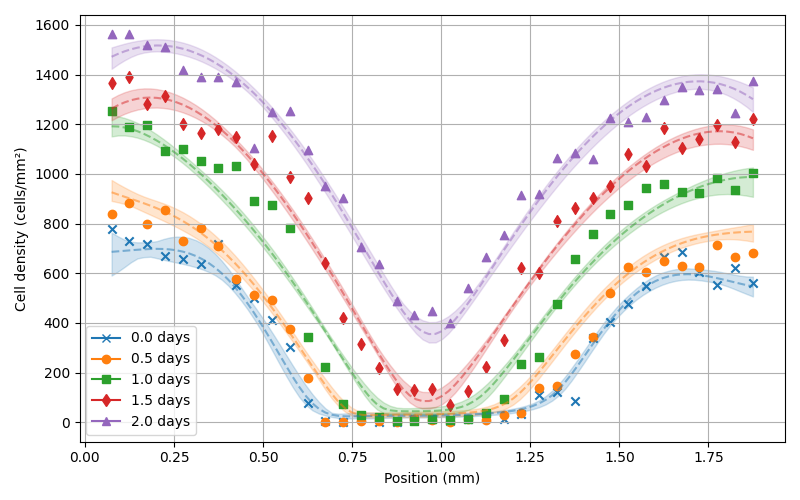}
        \caption*{(d) Initial density 16,000}
    \end{minipage}
    
    \vspace{1em}
    
    \begin{minipage}[b]{0.45\textwidth}
        \centering
        \includegraphics[width=\linewidth]{figures/fkpp/p_fkpp4.png}
        \caption*{(e) Initial density 18,000}
    \end{minipage}
    \hfill
    \begin{minipage}[b]{0.45\textwidth}
        \centering
        \includegraphics[width=\linewidth]{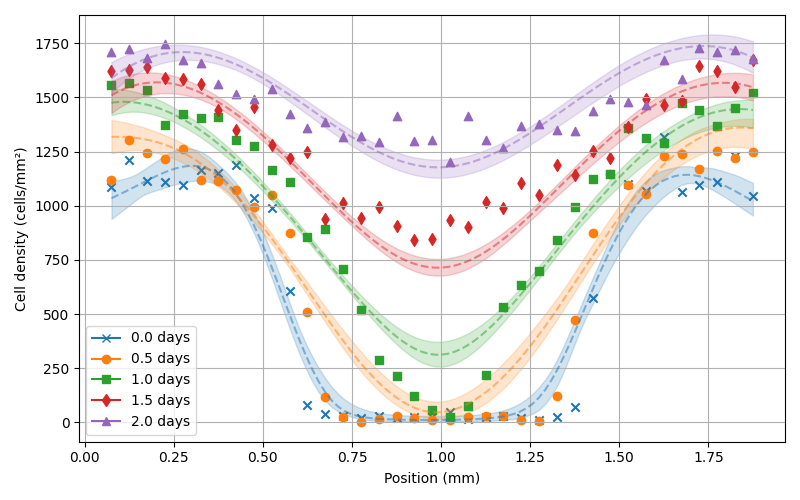}
        \caption*{(f) Initial density 20,000}
    \end{minipage}

    \caption{Porous-FKPP model}
    \label{fig:p-fkpp}
\end{figure}

\section{Experimental Settings} \label{sect:settings}

In all the simulation experiments, we begin by generating data from a specific physics model. Using the simulated data, we iteratively run the algorithm to estimate the model parameters. To enhance convergence of Algorithm \ref{EFIalgorithm}, some algorithmic  parameters (such as learning rate, SGD momentum, $\lambda=1/\epsilon$, and etc.) are adjusted during the initial iterations, referred to as the annealing period.
To tune different parameters, we use three different  annealing schemes, 
including  linear, exponential and polynomial. Their specific forms are in Table \ref{tab:notations}. For the nonlinear Poisson inverse problem,  
both the DNN parameters $\bvartheta$ and the unknown parameter $k$  are  estimated; while for all other problems, only the DNN parameters $\bvartheta$ are estimated. 
At the end of the simulation, the samples collected in the burn-in period are discarded, and the samples
collected in the remaining iterations are used for inference. The burn-in period is set to be at least as long as the annealing period across all experiments.

\textcolor{black}{ 
In our simulations, to ensure the weights of the 
$\bw$-network remain within a compact space as required in Assumption \ref{ass:1}-(i), we impose a Gaussian prior, N(0,100), on each connection weight. However, due to the large variance, this prior has minimal impact on the algorithm's performance, serving primarily to ensure its stability.}

\begin{table}[h!]
\centering
\caption{Notations}
\label{tab:notations}
  \vspace{0.1in}
 \begin{adjustbox}{width=1.0\textwidth}
\begin{tabular}{lp{4.0in}}
\toprule
\textbf{Notation} & \textbf{Meaning} \\
\midrule
epochs & total number of sampling/optimization iterations \\
burn-in period & the proportion of total iterations allocated to the burn-in process \\
annealing period & the proportion of total iterations allocated for parameter adjustment to enhance convergence  \\ 
linear\_x\_y (with progress $\rho \in[0,1]$) & $x+(y-x)\rho$ \\ 
exp\_x\_y (with progress $\rho \in[0,1]$)  & $x^{1-\rho}y^\rho$ \\
poly\_x\_c\_k (with progress $\rho\in[0,1]$) & $\frac{x}{1+(c \rho)^k}$ \\ 
\bottomrule
\end{tabular}
 \end{adjustbox}
\end{table}

\begin{table}[!ht]
  \caption{Parameter settings for 1D-Poisson}
  \label{table:params_poisson}
  \vspace{0.1in}
  \centering
  \footnotesize
   \begin{adjustbox}{width=1.0\textwidth}
  \begin{tabular}{lcccc}
    \toprule
    Parameter Name  & PINN   & Dropout  & Bayesian PINN  & EFI \\
    \midrule
    $t_{\text{start}}$ & -0.7 & -0.7 & -0.7 & -0.7 \\
    $t_{\text{end}}$ & 0.7 & 0.7 & 0.7 & 0.7 \\
    noise sd in $u$ & 0.05 & 0.05 & 0.05 & 0.05 \\
    noise sd in $f$ & 0.0 & 0.0 & 0.0 & 0.0 \\
    \# of solution sensors & 2 & 2 & 2 & 2 \\
    \# of solution replicates & 10 & 10 & 10 & 10 \\
    \# of differential sensors & 200 & 200 & 200 & 200 \\
    \# of differential replicates & 1 & 1 & 1 & 1 \\
    \midrule
    epochs & 500000 & 500000 & 100000 & 200000 \\
    burn-in period & 0.5 & 0.5 & 0.4 & 0.1 \\
    hidden layers & $[50,50]$ & $[50,50]$ & $[50,50]$ & $[50,50]$ \\
    activation function & tanh & tanh & tanh & tanh \\
    learning rate & 3e-4 & 3e-4 & 1e-4/1e-4/linear\_1e-4\_3e-6 & poly\_5e-6\_100.0\_0.55 \\
    $\eta_{f}$ & 1.0 & 1.0 & / & 1.0 \\
    dropout rate & / & 0.5\%/1\%/5\% & / & / \\
    $L$ & / & / & 6 & / \\
    $\sigma_f$ & / & / & 0.05/exp\_0.05\_0.005/exp\_0.05\_0.0005 & / \\
    $\sigma_u$ & / & / & 0.05 & / \\
    annealing period & / & / & 0.3 & 0.1 \\
    sgd momentum & / & / & / & linear\_0.9\_0.0 \\
    sgld learning rate & / & / & / & poly\_5e-6\_10.0\_0.55 \\
    $\lambda$ & / & / & / & linear\_50.0\_500.0 \\
    $\eta_{\theta}$ & / & / & / & 1.0 \\ 
    encoder hidden layers & / & / & / & $[16,16,16]$ \\
    encoder activation function & / & / & / & leaky relu \\
    \bottomrule
  \end{tabular}
\end{adjustbox}
\end{table}

\begin{table}[!ht]
  \caption{Parameter settings for 1D-Poisson (with $f$ error)}
  \label{table:params_poisson_v2}
  \vspace{0.1in}
  \centering
  \footnotesize
  \begin{tabular}{lcccc}
    \toprule
    Parameter Name  & PINN   & Dropout  & Bayesian PINN  & EFI \\
    \midrule
    $t_{\text{start}}$ & -0.7 & -0.7 & -0.7 & -0.7 \\
    $t_{\text{end}}$ & 0.7 & 0.7 & 0.7 & 0.7 \\
    noise sd in $u$ & 0.05 & 0.05 & 0.05 & 0.05 \\
    noise sd in $f$ & 0.05 & 0.05 & 0.05 & 0.05 \\
    \# of solution sensors & 4 & 4 & 4 & 4 \\
    \# of solution replicates & 10 & 10 & 10 & 10 \\
    \# of differential sensors & 40 & 40 & 40 & 40 \\
    \# of differential replicates & 10 & 10 & 10 & 10 \\
    \midrule
    epochs & 100000 & 100000 & 50000 & 200000 \\
    burn-in period & 0.5 & 0.5 & 0.4 & 0.1 \\
    hidden layers & $[50,50]$ & $[50,50]$ & $[50,50]$ & $[50,50]$ \\
    activation function & tanh & tanh & tanh & tanh \\
    learning rate & 3e-4 & 3e-4 & 1e-4 & poly\_2.5e-6\_50.0\_0.55 \\
    $\eta_{f}$ & 1.0 & 1.0 & / & 1.0 \\
    dropout rate & / & 0.5\%/1\%/5\% & / & / \\
    $L$ & / & / & 6 & / \\
    $\sigma_f$ & / & / & 0.05 & / \\
    $\sigma_u$ & / & / & 0.05 & / \\
    annealing period & / & / & / & 0.1 \\
    sgd momentum & / & / & / & linear\_0.9\_0.0 \\
    sgld learning rate & / & / & / & poly\_5e-6\_100.0\_0.55 \\
    $\lambda$ & / & / & / & linear\_50.0\_1000.0 \\
    $\eta_{\theta}$ & / & / & / & 1.0 \\ 
    encoder hidden layers & / & / & / & $[64,64,16]$ \\
    encoder activation function & / & / & / & leaky relu \\
    \bottomrule
  \end{tabular}
\end{table}

\begin{table}[!ht]
  \caption{Parameter settings for nonlinear 1D-Poisson (with $f$ error)}
  \label{table:params_poisson_nonlinear}
  \vspace{0.1in}
  \centering
  \footnotesize
  \begin{tabular}{lcccc}
    \toprule
    Parameter Name  & PINN   & Dropout  & Bayesian PINN  & EFI \\
    \midrule
    $t_{\text{start}}$ & -0.7 & -0.7 & -0.7 & -0.7 \\
    $t_{\text{end}}$ & 0.7 & 0.7 & 0.7 & 0.7 \\
    noise sd in $u$ & 0.05 & 0.05 & 0.05 & 0.05 \\
    noise sd in $f$ & 0.05 & 0.05 & 0.05 & 0.05 \\
    \# of solution sensors & 4 & 4 & 4 & 4 \\
    \# of solution replicates & 10 & 10 & 10 & 10 \\
    \# of differential sensors & 40 & 40 & 40 & 40 \\
    \# of differential replicates & 10 & 10 & 10 & 10 \\
    $k$ & 0.7 & 0.7 & 0.7 & 0.7 \\
    \midrule
    epochs & 100000 & 100000 & 100000 & 200000 \\
    burn-in period & 0.5 & 0.5 & 0.4 & 0.1 \\
    hidden layers & $[50,50]$ & $[50,50]$ & $[50,50]$ & $[50,50]$ \\
    activation function & tanh & tanh & tanh & tanh \\
    learning rate & 3e-4 & 3e-4 & 1e-4 & poly\_5e-6\_100.0\_0.55 \\
    $\eta_{f}$ & 1.0 & 1.0 & / & 1.0 \\
    dropout rate & / & 0.5\%/1\%/5\% & / & / \\
    $L$ & / & / & 6 & / \\
    $\sigma_f$ & / & / & exp\_0.2\_0.05 & / \\
    $\sigma_u$ & / & / & 0.05 & / \\
    annealing period & / & / & 0.3 & 0.1 \\
    sgd momentum & / & / & / & linear\_0.9\_0.0 \\
    sgld learning rate & / & / & / & poly\_5e-6\_100.0\_0.55 \\
    $\lambda$ & / & / & / & exp\_50.0\_1000.0 \\
    $\eta_{\theta}$ & / & / & / & 1.0 \\ 
    encoder hidden layers & / & / & / & $[64,64,16]$ \\
    encoder activation function & / & / & / & leaky relu \\
    \bottomrule
  \end{tabular}
\end{table}

\begin{table}[!ht]
  \caption{Parameter settings for nonlinear 1D-Poisson with parameter estimation}
  \label{table:params_poisson_inverse}
  \vspace{0.1in}
  \centering
  \footnotesize
  \begin{tabular}{lcccc}
    \toprule
    Parameter Name  & PINN   & Dropout  & Bayesian PINN  & EFI \\
    \midrule
    $t_{\text{start}}$ & -0.7 & -0.7 & -0.7 & -0.7 \\
    $t_{\text{end}}$ & 0.7 & 0.7 & 0.7 & 0.7 \\
    noise sd in $u$ & 0.05 & 0.05 & 0.05 & 0.05 \\
    noise sd in $f$ & 0.0 & 0.0 & 0.0 & 0.0 \\
    \# of solution sensors & 8 & 8 & 8 & 8 \\
    \# of solution replicates & 10 & 10 & 10 & 10 \\
    \# of differential sensors & 200 & 200 & 200 & 200 \\
    \# of differential replicates & 1 & 1 & 1 & 1 \\
    $k$ & 0.7 & 0.7 & 0.7 & 0.7 \\
    \midrule
    epochs & 50000 & 500000 & 50000 & 350000 \\
    burn-in period & 0.5 & 0.5 & 0.4 & 0.1 \\
    hidden layers & $[50,50]$ & $[50,50]$ & $[50,50]$ & $[50,50]$ \\
    activation function & tanh & tanh & tanh & tanh \\
    learning rate & 3e-4 & 3e-4 & 1e-4 & poly\_5e-6\_100.0\_0.55 \\
    $\eta_{f}$ & 1.0 & 1.0 & / & 1.0 \\
    dropout rate & / & 0.5\%/1\%/5\% & / & / \\
    $L$ & / & / & 10 & / \\
    $\sigma_f$ & / & / & 0.05 & / \\
    $\sigma_u$ & / & / & 0.05 & / \\
    annealing period & / & / & / & 0.1 \\
    sgd momentum & / & / & / & linear\_0.9\_0.0 \\
    sgld learning rate & / & / & / & poly\_5e-6\_100.0\_0.55 \\
    $\lambda$ & / & / & / & linear\_50.0\_1000.0 \\
    $\eta_{\theta}$ & / & / & / & 1.0 \\ 
    encoder hidden layers & / & / & / & $[128,128,12]$ \\
    encoder activation function & / & / & / & leaky relu \\
    \bottomrule
  \end{tabular}
\end{table}

\begin{table}[!ht]
  \caption{Parameter settings for Black-Scholes Model}
  \label{table:params_euro_call}
  \vspace{0.1in}
  \centering
  \footnotesize
   \begin{adjustbox}{width=1.0\textwidth}
  \begin{tabular}{lcccc}
    \toprule
    Parameter Name  & PINN   & Dropout  & Bayesian PINN  & EFI \\
    \midrule
    $S$ range & $[0.0,2.0]$ & $[0.0,2.0]$ & $[0.0,2.0]$ & $[0.0,2.0]$ \\
    $t$ range & $[0.0,1.0]$ & $[0.0,1.0]$ & $[0.0,1.0]$ & $[0.0,1.0]$ \\
    $\sigma$ & 0.5 & 0.5 & 0.5 & 0.5 \\
    $r$ & 0.05 & 0.05 & 0.05 & 0.05 \\
    $K$ & 1.0 & 1.0 & 1.0 & 1.0 \\
    noise sd & 0.05 & 0.05 & 0.05 & 0.05 \\
    \# of price sensors & 5 & 5 & 5 & 5 \\
    \# of price replicates & 10 & 10 & 10 & 10 \\
    \# of boundary samples & 50 & 50 & 50 & 50 \\
    \# of differential samples & 800 & 800 & 800 & 800 \\
    \midrule
    epochs & 200000 & 200000 & 100000 & 300000 \\
    burn-in period & 0.5 & 0.5 & 0.4 & 0.1 \\
    hidden layers & $[50,50]$ & $[50,50]$ & $[50,50]$ & $[50,50]$ \\
    activation function & softplus($\beta=5$) & softplus($\beta=5$) & softplus($\beta=5$) & softplus($\beta=10$) \\
    learning rate & 3e-4 & 3e-4 & 1e-4/linear\_1e-4\_1e-5 & poly\_5e-6\_100.0\_0.55 \\
    $\eta_{f}$ & 1.0 & 1.0 & / & 1.0 \\
    dropout rate & / & 0.5\%/1\%/5\% & / & / \\
    $L$ & / & / & 6 & / \\
    $\sigma_f$ & / & / & 0.05/exp\_0.05\_0.005 & / \\
    $\sigma_u$ & / & / & 0.05 & / \\
    pretrain epochs & / & / & 5000 & / \\
    annealing period & / & / & 0.3 & 0.1 \\
    sgd momentum & / & / & / & linear\_0.9\_0.0 \\
    sgld learning rate & / & / & / & poly\_5e-6\_100.0\_0.55 \\
    sgld alpha & / & / & / & 1.0 \\
    $\lambda$ & / & / & / & linear\_50.0\_1000.0 \\
    $\eta_{\theta}$ & / & / & / & 1.0 \\ 
    encoder hidden layers & / & / & / & $[64,64,16]$ \\
    encoder activation function & / & / & / & leaky relu \\
    \bottomrule
  \end{tabular}
\end{adjustbox}
\end{table}

\begin{table}[!ht]
  \caption{Parameter settings for real data}
  \label{table:params_fkpp}
  \vspace{0.1in}
  \centering
  \footnotesize
   \begin{adjustbox}{width=1.0\textwidth}
  \begin{tabular}{lcccc}
    \toprule
    Parameter Name  & Montroll growth   & FKPP & P-FKPP \\
    \midrule
    epochs & 200000 & 200000 & 200000  \\
    burn-in period & 0.1 & 0.1 & 0.1  \\
    hidden layers & $[50,50]$ & $[50,50]$ & $[50,50]$  \\
    activation function & softplus($\beta=10$) & tanh & tanh  \\
    learning rate & poly\_1e-6\_10.0\_0.55 & poly\_1e-6\_10.0\_0.55 & poly\_1e-6\_10.0\_0.55\\
    $\eta_{f}$ & 1.0 & 1.0 & 1.0 \\
    sgd momentum & 0.9 & 0.9 & 0.9 \\
    sgld learning rate & poly\_1e-4\_10.0\_0.95 & poly\_1e-6\_10.0\_0.95 & poly\_1e-6\_10.0\_0.95\\
    sgld alpha & 1.0 & 1.0 & 1.0\\
    $\lambda$ & log\_50.0\_500.0 & log\_50.0\_500.0 & log\_50.0\_500.0 \\
    $\eta_{\theta}$ & 1.0  & 1.0  & 1.0 \\ 
    encoder hidden layers & $[32,32,16]$ & $[64,64,16]$ & $[64,64,16]$ \\
    encoder activation function  & leaky relu & leaky relu & leaky relu \\
    \bottomrule
  \end{tabular}
\end{adjustbox}
\end{table}





\end{document}